\documentclass{article}

\usepackage[numbers]{natbib}
\usepackage[final]{neurips_2021}

\usepackage[frozencache=true]{minted}
\usepackage[utf8]{inputenc} 
\usepackage[T1]{fontenc}    
\usepackage{url}            
\usepackage{booktabs}       
\usepackage{amsfonts}       
\usepackage{nicefrac}       
\usepackage{microtype}      
\usepackage{xcolor}         
\usepackage{comment}

\usepackage{pifont}
\newcommand{\cmark}{\ding{51}}%
\newcommand{\xmark}{\ding{55}}%

\usepackage{multirow}
\usepackage{enumitem}
\usepackage{float}
\usepackage{caption}
\usepackage{graphicx}
\usepackage{dirtytalk}
\usepackage{amsmath,amssymb,amsthm,amsfonts,bm}       
\PassOptionsToPackage{hyphens}{url}\usepackage[colorlinks=true]{hyperref}

\usepackage{amsmath,amsfonts,bm}

\newtheorem{Theorem}{Theorem}

\newtheorem{Lemma}{Lemma}
\newtheorem{Proposition}{Proposition}

\newtheorem{Definition}{Definition}

\NewEnviron{medium1}{%
\scalebox{1.1}{\BODY}}
\NewEnviron{medium2}{%
\scalebox{1.05}{\BODY}}

\def\meanstd#1#2{{#1}{\scriptsize $\pm$#2}}
\def\meanstdbf#1#2{\textbf{#1}{\scriptsize $\pm$#2}}
\def\mean#1#2{#1}
\def\meanbf#1#2{\textbf{#1}}

\def\Tableref#1{Table~\ref{#1}}

\def\Figref#1{Figure~\ref{#1}}

\def\Secref#1{Section~\ref{#1}}
\def\Secref#1{Section~\ref{#1}}

\def\eqref#1{equation~\ref{#1}}
\def\Eqref#1{Eq.~(\ref{#1})}
\def\Thmref#1{Theorem~\ref{#1}}
\def\Lemmaref#1{Lemma~\ref{#1}}
\def\Propref#1{Proposition~\ref{#1}}

\def\Defref#1{Definition~\ref{#1}}

\def\Appendref#1{Appendix~\ref{#1}}

\def\1{\bm{1}}

\def\rvb{{\mathbf{b}}}

\def\rvr{{\mathbf{r}}}

\def\rvx{{\mathbf{x}}}
\def\rvy{{\mathbf{y}}}
\def\rvz{{\mathbf{z}}}

\def\vgamma{{\bm{\gamma}}}
\def\vbeta{{\bm{\beta}}}
\def\vomega{{\bm{\omega}}}

\def\mW{{\bm{W}}}

\DeclareMathAlphabet{\mathsfit}{\encodingdefault}{\sfdefault}{m}{sl}
\SetMathAlphabet{\mathsfit}{bold}{\encodingdefault}{\sfdefault}{bx}{n}

\def\sN{{\mathbb{N}}}

\def\0{{\bm{0}}}

\renewcommand{\P}{\mathbb{P}}
\newcommand{\E}{\mathbb{E}}

\newcommand{\R}{\mathbb{R}}

\newcommand{\Var}{\mathrm{Var}}

\newcommand{\normltwo}{L^2}
\newcommand{\normlfour}{L^4}

\newcommand{\pwr}{\mathcal{P}}

\newcommand{\D}{\mathcal{D}}

\newcommand{\hr}{\hat{\rvr}}

\newcommand{\hx}{\hat{\rvx}}
\newcommand{\hy}{\hat{\rvy}}
\newcommand{\hz}{\hat{\rvz}}

\newcommand{\cy}{\check{\rvy}}
\newcommand{\cz}{\check{\rvz}}
\newcommand{\bz}{\breve{\rvz}}
\newcommand{\wtz}{\widetilde{\rvz}}

\newcommand{\ty}{\tilde{\rvy}}
\newcommand{\tz}{\tilde{\rvz}}

\DeclareMathOperator*{\argmin}{arg\,min}
\DeclareMathOperator*{\relu}{ReLU}
\DeclareMathOperator*{\swish}{Swish}

\DeclareMathOperator*{\identity}{identity}

\DeclareMathOperator{\BN}{BN}

\newcommand{\pk}{^{k}}
\newcommand{\pkm}{^{k-1}}

\newcommand{\uo}{_{0}}

\newcommand{\ul}{_{l}}
\newcommand{\ulm}{_{l-1}}

\newcommand{\uk}{_{k}}

\newcommand{\po}{^{0}}
\newcommand{\pl}{^{l}}

\newcommand{\plm}{^{l-1}}

\newcommand{\plplus}{^{l,+}}
\newcommand{\plminus}{^{l,-}}

\newcommand{\uc}{_{c}}

\newcommand{\pwrone}{\pwr^{(1)}}
\newcommand{\pwrtwo}{\pwr^{(2)}}
\newcommand{\pwrthree}{\pwr^{(3)}}
\newcommand{\pwrfour}{\pwr^{(4)}}

\newcommand{\Norm}{\operatorname{Norm}}
\newcommand{\Conv}{\operatorname{Conv}}
\newcommand{\Act}{\operatorname{Act}}
\newcommand{\PNAct}{\operatorname{PN-Act}}

\newcommand{\LN}{\operatorname{LN}}
\newcommand{\IN}{\operatorname{IN}}
\newcommand{\GN}{\operatorname{GN}}

\renewenvironment{proof}{{\bfseries Proof}}{}

\title{Proxy-Normalizing Activations to Match Batch Normalization while Removing Batch Dependence}

\author{%
  Antoine Labatie\quad Dominic Masters\quad Zach Eaton-Rosen\quad Carlo Luschi \\
  Graphcore Research, UK\\
  \texttt{antoine.labatie@centraliens.net}\\ \texttt{\{dominicm,\,zacher,\,carlo\}@graphcore.ai}\\ 
}

\begin{document}

\maketitle

\begin{abstract}
We investigate the reasons for the performance degradation incurred with batch-independent normalization. We find that the prototypical techniques of layer normalization and instance normalization both induce the appearance of failure modes in the neural network's pre-activations: (i) layer normalization induces a collapse towards channel-wise constant functions; (ii) instance normalization induces a lack of variability in instance statistics, symptomatic of an alteration of the expressivity. To alleviate failure mode (i) without aggravating failure mode (ii), we introduce the technique \say{Proxy Normalization} that normalizes post-activations using a proxy distribution. When combined with layer normalization or group normalization, this batch-independent normalization emulates batch normalization's behavior and consistently matches or exceeds its performance.
\end{abstract}

\section{Introduction}
Normalization plays a critical role in deep learning as it allows successful scaling to large and deep models. In vision tasks, the most well-established normalization technique is Batch Normalization (BN) \citep{Ioffe15}. At every layer in the network, BN normalizes the intermediate activations to have zero mean and unit variance in each channel. While indisputably successful when training with large batch size, BN incurs a performance degradation in the regime of small batch size \cite{Ioffe17,Wu18,Ying18,Singh19b,Summers20,Yan20}. This performance degradation is commonly attributed to an excessive or simply unwanted regularization stemming from the noise in the approximation of full-batch statistics by mini-batch statistics.

Many techniques have been proposed to avoid this issue, while at the same time retaining BN's benefits. Some techniques mimic BN's principle while decoupling the computational batch from the normalization batch \citep{Ioffe17,Chiley19,Yan20}. Other techniques are \say{batch-independent} in that they operate independently of the batch in various modalities: through an explicit normalization either in activation space \citep{Ba16,Ulyanov16b,Ren17,Luo19b,Luo19c,Wu18,Singh19b,Zhang19c,Liu20,Summers20} or in weight space \citep{Salimans16,Huang17a,Qiao19,Ruff19,Brock21a}; through the use of an analytic proxy to track activation statistics \citep{Arpit16,Laurent17,Shekhovtsov18b}; \mbox{through a change of activation function \citep{Klambauer17,Singh19b,Liu20}.}

In this paper, we push the endeavor to replace BN with a batch-independent normalization a step further. Our main contributions are as follows: (i) we introduce a novel framework to finely characterize the various neural network properties affected by the choice of normalization; (ii) using this framework, we show that while BN's beneficial properties \emph{are not} retained when solely using the prototypical batch-independent norms, they \emph{are} retained when combining some of these norms with \say{Proxy Normalization}, a novel technique that we hereby introduce; (iii) we demonstrate on an extensive set of experiments that, by reproducing BN's beneficial properties, our batch-independent normalization approach consistently matches or exceeds BN's performance.

As a starting point of our analysis, we need to gain a better understanding of those beneficial properties of BN that we aim at reproducing.

\section{Batch Normalization's beneficial properties}
\label{sec:bn_properties}

We consider throughout this paper a convolutional neural network with $d=2$ spatial axes. This neural network receives an input $\rvx~\in~\R^{H \times W \times C\uo}$ which, unless otherwise stated, is assumed sampled from a finite dataset $\D$. The neural network maps this input $\rvx$ to intermediate activations $\rvx\pl~\in~\R^{H \times W \times C\ul}$ of \emph{height} $H$, \emph{width} $W$ and \emph{number of channels} $C\ul$ at each layer $l$. The value of $\rvx\pl$ at spatial position \mbox{$\alpha \in \{1, \dots, H \} \times \{1, \dots, W\}$} and channel $c \in \{1, \dots, C\ul \}$ is denoted as $\rvx\pl_{\alpha,c}$, with the dependency on $\rvx$ kept \emph{implicit} to avoid overloading notations.

The inclusion of BN at layer $l$ leads in the full-batch setting to adding the following operations $\forall \alpha,c$:
\begin{align}
\rvy\pl_{\alpha,c} & = \frac{\rvx\pl_{\alpha,c} - \mu\uc(\rvx\pl)}{\sigma\uc(\rvx\pl)}, &
\ty\pl_{\alpha,c} & = \vgamma\pl\uc \rvy\pl_{\alpha,c} +\vbeta\pl\uc,
\label{eq:bn}
\end{align}
where $\mu\uc(\rvx\pl)$, $\sigma\uc(\rvx\pl)$ are the mean and standard deviation of $\rvx\pl$ in channel $c$, and $\vgamma\pl\uc$, $\vbeta\pl\uc$ are channel-wise scale and shift parameters restoring the degrees of freedom lost in the standardization. In the mini-batch setting, \mbox{the full-batch statistics $\mu\uc(\rvx\pl)$, $\sigma\uc(\rvx\pl)$ are approximated by mini-batch statistics.}

\Tableref{table:properties} summarizes the beneficial properties that result from including BN in the neural network. Below, we provide details on each of these properties, and we discuss whether each property is reproduced with batch-independent norms.

\textbf{Scale invariance.} 
When BN is present, the input-output mapping of the neural network is invariant to the scale of weights preceding any BN layer. With such scale invariance plus weight decay, the scale of weights during training reaches an equilibrium with an \say{effective} learning rate depending on both the learning rate and the weight decay strength \citep{Laarhoven17,Arora19,Hoffer18,Chiley19,Li20,Zhang19b,Wan20}. Such mechanism of \say{auto rate-tuning} has been shown to provide optimization benefits \citep{Cho17,Arora19,Kohler19}.

\emph{This property is easy to reproduce. It is already obtained with most existing batch-independent norms.}

\textbf{Control of activation scale in residual networks.}
To be trainable, residual networks require the scale of activations to remain well-behaved at initialization \citep{Hanin18,Zhang19a,De20,Brock21a,Shao20,Lubana21}. While this property naturally arises when BN is present on the residual path, when BN is not present it can also be enforced by a proper scaling decaying with the depth of the residual path. This \say{dilution} of the residual path with respect to the skip connection path reduces the effective depth of the network and enables to avoid coarse-grained failures modes \citep{Veit16,Philipp18,Yang19,Labatie19,Galloway19,De20}. 

\emph{This property is easy to reproduce. It is already obtained with most existing batch-independent norms.}

\textbf{Regularizing noise.}
Due to the stochasticity of the approximation of $\mu\uc(\rvx\pl)$, $\sigma\uc(\rvx\pl)$ by mini-batch statistics, training a neural network with BN in the mini-batch setting can be seen as equivalent to performing Bayesian inference \citep{Teye18,Shekhovtsov18a} or to adding a regularizing term to the training of the same network with full-batch statistics \citep{Luo19a}. \mbox{As a result, BN induces a specific form of regularization.}

\emph{This regularization is not reproduced with batch-independent norms, but we leave it out of the scope of this paper.} To help minimize the bias in our analysis and \say{subtract away} this effect, we will perform all our experiments without and with extra degrees of regularization. This procedure can be seen as a coarse disentanglement of normalization's effects from regularization's effects.

\textbf{Avoidance of collapse.}
Unnormalized networks with non-saturating nonlinearities are subject to a phenomenon of \say{collapse} whereby the distribution with respect to $\rvx$, $\alpha$ of the intermediate activation vectors $(\rvx\pl_{\alpha,1},\dots,\rvx\pl_{\alpha,C\ul})^\mathrm{T}$ becomes close to zero- or one-dimensional in deep layers \citep{Philipp18,Bjorck18,Labatie19,Jacot19,Daneshmand20,Rao20,Lubana21}. This means that deep in an unnormalized network: (i) layers tend to have their channels imbalanced; (ii) nonlinearities tend to become channel-wise linear with respect to $\rvx$, $\alpha$ and not add any effective capacity \citep{Philipp18,Hanin19a,Hanin19b}. Consequently, unnormalized networks can neither effectively use their whole width \mbox{(imbalanced channels) nor effectively use their whole depth (channel-wise linearity).}

Conversely, when BN is used, the standardization at each layer prevents this collapse from happening. Even in deep layers, channels remain balanced and nonlinearities remain channel-wise nonlinear with respect to $\rvx$, $\alpha$. \mbox{Consequently, networks with BN can effectively use their whole width and depth.}

The collapse is, on the other hand, not always avoided with batch-independent norms \citep{Philipp18,Qiao19,Rao20}. Most notably, it is \emph{not} avoided with Layer Normalization (LN) \citep{Ba16} or Group Normalization (GN) \citep{Wu18}, \mbox{as we show both theoretically and experimentally on commonly found networks in \Secref{sec:failure_modes}.}

\emph{To the extent possible, we aim at designing a batch-independent norm that avoids this collapse.}

\textbf{Preservation of expressivity.} 
We can always express the identity with \Eqref{eq:bn} by choosing \mbox{$\vbeta\pl\uc=\mu\uc(\rvx\pl)$} and $\vgamma\pl\uc=\sigma\uc(\rvx\pl)$. Conversely, for any choice of $\vbeta\pl\uc$, $\vgamma\pl\uc$, we can always \say{re-absorb} \Eqref{eq:bn} into a preceding convolution with bias. This means that BN in the full-batch setting \emph{does not alter the expressivity} compared to an unnormalized network, i.e. it amounts to a plain reparameterization of the hypothesis space.

The expressivity is, on the other hand, not always preserved with batch-independent norms. In activation space, the dependence of batch-independent statistics on the input $\rvx$ turns the standardization into a channel-wise nonlinear operation that cannot be \say{re-absorbed} into a preceding convolution with bias \citep{Qiao19}. This phenomenon is most pronounced when statistics get computed over few components. This means e.g. that Instance Normalization (IN) \citep{Ulyanov16b} induces a greater change of expressivity than GN, which itself induces a greater change of expressivity than LN.

In weight space, the expressivity can also be altered, namely by the removal of degrees of freedom. This is the case with Weight Standardization (WS) \citep{Qiao19,Brock21a} and Centered Weight Normalization~\citep{Huang17a} that remove degrees of freedom (one per unit) that \emph{cannot} be restored in a succeeding affine transformation. This reduction of expressivity could explain the ineffectiveness of these techniques in EfficientNets \citep{Tan19}, \mbox{as previous works observed \citep{Brock21a} and as we confirm in \Secref{sec:results}.}

\emph{To the extent possible, we aim at designing a batch-independent norm that preserves expressivity.}

\begin{table}
\centering
\caption{\textbf{BN's beneficial properties}. We show whether each property is (at least approximately) present (\ding{51}) or absent (\ding{55}) with various batch-independent norms: Layer Normalization (LN), Instance Normalization (IN), Layer Normalization + Proxy Normalization (LN+PN, cf \Secref{sec:pn}). In this categorization, BN is considered in the mini-batch setting but still close to the full-batch setting, such that it approximately preserves expressivity \cite{Huang20}.}
\label{table:properties}
\scalebox{1}{\begin{tabular}{lccccc}
\toprule
  & Scale
  & Control of
  & Regularizing
  & Avoidance of
  & Preservation of \\
  & invariance & activation scale & noise & collapse & expressivity \\
\midrule
BN    &  \cmark&  \cmark & \cmark &  \cmark & \cmark \\
LN    &  \cmark&  \cmark & \xmark & \xmark & \cmark \\
IN    &  \cmark&  \cmark & \xmark & \cmark & \xmark \\
LN+PN &  \cmark&  \cmark & \xmark & \cmark & \cmark \\
\bottomrule
\end{tabular}}
\end{table}

\section{Theoretical framework of analysis}
\label{sec:framework}
We specified the different properties that we wish to retain in our design of batch-independent normalization: (i) scale invariance, (ii) control of activation scale; (iii) avoidance of collapse; (iv) preservation of expressivity. We now introduce a framework to quantify the presence or absence of the specific properties (iii) and (iv) with various choices of normalization.

\paragraph{Propagation.}
For simplicity, we assume in our theoretical setup that any layer $l$ up to depth $L$ consists of the following three steps: (i)~convolution step with weights $\vomega\pl \in \R^{K\ul \times K\ul \times C\ulm \times C\ul}$; (ii)~normalization step; (iii)~activation step sub-decomposed into an affine transformation with scale and shift parameters $\vgamma\pl,\vbeta\pl \in\R^{C\ul}$ and an activation function $\phi$ which, unless otherwise stated, is assumed positive homogeneous and nonzero (e.g. $\phi=\relu$). If we denote $\rvx\pl,\rvy\pl,\rvz\pl~\in~\R^{H \times W \times C\ul}$ the intermediate activations situated just after (i), (ii), (iii) with the convention $\rvz\po\equiv\rvx$, we may write the propagation through layer $l$ as
\begin{align}
\rvx\pl & = \Conv(\rvz\plm), & \qquad\forall \alpha,c: & \quad \Conv(\rvz\plm)_{\alpha,c} = (\vomega\pl \ast \rvz\plm)_{\alpha,c}, \label{eq:propagation1} \\
\rvy\pl & = \Norm(\rvx\pl), & \qquad \forall \alpha,c: & \quad \Norm(\rvx\pl)_{\alpha,c} = \frac{\rvx\pl_{\alpha,c} - \mu_{I_{\rvx,c}}(\rvx\pl)}{\sigma^{}_{I_{\rvx,c}}(\rvx\pl)},\footnotemark  \label{eq:propagation2}\\
 \rvz\pl & = \Act(\rvy\pl), &\qquad \forall \alpha,c: & \quad \Act(\rvy\pl)_{\alpha,c} = \phi\big(\ty\pl_{\alpha,c}\big)=\phi\big(\vgamma\pl_c \rvy\pl_{\alpha,c} + \vbeta\pl_c\big), \label{eq:propagation3} 
\end{align}\footnotetext{We omit the numerical stability constant and adopt the convention $\Norm(\rvx\pl)_{\alpha,c}=0$, $\forall\alpha$ if $\sigma^{}_{I_{\rvx,c}}(\rvx\pl)=0$.}where $\mu_{I_{\rvx,c}}(\rvx\pl)$, $\sigma_{I_{\rvx,c}}(\rvx\pl)$ denote the mean and standard deviation of $\rvx\pl$ conditionally on $I_{\rvx,c}\equiv\{c\}$, $\{\rvx\}$, $\{\rvx,c\}$, $\{\rvx,c\bmod{} G\}$ for the respective cases $\Norm=\BN$, $\LN$, $\IN$, $\GN$ with $G$ groups.

\paragraph{Moments.} Extending the previous notations, we use $\mu$,~$\sigma$,~$\pwr$ indexed with a (possibly empty) subset of variables to denote the operators of conditional mean, standard deviation and power. If we apply these operators to the intermediate activations $\rvy\pl$, that implicitly depend on the input $\rvx$ and that explicitly depend on the spatial position $\alpha$ and the channel $c$, we get e.g.
\begin{align*}
\mu\uc(\rvy\pl) & = \E_{\rvx,\alpha}\big[\rvy\pl_{\alpha,c}\big], & 
\sigma\uc(\rvy\pl) & = \sqrt{\Var_{\rvx,\alpha}\big[\rvy\pl_{\alpha,c}\big]}, &  
\pwr\uc(\rvy\pl) & = \underbrace{\E_{\rvx,\alpha}\big[ (\rvy\pl_{\alpha,c})^2 \big]}_{\mu\uc(\rvy\pl)^2+\sigma\uc(\rvy\pl)^2}, \\
\mu_{\rvx,c}(\rvy\pl) & = \E_{\alpha}\big[\rvy\pl_{\alpha,c}\big], & 
\sigma_{\rvx,c}(\rvy\pl) & = \sqrt{\Var_{\alpha}\big[\rvy\pl_{\alpha,c}\big]}, &  
\pwr_{\rvx,c}(\rvy\pl) & = \hspace{-9.5pt}\underbrace{\E_{\alpha}\big[ (\rvy\pl_{\alpha,c})^2 \big]}_{\mu_{\rvx,c}(\rvy\pl)^2+\sigma_{\rvx,c}(\rvy\pl)^2}\hspace{-12pt},
\end{align*}
where, by convention, $\rvx$, $\alpha$, $c$ are considered uniformly sampled among inputs of $\D$, spatial positions and channels, whenever they are considered as random.

\paragraph{Power decomposition.} Using these notations, we may gain important insights by decomposing the power in channel $c$ of $\rvy\pl$, just after the normalization step, as
\begin{align}
\pwr\uc(\rvy\pl) 
 & = \underbrace{\E_{\rvx} \Big[ \mu_{\rvx,c}(\rvy\pl) \Big]^2}_{\pwrone\uc(\rvy\pl)}
 + \underbrace{\Var_{\rvx} \Big[ \mu_{\rvx,c}(\rvy\pl) \Big]}_{\pwrtwo\uc(\rvy\pl)}
 + \underbrace{\E_{\rvx} \Big[ \sigma_{\rvx,c}(\rvy\pl) \Big]^2}_{\pwrthree\uc(\rvy\pl)} 
 + \underbrace{\Var_{\rvx} \Big[ \sigma_{\rvx,c}(\rvy\pl) \Big].}_{\pwrfour\uc(\rvy\pl)} 
 \label{eq:pwr_decomp}
\end{align}
Since this four-terms \emph{power decomposition} will be at the core of our analysis, we detail two useful views of it. The first view is that of a \emph{hierarchy of scales}: $\pwrone\uc(\rvy\pl)$ measures the power of $\mu\uc(\rvy\pl)$ at the dataset scale; $\pwrtwo\uc(\rvy\pl)$ measures the power of $\mu_{\rvx,c}(\rvy\pl)-\mu\uc(\rvy\pl)$ at the instance scale; the sum of $\pwrthree\uc(\rvy\pl)$ and $\pwrfour\uc(\rvy\pl)$ measures the power of $\rvy\pl_{\alpha,c}-\mu_{\rvx,c}(\rvy\pl)$ at the pixel scale. A particular situation where the power would be concentrated at the dataset scale with $\pwrone\uc(\rvy\pl)$ equal to $\pwr\uc(\rvy\pl)$ would imply that $\rvy\pl$ \mbox{has its distribution fully \say{collapsed} in channel $c$, i.e. that $\rvy\pl$ is constant in channel $c$.\qquad}

The second view is that of a \emph{two-level binary tree}: on one half of the tree, the sum of $\smash{\pwrone\uc(\rvy\pl)}$ and $\smash{\pwrtwo\uc(\rvy\pl)}$ measures the power coming from $\mu_{\rvx,c}(\rvy\pl)$, with the relative proportions of $\pwrone\uc(\rvy\pl)$ and $\pwrtwo\uc(\rvy\pl)$ functions of the inter-$\rvx$ similarity and inter-$\rvx$ variability of $\mu_{\rvx,c}(\rvy\pl)$; on the other half of the tree, the sum of $\pwrthree\uc(\rvy\pl)$ and $\pwrfour\uc(\rvy\pl)$ measures the power coming from $\sigma_{\rvx,c}(\rvy\pl)$, with the relative proportions of $\pwrthree\uc(\rvy\pl)$ and $\pwrfour\uc(\rvy\pl)$ functions of the inter-$\rvx$ similarity and inter-$\rvx$ variability of $\sigma_{\rvx,c}(\rvy\pl)$. A particular situation where $\pwrtwo\uc(\rvy\pl)$, $\pwrfour\uc(\rvy\pl)$ would be equal to zero would imply that $\mu_{\rvx,c}(\rvy\pl)$, $\sigma_{\rvx,c}(\rvy\pl)$ have zero inter-$\rvx$ variability, i.e. that $\mu_{\rvx,c}(\rvy\pl)$, $\sigma_{\rvx,c}(\rvy\pl)$ are constant for all $\rvx$.

A version of \Eqref{eq:pwr_decomp} at the layer level instead of channel level will be easier to work with. Defining $\pwr^{(i)}(\rvy\pl)$ as the averages of $\pwr^{(i)}\uc(\rvy\pl)$ over $c$ for $i\in \{1,2,3,4\}$, we obtain 
\begin{align*}
 \pwr(\rvy\pl) & = \pwrone(\rvy\pl) + \pwrtwo(\rvy\pl) + \pwrthree(\rvy\pl) + \pwrfour(\rvy\pl). 
\end{align*}
It should be noted that $\pwr(\rvy\pl)=1$ for any choice of $\Norm \in \{ \BN, \LN, \IN, \GN \}$ as long as the denominator of \Eqref{eq:propagation2} is nonzero for all $\rvx,c$ [\ref{sec:powers_one}]. Consequently, the terms $\pwr^{(i)}(\rvy\pl)$ sum to one, meaning they can be conveniently seen as the proportion of each term $i\in \{1,2,3,4\}$ into $\pwr(\rvy\pl)$.

\paragraph{Revisiting BN's avoidance of collapse.} 

 When BN is used, $\rvy\pl$ is normalized not only layer-wise but also channel-wise with $\pwrone\uc(\rvy\pl)=0$ and $\pwr\uc(\rvy\pl)=1$. As a first consequence, $\ty\pl$ (that is only one affine transformation away from $\rvy\pl$) is \emph{unlikely} to have its channel-wise distributions collapsed. This means that the nonlinearity $\phi$ acting on $\ty\pl$ is \emph{likely} to be effectively nonlinear with respect to $\ty\pl$'s channel-wise distributions.\footnote{Note that: (i) the effective nonlinearity of $\phi$ with respect to $\ty\pl$'s channel-wise distributions could be quantified in the context of random nets of \Defref{def:random} with \say{reasonable} choices of $\beta$, $\gamma$; (ii) BN only guarantees an intra-distribution nonlinearity and not an intra-mode nonlinearity in contexts such as adversarial training \citep{Xie20a,Xie20b} or conditional GANs \citep{Miyato18}, unless modes are decoupled in BN's computation \citep{Vries17,Deecke18,Wang19,Chang19,Wu21}.}  As a result, each layer adds capacity and the network effectively uses its whole \emph{depth}. This is opposite to the situation where $\ty\pl$ has its channel-wise distributions collapsed with $\pwr\uc(\ty\pl)-\pwrone\uc(\ty\pl) \ll \pwr\uc(\ty\pl)$ for all $c$, which results in $\phi$ being close to linear with respect to $\ty\pl$'s channel-wise distributions. This is illustrated in \Figref{fig:channelwise_linearity} and formalized in \Appendref{sec:channelwise_linearity}.
\newpage

As an additional consequence, $\rvy\pl$ is guaranteed to have its channels well balanced with equal power $\pwr\uc(\rvy\pl)$ for all $c$. As a result, the network effectively uses its whole \emph{width}. This is opposite to the situation where a single channel $c$ becomes overly dominant over the others with $\pwr\uc(\rvy\pl) \gg \pwr_{c'}(\rvy\pl)$ for $c\neq c'$, which results in downstream layers only \say{seeing} this channel $c$ and the network behaving as if it had a width equal to one at layer $l$.

\begin{figure}
\centering
\includegraphics[width=0.95\textwidth]{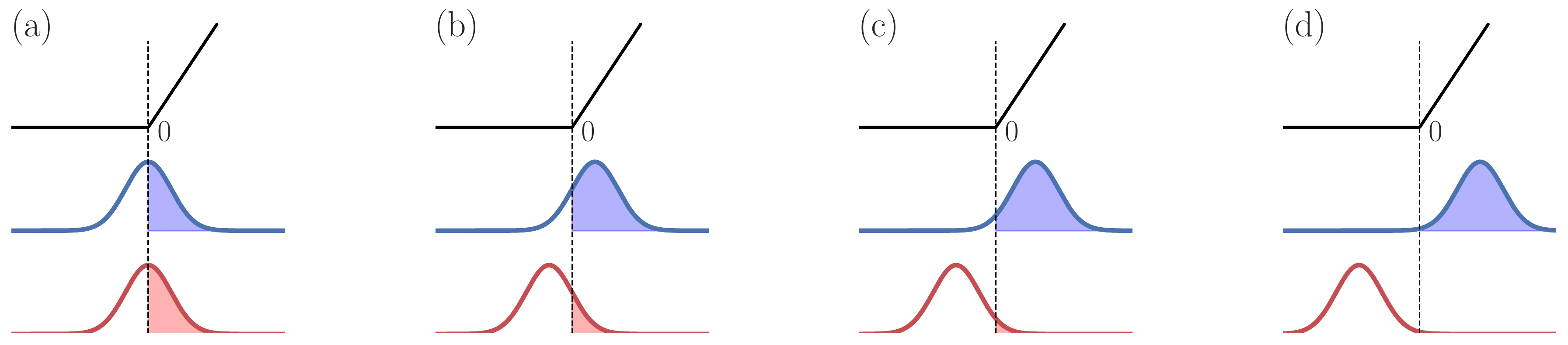}
  \caption{\textbf{Channel-wise collapse induces channel-wise linearity}. Each subplot shows $\phi=\relu$ (black activation function) as well as two channel-wise distributions (blue and red distributions) positioned symmetrically around $0$ with $\frac{\pwr\uc(\ty\pl)-\pwrone\uc(\ty\pl)}{\pwr\uc(\ty\pl)}=1, \frac{1}{2}, \frac{1}{4}, \frac{1}{8}$ for (a), (b), (c), (d), respectively. When progressing from (a) to (d), the part of the distribution corresponding to active $\relu$ (shaded region) becomes either overly dominant (blue distribution) or negligible (red distribution). In either case, \mbox{the channel-wise distribution ends up concentrated on only one side of piece-wise linearity.}} 
  \label{fig:channelwise_linearity} 
\end{figure}

\paragraph{Revisiting BN's preservation of expressivity.} 
When BN is used, $\pwrone\uc(\rvy\pl)=0$ implies for all $c$ that the terms $\pwrtwo\uc(\rvy\pl),\pwrthree\uc(\rvy\pl),\pwrfour\uc(\rvy\pl)$ sum to one. Apart from that, BN does not impose any particular constraints on the relative proportions of each term into the sum. This means that the relative proportions of $\pwr^{(i)}\uc(\rvy\pl)$ and $\pwr^{(i)}(\rvy\pl)$ for $i\in \{2,3,4\}$ are free to evolve as naturally dictated by the task and the optimizer during learning.

This absence of constraints seems sensible. Indeed, imposing constraints on these relative proportions would \emph{alter the expressivity}, which would not have any obvious justification in general and could even be detrimental in some cases, as we discuss in \Secref{sec:failure_modes}.

\section{Failure modes with batch-independent normalization}
\label{sec:failure_modes}
With our theoretical framework in hand, we now turn to showing that the prototypical batch-independent norms are subject to failures modes opposite to BN's beneficial properties.

In the case of LN, the failure mode does not manifest in an absolute sense but rather as a \say{soft} inductive bias, i.e. as a preference or a favoring in the hypothesis space. This \say{soft} inductive bias is quantified by \Thmref{thm:ln_collapse} in the context of networks with random model parameters.

\begin{Definition}[random net]
\label{def:random}
We define a \say{random net} as a neural network having an input $\rvx$ sampled from the dataset $\D$ and implementing \Eqref{eq:propagation1}, (\ref{eq:propagation2}), (\ref{eq:propagation3}) in every layer up to depth $L$, with the components of $\vomega\pl$,$\vgamma\pl$,$\vbeta\pl$ at every layer $l\in \{1,\dots,L\}$ sampled i.i.d. from the fixed distributions $\nu_{\vomega}$,$\nu_{\vgamma}$,$\nu_{\vbeta}$ (up to a fan-in's square root scaling for $\vomega\pl$). 

In such networks, we assume that: (i) none of the inputs in the dataset $\D$ are identically zero; (ii) $\nu_{\vomega}$, $\nu_{\vbeta}$, $\nu_{\vgamma}$ have well-defined moments, with strictly positive associated root mean squares $\omega, \gamma, \beta>0$; (iii) $\nu_{\vomega}$, $\nu_{\vbeta}$ are symmetric around zero.
\end{Definition}

\begin{Theorem}[layer-normalized networks collapse (informal)]\emph{[\ref{sec:proof_ln_collapse}]}
\label{thm:ln_collapse} Fix a layer $l\in \{1,\dots,L\}$ and $\nu_{\vomega}$,~$\nu_{\vbeta}$,~$\nu_{\vgamma}$,~$\D$ in \Defref{def:random}. Further suppose $\Norm=\LN$ and suppose that the convolution of \Eqref{eq:propagation1} uses periodic boundary conditions.

Then for random nets of \Defref{def:random}, it holds when widths are large enough that
 \begin{align}
 \pwr(\rvy\pl) - \pwrone(\rvy\pl) & \lesssim \rho\plm,  &  \quad
\pwr(\rvy\pl) & \simeq 1, \label{eq:ln_collapse1} 
 \end{align} 
where $\rho \equiv \gamma^2 /(\gamma^2+\beta^2) <1$, and $\lesssim$ and $\simeq$ denote inequality and equality up to arbitrarily small constants with probability arbitrarily close to 1 when $\min_{1\leq k \leq l} C\uk$ is large enough.
\end{Theorem}

\paragraph{Discussion on LN's failure mode.} \Thmref{thm:ln_collapse} implies that, with high probability, $\rvy\pl$ is subject to \emph{channel-wise collapse} in deep layers ($l\gg 1$) with $\pwr(\rvy\pl) - \pwrone(\rvy\pl) \ll \pwr(\rvy\pl)$. This means that $\ty\pl$ (that is only one affine transformation away from $\rvy\pl$) is likely to have its channel-wise distributions collapsed with $\pwr\uc(\ty\pl) - \pwrone\uc(\ty\pl) \ll \pwr\uc(\ty\pl)$ for most $c$. The nonlinearity $\phi$ acting on $\ty\pl$ is then likely to be close to linear with respect to $\ty\pl$'s channel-wise distributions [\ref{sec:channelwise_linearity}]. Being close to channel-wise linear in deep layers, \mbox{layer-normalized networks are unable to effectively use their whole depth.}

Since the inequality $\lesssim$ can be replaced by an equality $\simeq$ in the case $\phi=\identity$ of \Thmref{thm:ln_collapse}~[\ref{sec:proof_ln_collapse_identity}], the aggravation at each layer $l$ of the upper bound of \Eqref{eq:ln_collapse1} does not stem from the activation function itself but rather by the preceding affine transformation.  The phenomenon of channel-wise collapse --- also known under the terms of \say{domain collapse} \citep{Philipp18} or \say{elimination singularity} \citep{Qiao19} --- is therefore not only induced by a \say{mean-shifting} activation function such as $\phi=\relu$ \citep{Page19,Brock21a}, but also by the injection of non-centeredness through the application of the channel-wise shift parameter $\vbeta\pl$ at each layer $l$. The fact that the general case of positive homogeneous $\phi$ is upper bounded by the case $\phi=\identity$ in \Eqref{eq:ln_collapse1} still means that the choice $\phi=\relu$ can only be an aggravating factor.

Crucially, in the context of random nets of large widths, LN's operation at each layer $l$ does not compensate this \say{mean shift}. This comes from the fact that LN's mean and variance statistics can be approximated by zero and a constant value independent of $\rvx$, respectively. This means that LN's operation can be approximated by a layer-wise constant scaling independent of $\rvx$.\footnote{In this view, we expect layer-normalized networks to be also subject to a phenomenon of increasingly imbalanced channels with depth \citep{Bjorck18,Pennington17,Daneshmand20}.}

The predominance of LN's failure mode in the hypothesis space --- implied by its predominance in random nets --- is expected to have at least two negative effects on the actual learning and final performance: (i) being expected along the training trajectory and being associated with reduced effective capacity, the failure mode is expected to cause degraded performance on the training loss; (ii) even if avoided to some extent along the training trajectory, the failure mode is still expected in the vicinity of this training trajectory, implying an ill-conditioning of the loss landscape \citep{Page19,Ghorbani19} and a prohibition of large learning rates that could have led otherwise to generalization benefits \cite{Masters18}.

After detailing LN's failure mode, we now detail IN's failure mode.
\begin{Theorem}[instance-normalized networks lack variability in instance statistics]\emph{[\ref{sec:in_variability}]}
\label{thm:in_variability}
Fix a layer $l\in\{1,\dots,L\}$ and lift any assumptions on $\phi$. Further suppose $\Norm=\IN$, with \Eqref{eq:propagation2} having nonzero denominator at layer $l$ for all inputs and channels.

Then it holds that 
\begin{itemize}
\item $\rvy\pl$ is normalized in each channel $c$ with 
 \begin{align*}
\pwrone\uc(\rvy\pl) & = 0, &  \quad
\pwr\uc(\rvy\pl) & =1; 
 \end{align*}
\item $\rvy\pl$ lacks variability in instance statistics in each channel $c$ with
 \begin{align*}
 \pwrtwo\uc(\rvy\pl)  & = 0, & \quad
 \pwrthree\uc(\rvy\pl) & = 1, & \quad
\pwrfour\uc(\rvy\pl) & = 0.
 \end{align*}
\end{itemize}
\end{Theorem}

\paragraph{Discussion on IN's failure mode.} 
We see in \Thmref{thm:in_variability} that $\rvy\pl$'s power decomposition with IN is constrained to be such that $\pwrone(\rvy\pl),\pwrtwo(\rvy\pl),\pwrfour(\rvy\pl)=0$ and $\pwrthree(\rvy\pl)=1$. While the constraints on $\pwrone(\ty\pl)$, $\pwrthree(\ty\pl)$ are removed by the affine transformation between $\rvy\pl$ and $\ty\pl$, the constraints on $\pwrtwo(\ty\pl)$, $\pwrfour(\ty\pl)$, on the other hand, remain even after the affine transformation. These constraints on $\pwrtwo(\ty\pl)$, $\pwrfour(\ty\pl)$ translate into $2C\ul$ fixed constraints in activation space that apply to each $\ty\pl\in\R^{H\times W \times C\ul}$ associated \mbox{with each choice of input $\rvx$ in the dataset $\D$ \cite{Huang20}.}

Such constraints on $\ty\pl$ are symptomatic of an alteration of expressivity. They notably entail that some network mappings that can be expressed without normalization cannot be expressed with IN. One such example is the identity mapping [\ref{sec:alteration_IN}]. Another such example is a network mapping that would provide in channel $c$ through $\ty\pl_{\alpha,c}$, just before the nonlinearity, a detector of a given concept at position $\alpha$ in the input $\rvx$. With IN, the lack of variability in instance statistics implies that the mean $\mu_{\rvx,c}(\ty\pl)$ and standard deviation $\sigma_{\rvx,c}(\ty\pl)$ of the feature map in channel $c$ are necessarily constant for all $\rvx$, equal to $\vbeta\uc$ and $\vgamma\uc$, respectively. This does not allow to express for some inputs $\rvx$ the presence of the concept at some position $\alpha$: $\mu_{\rvx,c}(\ty\pl)>0$, $\sigma_{\rvx,c}(\ty\pl)>0$; and for other inputs $\rvx$ the absence of the concept: $\mu_{\rvx,c}(\ty\pl)=0$, $\sigma_{\rvx,c}(\ty\pl)=0$.

\newpage
This latter example is not just anecdotal. Indeed, it is accepted that networks trained on high-level conceptual tasks have their initial layers related to low-level features and their deep layers related to high-level concepts \citep{Ulyanov16a,Gatys16}. This view explains the success of IN on the low-level task of style transfer with fixed \say{style} input, IN being then incorporated inside a generator network that only acts on the low-level features of the \say{content} input \citep{Ulyanov16b,Ulyanov17,Dumoulin17,Huang17b}. On high-level conceptual tasks, on the other hand, this view hints at a harmful tension between IN's constraints and the requirement of instance variability to express high-level concepts in deep layers. In short, with IN not only is the expressivity altered, \mbox{but the alteration of expressivity results in the exclusion of useful network mappings.}

\begin{figure}[t]
\centering
\includegraphics[width=1.\linewidth]{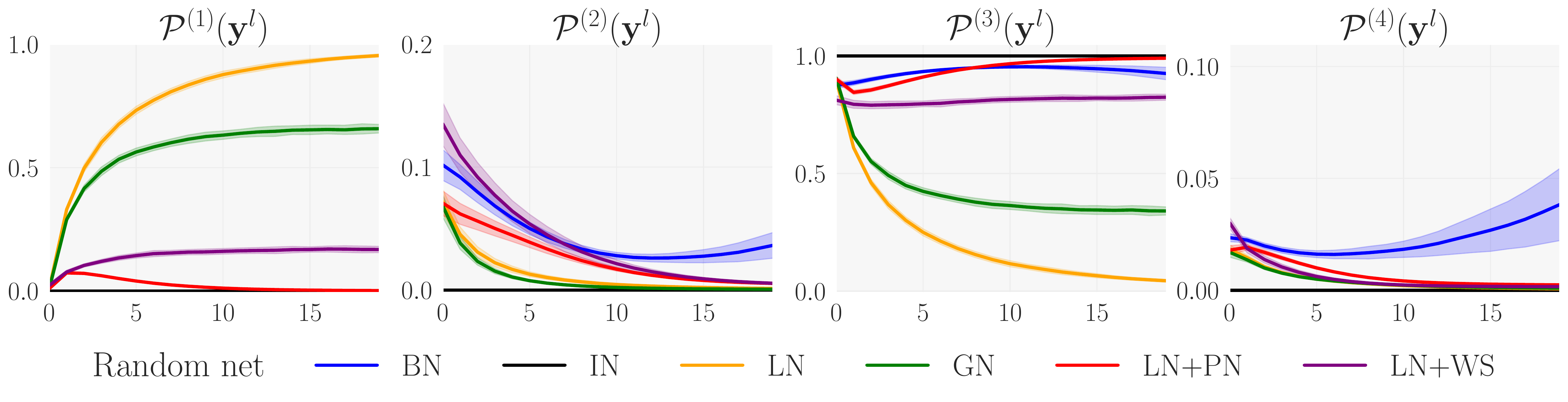} \\
\vspace{3pt}
\includegraphics[width=1.\linewidth]{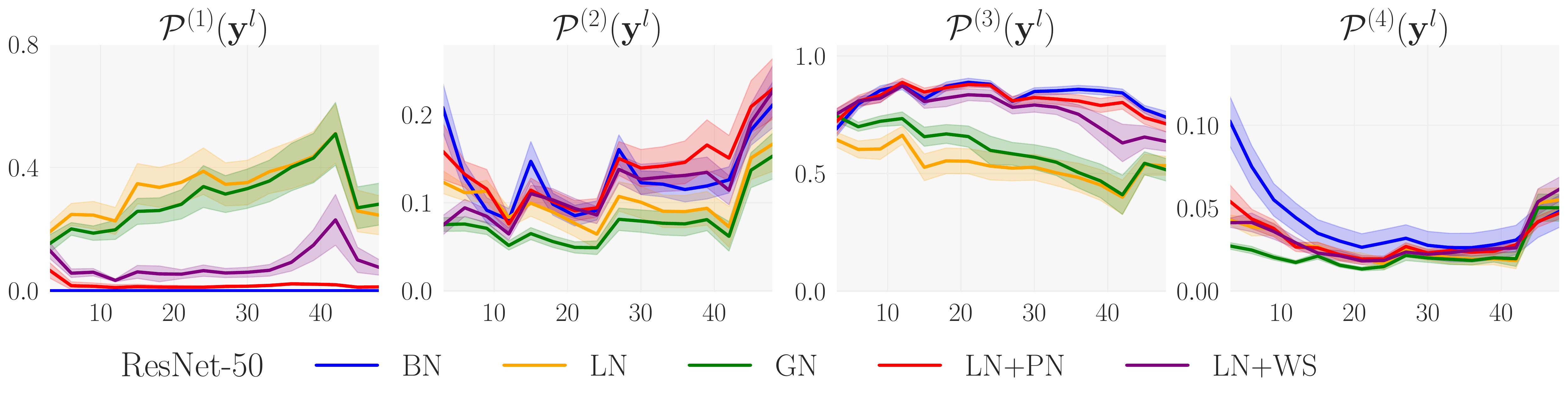}
\captionof{figure}{\label{fig:powers} \textbf{``Power plots''}. The power decomposition of $\rvy\pl$ is shown as a function of the depth $l$ for BN and different batch-independent norms: IN, LN, GN, LN+PN, LN+WS. Top row: Random net of \Defref{def:random} with $\phi=\relu$, widths $C\ul=1024$, kernel sizes $K\ul=3$, convolutions using periodic boundary conditions and components of $\vbeta\pl$, $\vgamma\pl$ sampled i.i.d. from $\mathcal{N}(0,0.2^2)$ and $\mathcal{N}(1,0.2^2)$, respectively. Bottom row: ResNet-50 (v2) throughout 100 epochs of training on ImageNet (IN is not shown in this row due to numerical stability issues). \mbox{Further experimental details are reported in \Appendref{sec:details_power_plots}.}}
\end{figure}

\paragraph{Failure modes with GN.} 
Group Normalization is a middle ground between the two extremes of LN ($G=1$ group) and IN ($G=C\ul$ groups at layer $l$). Networks with GN are consequently affected by both failure modes of \Thmref{thm:ln_collapse} and \Thmref{thm:in_variability}, but to a lesser extent than networks with LN for the failure mode of \Thmref{thm:ln_collapse} and IN for the failure mode of \Thmref{thm:in_variability}. On the one hand, since GN becomes equivalent to a constant scaling when group sizes become large, networks with GN are likely to be subject to channel-wise collapse. On the other hand, since GN can be seen as removing a fraction --- with an inverse dependence on the group size --- of $\smash{\pwrtwo(\rvy\pl)}$, $\smash{\pwrfour(\rvy\pl)}$ in between each layer, networks with GN are likely to lack variability in instance statistics.

The balance struck by GN between the two failure modes of LN and IN could still be beneficial, which would explain GN's superior performance in practice. It makes sense intuitively that being weakly subject to two failure modes is preferable over being strongly subject to one failure mode.

\paragraph{Experimental validation.} The \say{power plots} of \Figref{fig:powers} show the power decomposition of $\rvy\pl$ as a function of the depth $l$ in both a random net of \Defref{def:random} and ResNet-50 (v2) \citep{He16b} trained on ImageNet.\footnote{To ensure that activation steps are directly preceded by normalization steps (cf \Secref{sec:pn}), we always use v2 instantiations of ResNets and instantiations of ResNeXts having the same reordering of operations inside residual blocks as ResNets v2 (cf \Appendref{sec:experimental_setup}).} Looking at the cases of BN and IN, LN, GN in these power plots, we confirm that: (i) unlike networks with BN and IN, networks with LN, and to a lesser extent GN, are subject to channel-wise collapse as depth increases (see \Appendref{sec:details_power_plots_randomnets} for a precise verification of \Thmref{thm:ln_collapse}); (ii) networks with IN, and to a lesser extent GN, lack variability in instance statistics.

\section{Proxy Normalization}
\label{sec:pn}
With the goal of remedying the failure modes of \Secref{sec:failure_modes}, we now introduce our novel technique \emph{\say{Proxy Normalization} (PN)} that is at the core of our batch-independent normalization approach.

PN is incorporated into the neural network by replacing the activation step of \Eqref{eq:propagation3} by the following \say{proxy-normalized activation step} \mbox{(cf \Figref{fig:pna_schematic} and the practical implementation of \Appendref{sec:model_code}):}
\begin{align}
\tz\pl & = \PNAct(\rvy\pl), &\quad \forall \alpha,c: \quad \PNAct(\rvy\pl)_{\alpha,c} 
 = \frac{\phi\big(\vgamma\pl\uc \rvy\pl_{\alpha,c} + \vbeta\pl\uc\big) - \E_{Y\pl\uc}\big[\phi\big(\vgamma\pl\uc Y\pl\uc + \vbeta\pl\uc)\big]}{\sqrt{\Var_{Y\pl\uc}\big[\phi\big(\vgamma\pl\uc Y\pl\uc + \vbeta\pl\uc)\big]+\epsilon}},  \label{eq:propagation4}
\end{align}
where $\epsilon\geq0$ is a numerical stability constant and $Y\pl\uc$ is a Gaussian \say{proxy} variable of mean $\tilde{\vbeta}\pl\uc$ and variance $(1+\tilde{\vgamma}\pl\uc)^2$ depending on the additional parameters $\tilde{\vbeta}\pl, \tilde{\vgamma}\pl \in \R^{C\ul}$ of PN. Unless otherwise stated, we let $\tilde{\vbeta}\pl$,~$\tilde{\vgamma}\pl$ be nonzero but still subject to weight decay and thus close to zero. We show in \Appendref{sec:pn_noadditional} that it is also effective to let $\tilde{\vbeta}\pl$,~$\tilde{\vgamma}\pl$ be strictly zero and $Y\pl\uc \sim \mathcal{N}(0,1)$. 

If we assume (as hinted by \Secref{sec:failure_modes}) that only the affine transformation and the activation function~$\phi$ (i.e. the activation step) play a role in the aggravation at each layer $l$ of channel-wise collapse and channel imbalance, \mbox{then PN provides the following guarantee of channel-wise normalization.}

\begin{Theorem}[guarantee of channel-wise normalization in proxy-normalized networks (informal)]\emph{[\ref{sec:pn_iterative}]}
\label{thm:pn_iterative}
Fix a layer $l\in\{1,\dots,L\}$ and lift any assumptions on $\phi$ and $\rvx$'s distribution. Further suppose that the neural network implements \Eqref{eq:propagation1},~(\ref{eq:propagation2}),~(\ref{eq:propagation4}) at every layer up to depth $L$, with $\epsilon=0$ and \Eqref{eq:propagation2},~(\ref{eq:propagation4}) having nonzero denominators for all layers, inputs and channels.

Then both $\rvy\pl$ and $\tz\pl$ at layer $l$ are channel-wise normalized if the following conditions hold:
\begin{enumerate}[label=(\roman*)]
\item $\tz\plm$ at layer $l-1$ is channel-wise normalized;
\item The convolution and normalization steps at layer $l$ do not cause any aggravation of channel-wise collapse and channel imbalance;
\item $\rvy\pl$ at layer $l$ is channel-wise Gaussian and PN's additional parameters $\tilde{\vbeta}\pl$, $\tilde{\vgamma}\pl$ are zero.
\end{enumerate}
\end{Theorem}
\begin{minipage}{0.70\linewidth}
\paragraph{Our batch-independent approach: LN+PN or GN+PN.} At this point, we crucially note that PN: (i) is batch-independent; (ii) does not cause any alteration of expressivity. \emph{This leads us to adopt a batch-independent normalization approach that uses either LN or GN (with few groups) in the normalization step and that replaces the activation step by the proxy-normalized activation step (+PN)}. With such a choice of normalization step, we guarantee three of the benefits detailed in \Tableref{table:properties}: \say{scale invariance}, \say{control of activation scale} and \say{preservation of expressivity}. With the proxy-normalized activation step, we finally guarantee the fourth benefit of \say{avoidance of collapse} without \mbox{compromising any of the benefits provided by the normalization step.} \\[-5pt]

\paragraph{Experimental validation.} 
We confirm in \Figref{fig:powers} \emph{that BN's behavior is emulated in a fully batch-independent manner with our approach, LN+PN or GN+PN}. Indeed, the power plots of networks with LN+PN resemble the power plots of networks with BN. As desired, PN remedies LN's failure mode without incurring IN's failure mode. 
\end{minipage}
\hfill
\begin{minipage}{0.25\linewidth}
 \centering
\includegraphics[width=0.91\linewidth]{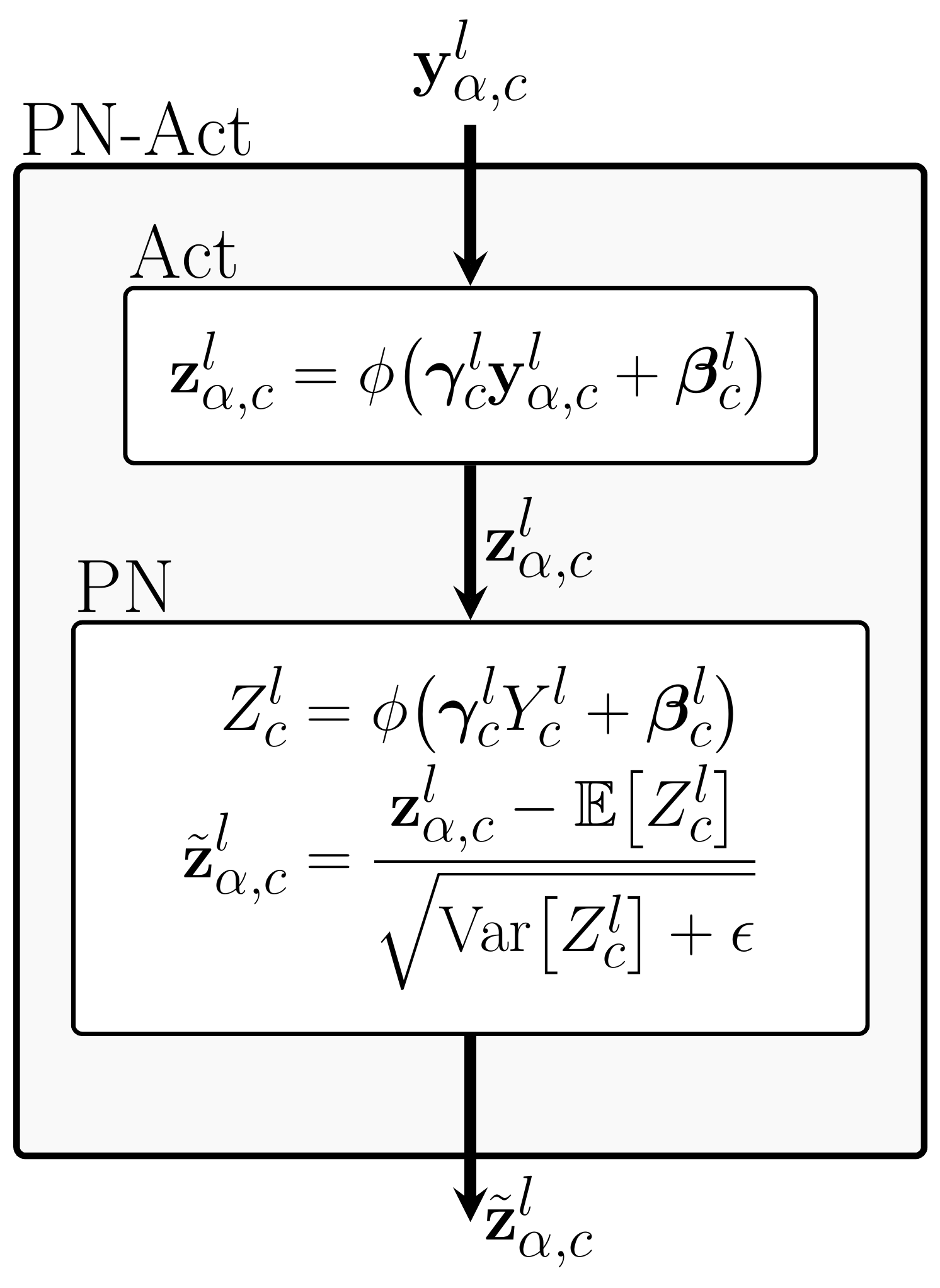}
\captionof{figure}{\label{fig:pna_schematic} \textbf{Plugging PN}. PN is plugged into the activation step $\Act$ to yield the proxy-normalized activation step $\PNAct$.}
\end{minipage}

\paragraph{Approach strength 1: Normalizing beyond initialization.} Our batch-independent approach maintains channel-wise normalization throughout training (cf \Figref{fig:powers}). In contrast, many alternative approaches, either explicitly or implicitly, focus on initialization \cite{Arpit16,Klambauer17,Huang17a,Qiao19,Rao20,Brock21a,Brock21b}. Centered Weight Normalization \citep{Huang17a}, WS \cite{Qiao19,Brock21a,Brock21b} or PreLayerNorm \cite{Rao20} notably rely on the implicit assumption that different channels have the same channel-wise means after the activation function $\phi$. While valid in networks at initialization with $\vbeta\pl=0$ and $\vgamma\pl=1$ (cf \Appendref{sec:details_power_plots_resnets}), this assumption becomes less valid as the affine transformation starts deviating from the identity. We see in \Figref{fig:powers} that networks with LN+WS are indeed less effective in maintaining channel-wise normalization, both in networks with random $\vbeta\pl$, $\vgamma\pl$ (top row) and in networks considered throughout training (bottom~row). Such a coarser channel-wise normalization might in turn lead to a less effective use of model capacity and a degradation of the conditioning of the loss landscape \citep{Page19,Ghorbani19}.

\paragraph{Approach strength 2: Wide applicability.} Our batch-independent approach matches BN with consistency across choices of model architectures, model sizes and activation functions (cf \Secref{sec:results}). Its only restriction, namely that activation steps should be immediately preceded by normalization steps for $\rvy\pl$ and its associated proxy $Y\pl$ to be at the same scale, has easy workarounds. Applicability restrictions might be more serious with alternative approaches: (i) alternative approaches involving a normalization in weight space \citep{Salimans16,Huang17a,Qiao19,Ruff19,Brock21a} might be ill-suited to architectures with less \say{dense} kernels such as EfficientNets \citep{Brock21a}; (ii) approaches involving the tracking of activation statistics \citep{Arpit16,Laurent17,Shekhovtsov18b} might be nontrivial to apply to residual networks \cite{Shang17,Gitman17}; (iii) approaches involving a change of \mbox{activation function \citep{Klambauer17,Singh19b,Liu20} might precisely restrict the choice of activation function.}

\paragraph{Approach strength 3: Ease of implementation.} Our approach is straightforward to implement when starting from a batch-normalized network. It simply requires: (i) replacing all BN steps with LN or GN steps, and (ii) replacing all activation steps with proxy-normalized activation steps. The proxy-normalized \mbox{activation steps themselves are easily implemented (cf \Appendref{sec:model_code}).} 

\section{Results}
\label{sec:results}
We finally assess the practical performance of our batch-independent normalization approach. While we focus on ImageNet \citep{Deng09} in the main text of this paper, we report in \Appendref{sec:details_experiments} some additional results on CIFAR \citep{Krizhevsky09}. \mbox{In \Appendref{sec:details_experiments}, we also provide all the details on our experimental setup.}

\paragraph{Choices of regularization and batch size.} As mentioned in \Secref{sec:bn_properties}, we perform all our experiments with different degrees of regularization to disentangle normalization's effects from regularization's effects. We detail all our choices of regularization in Appendices \ref{sec:experimental_setup} and \ref{sec:alternative_regul}.

In terms of batch size, we set: (i) the \say{global batch size} in between weight updates to the same value independently of the choice of norm; (ii) the \say{normalization batch size} to a near-optimal value with BN. These choices enable us to be conservative when concluding on a potential advantage of our approach over BN at small batch size. Indeed, while the performance of batch-independent approaches would remain the same or slightly improve at small batch size \cite{Goyal17,Masters18}, the performance of BN would eventually \mbox{degrade due to a regularization that eventually becomes excessive \cite{Ioffe17,Wu18,Ying18,Singh19b,Summers20,Yan20}.} 

\vspace{4pt}
\begin{minipage}{0.53\linewidth}
\paragraph{Effect of adding PN.} We start as a first experiment by analyzing the effect of adding PN on top of various norms in ResNet-50 (RN50). As visible in \Tableref{table:rn50}, PN is mostly beneficial when added on top of LN or GN with a small number of groups $G$. The consequence is that the optimal $G$ shifts to lower values in GN+PN compared to GN. This confirms the view that \mbox{PN's benefit lies in addressing LN's failure mode}.\\[-5pt]

It is also visible in \Tableref{table:rn50} that PN does not provide noticeable benefits to BN. This confirms again the view that PN's benefit lies in addressing the problem --- not present with BN --- of channel-wise collapse. Importantly, since PN does not entail effects other than normalization that could artificially boost the performance, \emph{GN+PN can be compared in a fair way to BN when assessing the effectiveness of normalization}. 
\end{minipage}
\hfill
\begin{minipage}{0.42\linewidth}
   \centering
  \captionof{table}{\label{table:rn50} \textbf{Effect of adding PN}. ResNet-50 is trained on ImageNet with BN and LN, GN, GN+WS with $G$ groups, either without or with PN added on top (plain vs. PN). Results are formatted as X~/~Y with X, Y the validation accuracies (\%) without and with extra regularization, respectively.}
\scalebox{0.86}
{\begin{tabular}{lccc}
\toprule
       &  &  \multicolumn{2}{c}{RN50} \\
      &  $G$ &                                   plain &                                       +PN \\
\midrule
BN      &    &  \meanbf{76.3}{0.1} / \mean{75.8}{0.2} &  \mean{76.2}{0.1} / \mean{76.0}{0.1} \\
\midrule
LN      &  1  &  \mean{74.5}{0.0} / \mean{74.6}{0.1} &  \mean{75.9}{0.1} / \mean{76.5}{0.0} \\
GN      &  8  &  \mean{75.4}{0.1} / \mean{75.4}{0.1} &  \mean{76.3}{0.1} / \meanbf{76.7}{0.0} \\
GN      &  32  &  \mean{75.4}{0.1} / \mean{75.3}{0.1} &  \mean{75.8}{0.2} / \mean{76.1}{0.1} \\
\midrule
GN+WS   &  8  &  \mean{76.6}{0.0} / \mean{76.7}{0.1} &  \mean{76.8}{0.1} / \meanbf{77.1}{0.1} \\
\bottomrule
\end{tabular}}
\end{minipage}
\vspace{4pt}

This is unlike WS which has been shown to improve BN's performance \citep{Qiao19}. In our results of \Tableref{table:rn50}, the high performance of GN+WS without extra regularization and the fact that PN still provides benefits to GN+WS suggests that: (i) on top of its normalization benefits, WS induces a form of regularization; (ii) GN+WS is still not fully optimal in terms of normalization.

\paragraph{GN+PN vs. BN.} Next, we turn to comparing the performance of our batch-independent approach, GN+PN, to that of BN across a broad range of models trained on ImageNet. As visible in \Figref{fig:results_all} and Tables~\ref{table:rn_rnx},~\ref{table:en}, GN+PN outperforms BN in ResNet-50 (RN50) and ResNet-101 (RN101) \citep{He16b}, matches BN in \mbox{ResNeXt-50} (RNX50) and ResNeXt-101 (RNX101) \citep{Xie17}, and matches BN in EfficientNet-B0 \mbox{(EN-B0)} and EfficientNet-B2 (EN-B2), both in the original variant with depthwise convolutions and expansion ratio of 6 \citep{Tan19} and in an approximately parameter-preserving variant (cf \Appendref{sec:experimental_setup}) with group convolutions of group size 16 and expansion ratio of 4 \citep{Masters21}. \emph{In short, our batch-independent normalization approach, GN+PN, matches BN not only in behavior but also in performance.}

With regard to matching BN's performance with alternative norms, various positive results have been reported in ResNets and ResNeXts \citep{Chiley19,Yan20,Luo19b,Luo19c,Liu20,Brock21a,Shao20} but only a limited number in EfficientNets \citep{Liu20}. In EfficientNets, we are notably not aware of any other work showing that BN's performance can be matched with a batch-independent approach. As a confirmation, we assess the performance of various existing batch-independent approaches: GN \citep{Wu18}, GN+WS \citep{Qiao19}, Evo-S0 \citep{Liu20}, FRN+TLU \mbox{\citep{Xiang17,Singh19b}}. Unlike GN+PN, none of these approaches is found in \Tableref{table:en} to match BN with consistency.

\paragraph{Normalization and regularization.} Our results suggest that while an efficient normalization is not sufficient in itself to achieve good performance on ImageNet, it is still a necessary condition, together with regularization. In our results, it is always with extra regularization that GN+PN yields the most benefits. Importantly, the fact that GN+PN consistently leads to large improvements in training accuracy (cf \Appendref{sec:training_acc_imagenet}) suggests that \emph{additional benefits would be obtained on larger datasets without the requirement of relying on regularization} \citep{Kolesnikov20,Brock21b}.

\vspace{4pt}
\begin{minipage}{0.34\linewidth}
 \centering
\includegraphics[width=1.\linewidth]{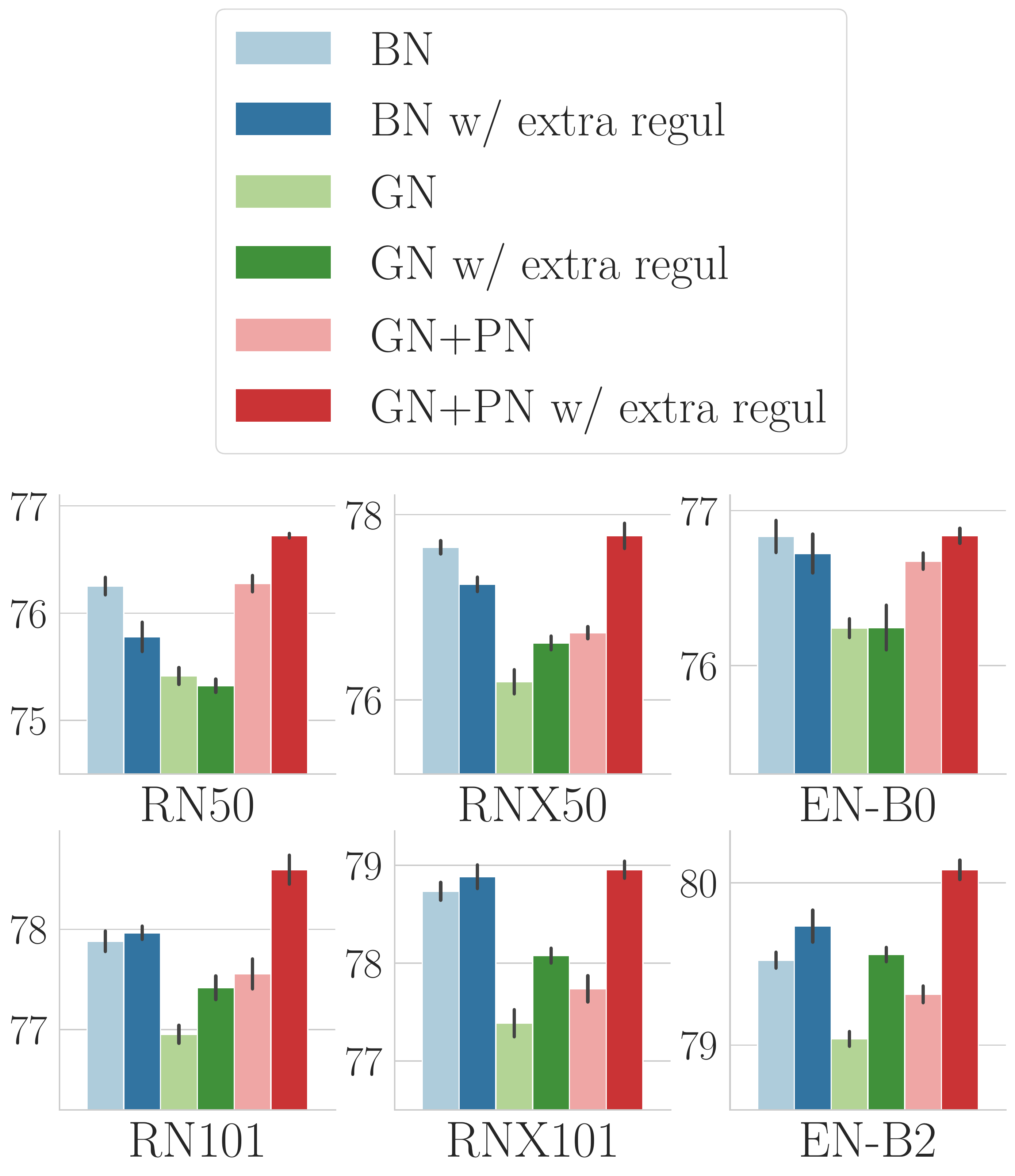}
\captionof{figure}{\label{fig:results_all} \textbf{BN vs. GN, GN+PN}. Validation accuracies (\%) of ResNets, ResNeXts and EfficientNets trained on ImageNet with BN and GN, GN+PN, without and with extra regularization. EfficientNets are considered in the variant \mbox{with group convolutions \citep{Masters21}.}}
\end{minipage}
\hfill
\begin{minipage}{0.61\linewidth}
   \centering
  \captionof{table}{\label{table:rn_rnx} \textbf{BN vs. GN, GN+PN}. ResNets and ResNeXts are trained on ImageNet with BN and GN, GN+PN. Results are formatted as in \Tableref{table:rn50}.}
\scalebox{0.85}{
\begin{tabular}{lcccc}
\toprule
  &                                     RN50 &                                    RN101 &                                    RNX50 &                                   RNX101 \\
\midrule
BN    &  \meanbf{76.3}{0.1} / \mean{75.8}{0.2} &  \mean{77.9}{0.1} / \meanbf{78.0}{0.1} &  \meanbf{77.6}{0.1} / \mean{77.2}{0.1} &  \mean{78.7}{0.1} / \meanbf{78.9}{0.1} \\
\midrule
GN    &  \mean{75.4}{0.1} / \mean{75.3}{0.1} &  \mean{77.0}{0.1} / \mean{77.4}{0.1} &  \mean{76.2}{0.2} / \mean{76.6}{0.1} &  \mean{77.4}{0.2} / \mean{78.1}{0.1} \\
GN+PN &  \mean{76.3}{0.1} / \meanbf{76.7}{0.0} &  \mean{77.6}{0.2} / \meanbf{78.6}{0.2} &  \mean{76.7}{0.1} / \meanbf{77.8}{0.2} &  \mean{77.7}{0.2} / \meanbf{79.0}{0.1} \\
\bottomrule
\end{tabular}} 
\\[15pt]
\centering
  \captionof{table}{\label{table:en} \textbf{BN vs. batch-independent approaches}. EfficientNets are trained on ImageNet with BN and various batch-independent approaches. \mbox{Results are formatted as in \Tableref{table:rn50}.}}
\scalebox{0.85}{
\begin{tabular}{lcccc}
\toprule
  &   \multicolumn{2}{c}{depthwise convs \citep{Tan19}} & \multicolumn{2}{c}{group convs \citep{Masters21}} \\
    \cmidrule(lr){2-3} \cmidrule(lr){4-5} 
  &                                 EN-B0 &                                 EN-B2 &                                EN-B0 &                                EN-B2 \\
\midrule
BN      &  \mean{76.9}{0.1} / \meanbf{77.2}{0.1} &  \mean{79.4}{0.0} / \meanbf{80.0}{0.0} &  \meanbf{76.8}{0.1} / \mean{76.7}{0.2} &  \mean{79.5}{0.1} / \meanbf{79.7}{0.1} \\
\midrule
GN      &  \mean{76.2}{0.1} / \mean{76.2}{0.1} &  \mean{78.9}{0.1} / \mean{79.4}{0.1} &  \mean{76.2}{0.1} / \mean{76.2}{0.2} &  \mean{79.0}{0.1} / \mean{79.6}{0.1} \\
GN+PN   &  \mean{76.8}{0.0} / \meanbf{77.0}{0.1} &  \mean{79.3}{0.1} / \meanbf{80.0}{0.1} &  \mean{76.7}{0.1} / \meanbf{76.8}{0.1} &  \mean{79.3}{0.1} / \meanbf{80.1}{0.1} \\
Evo-S0  &  \mean{75.8}{0.1} / \mean{75.8}{0.2} &  \mean{78.5}{0.1} / \mean{78.7}{0.1} &  \mean{76.2}{0.0} / \mean{76.5}{0.1} &  \mean{78.9}{0.0} / \mean{79.6}{0.0} \\
GN+WS   &  \mean{74.2}{0.1} / \mean{74.1}{0.1} &  \mean{77.8}{0.0} / \mean{77.8}{0.1} &  \mean{76.2}{0.1} / \mean{76.3}{0.1} &  \mean{79.2}{0.1} / \mean{79.4}{0.1} \\
FRN+TLU &  \mean{75.7}{0.1} / \mean{75.7}{0.2} &  \mean{78.4}{0.1} / \mean{78.8}{0.1} &  \mean{74.9}{0.2} / \mean{75.1}{0.1} &  \mean{78.2}{0.1} / \mean{78.6}{0.1} \\
\bottomrule
\end{tabular}}
\end{minipage}
\vspace{2pt}

\section{Summary and broader impact}
We have introduced a novel framework to finely characterize the various neural network properties affected by the choice of normalization. Using this framework, we have shown that while BN's beneficial properties are not retained when solely using the prototypical batch-independent norms, they are retained when combining some of these norms with the technique hereby introduced of Proxy Normalization. We have demonstrated on an extensive set of experiments that our batch-independent normalization approach \mbox{consistently matches BN in both behavior and performance.}

The main implications of this work could stem from the unlocked possibility to retain BN's normalization benefits while removing batch dependence. Firstly, our approach could be used to retain BN's normalization benefits while alleviating the burden of large activation memory stemming from BN's requirement of sufficiently large batch sizes. This is expected to be important in memory-intensive applications such as object detection or image segmentation, but also when using A.I. accelerators that leverage local memory to provide extra acceleration and energy savings in exchange for tighter memory constraints. Secondly, our approach could be used to retain BN's normalization benefits while avoiding BN's regularization when the latter is detrimental. As discussed in \Secref{sec:results}, this is expected \mbox{to be important in the context --- that will likely be prevalent in the future --- of large datasets.}

\begin{ack}
We are thankful to Simon Knowles, Luke Hudlass-Galley, Luke Prince, Alexandros Koliousis, Anastasia Dietrich and the wider research team at Graphcore for the useful discussions and feedbacks. We are also thankful to the anonymous reviewers for their insightful comments that helped improve the paper.
\end{ack}

\bibliography{proxy_norm}
\bibliographystyle{unsrt}


\newpage

\appendix 
\section{Experimental details}
\label{sec:details_experiments}

\subsection{Experimental setup}
\label{sec:experimental_setup}

\paragraph{Architectures\textnormal{.}} As stated in \Secref{sec:pn}, we use v2 instantiations of ResNets and instantiations of ResNeXts having the same reordering of operations inside residual blocks as ResNets v2. Practically, our instantiations of ResNeXts are obtained by starting from ResNets v2 and applying the same changes in bottleneck widths and number of groups in $3\times3$ convolutions as the changes yielding ResNeXts v1 from ResNets~v1, using a cardinality $32$ and a dimension $4$ \cite{Xie17}. 

As also stated in \Secref{sec:results}, we consider two variants of EfficientNets: (i) the original variant with one depthwise convolution per MBConv block and with expansion ratio of 6 \citep{Tan19}, and (ii) a variant with each depthwise convolution replaced by a group convolution of group size 16 and with expansion ratio of~4 \citep{Masters21}. Compared to the original variant, the variant with group convolutions has roughly the same number of parameters and slightly more floating point operations (FLOPs) (cf \Tableref{table:variants}), but it is executed more efficiently on common A.I. accelerators. Interestingly, the fact that GN+WS does not perform well even in this variant (cf Tables~\ref{table:en},~\ref{table:en_detailed} and \Figref{fig:results_en}) suggests that the problem related to the removal of degrees of freedom by WS goes beyond just depthwise convolutions \citep{Brock21a}.

\begin{table}[H]
\centering
\caption{\textbf{Number of parameters and number of FLOPs in EfficientNets}. Quantities are reported for EfficientNets-B0 (EN-B0) and EfficientNets-B2 (EN-B2) in the variant with depthwise convolutions and expansion ratio of 6 (left) and in the variant with group convolutions of group size 16 and expansion ratio of 4 (right).}
\label{table:variants}
\scalebox{1}{\begin{tabular}{lcccc}
\\[-9pt]
\toprule
  &   \multicolumn{2}{c}{depthwise convs} & \multicolumn{2}{c}{group convs} \\
    \cmidrule(lr){2-3} \cmidrule(lr){4-5} 
  & EN-B0
  & EN-B2
  & EN-B0
  & EN-B2 \\
\midrule
Number of parameters  & 5.3M & 9.1M & 5.9M & 9.5M \\
Number of FLOPs       & 0.4B & 1.0B & 0.6B & 1.5B \\
\bottomrule
\end{tabular}}
\end{table}

\paragraph{PN.} We always set PN's numerical stability constant to $\epsilon=0.03$, as we found smaller $\epsilon$ can lead to suboptimal performance. We use $200$ samples uniformly sampled in probability in the proxy distribution (cf \Appendref{sec:model_code}).

In all networks, we disable the scaling part of PN in the proxy-normalized activation step just before the final mean pooling. This is to avoid an alteration of the effective learning rate. An alternative option would be to altogether remove PN before the final mean pooling. 

In EfficientNets, we disable PN in squeeze-excite (SE) blocks given that no normalization step precedes each activation step in these blocks. When PN's additional parameters $\tilde{\vbeta}\pl$,~$\tilde{\vgamma}\pl$ are included, we replace the final affine transformation of each MBConv block by a single channel-wise scaling (i.e. we only keep the scaling parameter in the transformation). When PN's additional parameters $\tilde{\vbeta}\pl$,~$\tilde{\vgamma}\pl$ are omitted, on the other hand, we leave this final affine transformation as it is.

\paragraph{WS.} We set the numerical stability constant of WS to 0. 

In all networks, we disable WS in fully-connected layers and in SE blocks. In ResNets and ResNeXts, we add an extra scale parameter after the final convolution of each residual block when using WS.

\paragraph{Evo-S0.} In EfficientNets, the final norm and affine transformation in each MBConv block are replaced by a single affine transformation when using Evo-S0.

\paragraph{Initialization.} We initialize the affine transformation's parameters as $\vbeta\pl=0$, $\vgamma\pl=1$, and PN's additional parameters as $\tilde{\vbeta}\pl=0$, $\tilde{\vgamma}\pl=0$ when these additional parameters are included. We initialize weights $\vomega\pl$ by sampling from truncated normal distributions with an inverse square root scaling with respect to fan-in (expect for some kernels in EfficientNets where the scaling is with respect to fan-out).

\paragraph{ResNets and ResNeXts trained on ImageNet.}
We train for 100 epochs with SGD with a momentum of $0.9$ and a batch size of 256. We start with a learning rate of 0.1 after a linear warmup over the first 5 epochs \citep{Goyal17}, and we decrease this learning rate four times at the epochs 30, 60, 80 and 90, each time by a factor $10$. We apply weight decay with a strength of $10^{-4}$ to all parameters including the additional parameters $\tilde{\vbeta}\pl$, $\tilde{\vgamma}\pl$ of PN and the channel-wise scale and shift parameters $\vbeta\pl$,~$\vgamma\pl$ (this is sensible as $\phi=\relu$ is positive homogeneous).

We set the norm's numerical stability constant to $10^{-6}$ and, unless otherwise specified, we set the number of groups to $G=8$ when using GN+PN and to $G=32$ when using GN without PN. When using BN, we compute BN's statistics over 32 inputs $\rvx$ and we compute moving average statistics by exponentially weighted average with decay factor $0.97$.

For the pre-processing, we follow \cite{He16b}. When using extra regularization, we use label smoothing with factor $0.1$ \citep{Szegedy16}, dropout with rate $0.1$ \citep{Srivastava14} and stochastic depth with rate $0.05$ \citep{Huang16}. As the only exception, when changing the choice of the extra regularization in \Appendref{sec:alternative_regul}, we use Mixup with strength $0.1$ \citep{Zhang18} in all networks, and in ResNet-101 and ResNeXt-101, we additionally use a variant of CutMix \cite{Yun19} that samples $U_0\sim\operatorname{Uniform}(0,1)$ and $U_1\sim\operatorname{Uniform}(e^{-4},1)$ and sets the combination ratio as $\lambda=1$ if $U_0\leq 0.435$ and $\lambda=1+\frac{1}{4}\log(U_1)$ otherwise.

While we use float-16 to store and process intermediate activations (except in normalization steps and PN's statistics computation), model parameters are still stored and updated in float-32. Each time we provide a result, the mean and standard deviation are computed over 3 independent runs, at the final epoch of each run. As the only exception, the mean and $1\sigma$ intervals in the power plots of Figures~\ref{fig:powers},~\ref{fig:powers_init} are computed by \say{pooling together} either all 100 epochs in 5 independent runs (Figure~\ref{fig:powers}) or the initialization state in 5 independent runs (Figure~\ref{fig:powers_init}).

\paragraph{EfficientNets with batch-independent norms trained on ImageNet.} Our experimental setup closely follows \cite{Tan19}. We train for 350 epochs with RMSProp with a batch size of 768. We start with a learning rate of $768 \times 2^{-14}$ (i.e. using a linear scaling \citep{Masters18}) after a linear warmup over the first 5 epochs \citep{Goyal17}, and we decay the learning rate exponentially by a factor 0.97 every 2.4 epochs. In RMSProp, we use a momentum of 0.9, a decay of $1.0 - (768 \times 2^{-14})$ and a numerical stability constant of $10^{-3}$. We apply weight decay with a strength of $10^{-5}$ on the convolutional weights and the additional parameters $\tilde{\vbeta}\pl$, $\tilde{\vgamma}\pl$ of PN, but not on the other channel-wise parameters (this is sensible as $\phi=\swish$ is not positive homogeneous).

We set the norm's numerical stability constant to $10^{-3}$ and we set the number of groups to $G=4$ when using GN or Evo-S0. 

For the baseline pre-processing, we follow \cite{Tan19}. In terms of regularization, we always use label smoothing with factor $0.1$ \citep{Szegedy16}, dropout with rate $0.2$ \citep{Srivastava14} and stochastic depth with rate starting at $0.2$ in the first MBConv block and decaying to zero linearly with the depth of the MBConv block \citep{Huang16}. When using extra regularization, we use Mixup with strength $0.1$ \citep{Zhang18} in all networks, and in EfficientNets-B2, we additionally use a variant of CutMix \cite{Yun19} that samples $U_0\sim\operatorname{Uniform}(0,1)$ and $U_1\sim\operatorname{Uniform}(e^{-4},1)$ and sets the combination ratio as $\lambda=1$ if $U_0\leq 0.435$ and $\lambda=1+\frac{1}{4}\log(U_1)$ otherwise.

While we use float-16 to store and process intermediate activations (except in normalization steps and PN's statistics computation), model parameters are still stored and updated in float-32. Each time we provide a result, the mean and standard deviation are computed over 3 independent runs. For each run, performance is evaluated at the final epoch, with model parameters obtained by exponentially weighted average with decay factor $0.97$ over checkpoints from previous epochs.

\paragraph{EfficientNets with BN trained on ImageNet.} For these experiments, we run the public EfficientNet repository with the settings recommended in the repository.\footnote{\url{https://github.com/tensorflow/tpu/tree/master/models/official/efficientnet}} When considering the variant with group convolutions, our only modifications consist in (i) replacing depthwise convolutions with group convolutions of group size 16, and (ii) changing the expansion ratio from 6 to 4.

In addition to BN's inherent regularization, these runs always incorporate label smoothing \citep{Szegedy16}, dropout \citep{Srivastava14} and stochastic depth \citep{Huang16}. The runs with extra regularization additionally incorporate AutoAugment \citep{Cubuk19}.

Each time we provide a result, the mean and standard deviation are computed over 3 independent runs.

\paragraph{ResNets trained on CIFAR-10 and CIFAR-100 (cf \Appendref{sec:cifar_results}).}
We train for 160 epochs with SGD with a momentum of $0.9$ and a batch size of 128. We start with a learning rate of 0.1 after a linear warmup over the first 5 epochs \citep{Goyal17}, and we decrease this learning rate two times at the epochs 80 and 120, each time by a factor $10$. We apply weight decay with a strength of $10^{-4}$ to all parameters including the additional parameters $\tilde{\vbeta}\pl$, $\tilde{\vgamma}\pl$ of PN and the channel-wise scale and shift parameters $\vbeta\pl$,~$\vgamma\pl$ (this is sensible as $\phi=\relu$ is positive homogeneous). 

We set the norm's numerical stability constant to $10^{-6}$ and we set the number of groups to $G=4$ when using GN. When using BN, we compute BN's statistics over 128 inputs $\rvx$ and we compute moving average statistics by exponentially weighted average with decay factor $0.97$.

For the pre-processing, we follow \cite{He16b}. When using extra regularization, we use label smoothing with factor $0.1$ \citep{Szegedy16}, dropout with rate $0.25$ \citep{Srivastava14} and stochastic depth with rate $0.1$ \citep{Huang16}. 

We use float-16 to store and process intermediate activations (except in normalization steps and PN's statistics computation) and to store and update model parameters. Each time we provide a result, the mean \mbox{and standard deviation are computed over 10 independent runs, at the final epoch of each run.}

\paragraph{Random nets.} We consider random nets following \Defref{def:random}. For the cases of BN, IN, LN, GN, random nets implement \Eqref{eq:propagation1},~(\ref{eq:propagation2}),~(\ref{eq:propagation3}) at every layer $l$. For the case of LN+PN, we replace the activation step of \Eqref{eq:propagation3} by the proxy-normalized activation step of \Eqref{eq:propagation4}. For the case of LN+WS, we add a step of kernel standardization before the convolution of \Eqref{eq:propagation1}. In all cases, convolutions \mbox{use periodic boundary conditions to remain consistent with the assumptions of \Thmref{thm:ln_collapse}.}

We set the activation function to $\phi=\relu$, widths to $C\ul=1024$, kernel sizes to $K\ul=3$. 

We sample the components of the affine transformation's parameters $\vbeta\pl$, $\vgamma\pl$ i.i.d. from $\nu_{\vbeta}=\mathcal{N}(0,0.2^2)$ and $\nu_{\vgamma}=\mathcal{N}(1,0.2^2)$, respectively. This yields $\beta^2=0.2^2$, $\gamma^2=1^2+0.2^2$ and $\rho=\frac{1^2+0.2^2}{1^2+0.2^2+0.2^2}~\approx~0.963$ in \Defref{def:random}. We sample the components of weight parameters $\vomega\pl$ i.i.d. from truncated normal distributions with $\frac{1}{\sqrt{C\ul}}$ scaling. We set PN's additional parameters $\tilde{\vbeta}\pl$,~$\tilde{\vgamma}\pl$ to 0.

We set the norm's numerical stability constant to $10^{-6}$ and we set the number of groups to $G=128$ when using GN to roughly preserve group sizes compared to ResNet-50. We use a batch size of 128 and we compute BN's statistics over all 128 inputs $\rvx$ in the mini-batch when using BN.

We use CIFAR-10 as the dataset $\D$ and we follow \cite{He16b} for the pre-processing. To alleviate the memory burden, \mbox{we add a downsampling by setting the stride to $2$ in the first convolution of the network.}

We use float-32 to store and process intermediate activations. Each time we provide a result, the mean and $1\sigma$ intervals are computed over 50 independent realizations.

\paragraph{A.I. accelerators\textnormal{.}} We run all our experiments with batch-independent norms on Graphcore's IPUs.

\subsection{Additional details on power plots}
\label{sec:details_power_plots}

\subsubsection{Power plots in random nets}
\label{sec:details_power_plots_randomnets}

\paragraph{Additional experimental details.} We obtain the power plots in random nets using the experimental setup described in \Appendref{sec:experimental_setup} for random nets. We compute the terms $\pwrone(\rvy\pl)$, $\pwrtwo(\rvy\pl)$, $\pwrthree(\rvy\pl)$, $\pwrfour(\rvy\pl)$ (as well as the additional terms from \Figref{fig:verif_ln_collapse}) for each layer $l$ using the 128 randomly sampled inputs $\rvx$ in the mini-batch as a proxy for the full dataset $\D$. 

While we set the total depth to $L=200$, we show only the first $20$ layers in \Figref{fig:powers} to facilitate a side-by-side comparison with ResNet-50. Indeed, while the \say{effective} depth is smaller than the \say{computational} depth in ResNet-50 (cf \Appendref{sec:details_power_plots_resnets}), \mbox{the two notions of depth coincide in random nets.}

\paragraph{Verification of \Thmref{thm:ln_collapse}.} The case of random nets with LN enables us to precisely verify \Thmref{thm:ln_collapse}. We provide this verification in the left and center subplots of \Figref{fig:verif_ln_collapse}.

In the left subplot of \Figref{fig:verif_ln_collapse}, we show $\pwr(\rvy\pl)-\pwrone(\rvy\pl)$ (mean and $1\sigma$ interval) and the upper bound $\rho\plm=\Big(\frac{1^2+0.2^2}{1^2+0.2^2+0.2^2}\Big)\plm$ from \Thmref{thm:ln_collapse} for depths $l$ up to $L=200$. We confirm that $\pwr(\rvy\pl)-\pwrone(\rvy\pl)$ is upper bounded with high probability by $\rho\plm$ as predicted by \Thmref{thm:ln_collapse}. The rate of decay of $\pwr(\rvy\pl)-\pwrone(\rvy\pl)$ is initially above the prediction of \Thmref{thm:ln_collapse} due to the aggravating effect of $\phi=\relu$. In very deep layers ($l\gg 1$), this rate of decay ends up very slightly below the prediction of \Thmref{thm:ln_collapse} due to the facts that: (i) $\phi$ becomes effectively close to channel-wise linear; (ii) the channel-wise collapse is slightly mitigated by LN in the case of a finite width $C\ul=1024$.

In the center subplot of \Figref{fig:verif_ln_collapse}, we show $\pwr(\rvy\pl)$ (mean and $1\sigma$ interval) for depths $l$ up to $L=200$. We confirm that $\pwr(\rvy\pl)$ is with high probability very close to one.

\paragraph{Quantification of channel-wise linearity.} To confirm the connection between channel-wise collapse and channel-wise linearity, we finally report the evolution with depth of an additional measure of channel-wise linearity. In the right subplot of \Figref{fig:verif_ln_collapse}, we show the measure $\pwr(\phi(\ty\pl)-\wtz\pl) / \pwr(\phi(\ty\pl))$ (mean and $1\sigma$ interval) for depths $l$ up to $L=200$, with $\wtz\pl$ the channel-wise linear best-fit of $\phi(\ty\pl)$ using $\ty\pl$, that is defined in \Eqref{eq:linearized1}, (\ref{eq:linearized2}). We confirm that deep in random nets, layers are effectively: (i) very close to channel-wise linear with LN; (ii) close to channel-wise linear with GN.

\begin{figure}[H]
\centering
\includegraphics[width=1.\linewidth]{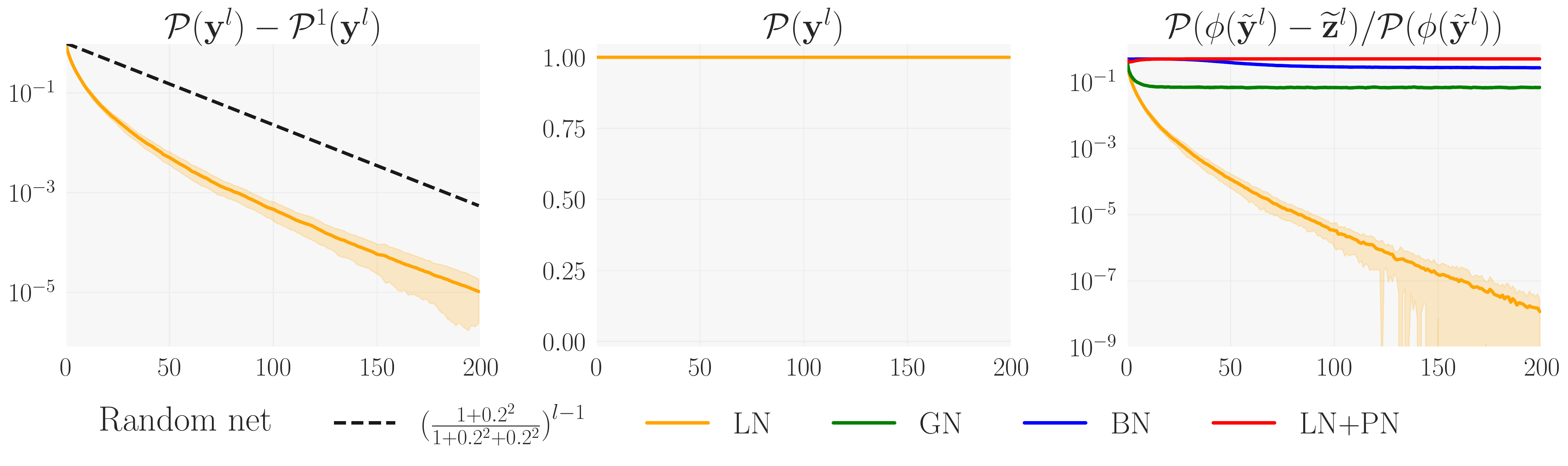}
\caption{\textbf{Verification of \Thmref{thm:ln_collapse} and quantification of channel-wise linearity}. Results are reported in random nets of \Defref{def:random} for depths $l$ up to $L=200$. Left: $\pwr(\rvy\pl)-\pwrone(\rvy\pl)$ (mean and $1\sigma$ interval) and upper bound $\rho\plm$ from \Thmref{thm:ln_collapse}. Center: $\pwr(\rvy\pl)$ (mean and $1\sigma$ interval). Right: Additional measure of channel-wise linearity $\pwr(\phi(\ty\pl)-\wtz\pl) / \pwr(\phi(\ty\pl))$ (mean and $1\sigma$ interval), \mbox{with $\wtz\pl$ the channel-wise linear best-fit of $\phi(\ty\pl)$ using $\ty\pl$, that is defined in \Eqref{eq:linearized1}, (\ref{eq:linearized2})} 
}
\label{fig:verif_ln_collapse}
\end{figure}

\subsubsection{Power plots in ResNet-50}
\label{sec:details_power_plots_resnets}

\paragraph{Additional experimental details.} We obtain the power plots in ResNet-50 using the experimental setup described in \Appendref{sec:experimental_setup} for ResNets on ImageNet. At each epoch, we compute the power terms $\pwrone\uc(\rvy\pl)$, $\pwrtwo\uc(\rvy\pl)$, $\pwrthree\uc(\rvy\pl)$, $\pwrfour\uc(\rvy\pl)$ for each layer $l$ and each channel $c$ using the last 256 randomly sampled inputs $\rvx$ as a proxy for the full dataset $\D$. 

When looking at each norm separately in ResNets, we noticed artefacts that we attributed to the discrepancy between the \say{computational} depth $l$ and the \say{effective} depth (that oscillates with $l$). Indeed, the effective depth, defined in terms of the statistical properties of intermediate activations, grows linearly inside each residual block but gets reduced each time a residual path is summed with a skip connection path (since the latter originates from earlier layers). This phenomenon is tightly connected to \mbox{the property discussed in \Secref{sec:bn_properties} on the control of activation scale in residual networks.}

To avoid such an artefact in Figures~\ref{fig:powers},~\ref{fig:powers_init}, we report only a single measurement of $\pwrone(\rvy\pl)$, $\pwrtwo(\rvy\pl)$, $\pwrthree(\rvy\pl)$, $\pwrfour(\rvy\pl)$ per residual block by \say{pooling together} all the channels from the three norms inside each residual block. For the same reason, we do not report $\pwrone(\rvy\pl)$, $\pwrtwo(\rvy\pl)$, $\pwrthree(\rvy\pl)$, $\pwrfour(\rvy\pl)$ for the final norm just before the final mean pooling.

The presence of this artefact confirms the fact that the effective depth evolves more slowly than the computational depth $l$ in residual networks. This explains why $\rvy\pl$ with LN or GN is not immoderately collapsed even at large $l$ in Figures~\ref{fig:powers},~\ref{fig:powers_init}.

\paragraph{Numerical stability issues with IN.}  As stated in the caption of Figure~\ref{fig:powers}, we did not succeed at training ResNet-50 v2 with IN. We found that using float-16 to store and process intermediate activations caused divergence in these networks. When replacing float-16 by float-32, even though divergence was avoided, ResNets-50 v2 still did not reach satisfactory performance with IN. We attributed this to a plain incompatibility of IN with v2 instantiations of ResNets, which could stem from the presence of a final block of normalization and activation just before the final mean pooling. Intuitively, if we denote this final block as $L$ and if we subtract away the activation function by supposing $\phi=\identity$, then $\mu_{\rvx,c}(\rvz^L)$ in channel $c$ is constant for all $\rvx$, equal to $\vbeta^L\uc$ \mbox{(cf \Secref{sec:failure_modes})}. Thus, if we subtract away the activation function, with IN all inputs $\rvx$ end up mapped to the same channel-wise constants after the final mean pooling, i.e. they become indistinguishable.

\paragraph{Power plots at initialization.} Figure~\ref{fig:powers_init} reports the same power plots as Figure~\ref{fig:powers}, except with $\pwrone(\rvy\pl)$, $\pwrtwo(\rvy\pl)$, $\pwrthree(\rvy\pl)$, $\pwrfour(\rvy\pl)$ computed at initialization (mean and $1\sigma$ intervals).

When comparing \Figref{fig:powers_init} to \Figref{fig:powers}, it is clearly visible that the channel-wise collapse with LN+WS gets aggravated during training compared to initialization. This confirms the importance of compensating during training the mean shift associated with the affine transformation. 

It is also visible that the difference between GN and LN gets narrower during training compared to initialization. This means that despite a similar behavior of $\pwrone(\rvy\pl)$ along the training trajectories with GN and LN, differences could still exist in the vicinity of these trajectories, implying a better conditioning of the loss landscape with GN. A similar argument would make us expect a better conditioning of the loss landscape when enforcing $\rvy\pl$ to be channel-wise normalized via an operation directly embedded in the network mapping \citep{Page19,Ghorbani19} as opposed to via an external penalty \citep{Collins19,Rao20,Dauphin21}, despite the two approaches potentially leading to the same reduction of $\pwrone(\rvy\pl)$. 

We believe that the notions of \say{channel-wise collapse} and \say{conditioning of the loss landscape} \citep{Page19,Ghorbani19} enable to quantify more accurately the underlying phenomenons at play than the notion of \say{internal covariate shift} \citep{Ioffe15,Santurkar18}, despite the former and latter notions being connected.

\begin{figure}[H]
\centering
\includegraphics[width=1.\linewidth]{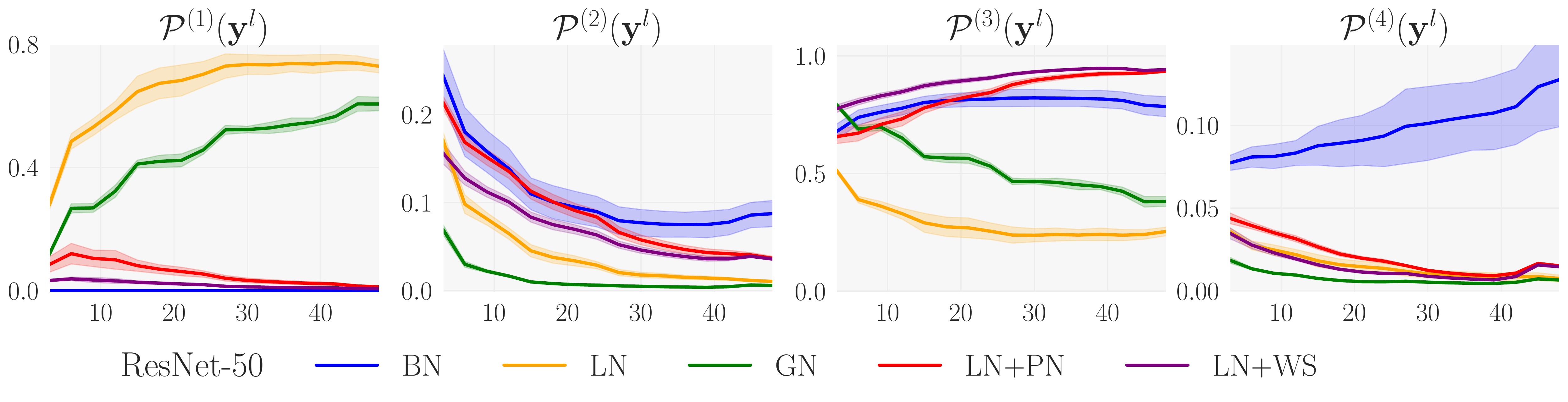}
\caption{\textbf{Power plots at initialization}. $\pwrone(\rvy\pl)$, $\pwrtwo(\rvy\pl)$, $\pwrthree(\rvy\pl)$, $\pwrfour(\rvy\pl)$ are shown as a function of the depth $l$ in ResNet-50 at initialization with BN and different batch-independent norms: LN, GN, LN+PN, LN+WS.}
\label{fig:powers_init}
\end{figure}

\subsection{More detailed results on ImageNet}
\subsubsection{$1\sigma$ intervals}
\label{sec:error_bars_imagenet}
In Tables~\ref{table:rn50_detailed},~\ref{table:rn_rnx_detailed},~\ref{table:en_detailed}, we complement the results of Tables~\ref{table:rn50},~\ref{table:rn_rnx},~\ref{table:en} with $1\sigma$ intervals. In \Figref{fig:results_en}, we provide a visualization of the results of Tables~\ref{table:en},~\ref{table:en_detailed}.

\begin{table}[H]
\centering
\caption{\textbf{Effect of adding PN}. ResNet-50 is trained on ImageNet with BN and LN, GN, GN+WS with $G$ groups, either without or with PN added on top (plain vs. PN). Results are formatted as X~/~Y with X,~Y the validation accuracies (mean and $1\sigma$ interval in \%) without and with extra regularization, respectively.}
\label{table:rn50_detailed}
\scalebox{0.85}{
\begin{tabular}{lccc}
\\[-9pt]
\toprule
       &  &  \multicolumn{2}{c}{RN50} \\
      &  $G$ &                                   plain &                                       +PN \\
\midrule
BN      &    &  \meanstdbf{76.3}{0.1} / \meanstd{75.8}{0.2} &  \meanstd{76.2}{0.1} / \meanstd{76.0}{0.1} \\
\midrule
LN          & 1  &  \meanstd{74.5}{0.0} / \meanstd{74.6}{0.1} &  \meanstd{75.9}{0.1} / \meanstd{76.5}{0.0} \\
GN 		    & 8  &  \meanstd{75.4}{0.1} / \meanstd{75.4}{0.1} &  \meanstd{76.3}{0.1} / \meanstdbf{76.7}{0.0} \\
GN          & 32  &  \meanstd{75.4}{0.1} / \meanstd{75.3}{0.1} &  \meanstd{75.8}{0.2} / \meanstd{76.1}{0.1} \\
\midrule
GN+WS       & 8  &  \meanstd{76.6}{0.0} / \meanstd{76.7}{0.1} &  \meanstd{76.8}{0.1} / \meanstdbf{77.1}{0.1} \\
\bottomrule
\end{tabular}}
\end{table}

\begin{table}[H]
\centering
\caption{\textbf{BN vs. GN, GN+PN}. ResNets and ResNeXts are trained on ImageNet with BN and GN, GN+PN. Results are formatted as in \Tableref{table:rn50_detailed}.}
\label{table:rn_rnx_detailed}
\scalebox{0.85}{
\begin{tabular}{lcccc}
\\[-9pt]
\toprule
  &                                     RN50 &                                    RN101 &                                    RNX50 &                                   RNX101 \\
\midrule
BN    &  \meanstdbf{76.3}{0.1} / \meanstd{75.8}{0.2} &  \meanstd{77.9}{0.1} / \meanstdbf{78.0}{0.1} &  \meanstdbf{77.6}{0.1} / \meanstd{77.2}{0.1} &  \meanstd{78.7}{0.1} / \meanstdbf{78.9}{0.1} \\
\midrule
GN    &  \meanstd{75.4}{0.1} / \meanstd{75.3}{0.1} &  \meanstd{77.0}{0.1} / \meanstd{77.4}{0.1} &  \meanstd{76.2}{0.2} / \meanstd{76.6}{0.1} &  \meanstd{77.4}{0.2} / \meanstd{78.1}{0.1} \\
GN+PN &  \meanstd{76.3}{0.1} / \meanstdbf{76.7}{0.0} &  \meanstd{77.6}{0.2} / \meanstdbf{78.6}{0.2} &  \meanstd{76.7}{0.1} / \meanstdbf{77.8}{0.2} &  \meanstd{77.7}{0.2} / \meanstdbf{79.0}{0.1} \\
\bottomrule
\end{tabular}}
\end{table}

\begin{table}[H]
\centering
\caption{\textbf{BN vs. batch-independent approaches}. EfficientNets are trained on ImageNet with BN and various batch-independent approaches: GN, GN+PN, Evo-S0, GN+WS, FRN+TLU. Results are formatted as in \Tableref{table:rn50_detailed}.}
\label{table:en_detailed}
\scalebox{0.85}{
\begin{tabular}{lcccc}
\\[-9pt]
\toprule
  &   \multicolumn{2}{c}{depthwise convs} & \multicolumn{2}{c}{group convs} \\
    \cmidrule(lr){2-3} \cmidrule(lr){4-5} 
  &                                 EN-B0 &                                 EN-B2 &                                EN-B0 &                                EN-B2 \\
\midrule
BN      &  \meanstd{76.9}{0.1} / \meanstdbf{77.2}{0.1} &  \meanstd{79.4}{0.0} / \meanstdbf{80.0}{0.0} &  \meanstdbf{76.8}{0.1} / \meanstd{76.7}{0.2} &  \meanstd{79.5}{0.1} / \meanstdbf{79.7}{0.1} \\
\midrule
GN      &  \meanstd{76.2}{0.1} / \meanstd{76.2}{0.1} &  \meanstd{78.9}{0.1} / \meanstd{79.4}{0.1} &  \meanstd{76.2}{0.1} / \meanstd{76.2}{0.2} &  \meanstd{79.0}{0.1} / \meanstd{79.6}{0.1} \\
GN+PN   &  \meanstd{76.8}{0.0} / \meanstdbf{77.0}{0.1} &  \meanstd{79.3}{0.1} / \meanstdbf{80.0}{0.1} &  \meanstd{76.7}{0.1} / \meanstdbf{76.8}{0.1} &  \meanstd{79.3}{0.1} / \meanstdbf{80.1}{0.1} \\
Evo-S0  &  \meanstd{75.8}{0.1} / \meanstd{75.8}{0.2} &  \meanstd{78.5}{0.1} / \meanstd{78.7}{0.1} &  \meanstd{76.2}{0.0} / \meanstd{76.5}{0.1} &  \meanstd{78.9}{0.0} / \meanstd{79.6}{0.0} \\
GN+WS   &  \meanstd{74.2}{0.1} / \meanstd{74.1}{0.1} &  \meanstd{77.8}{0.0} / \meanstd{77.8}{0.1} &  \meanstd{76.2}{0.1} / \meanstd{76.3}{0.1} &  \meanstd{79.2}{0.1} / \meanstd{79.4}{0.1} \\
FRN+TLU &  \meanstd{75.7}{0.1} / \meanstd{75.7}{0.2} &  \meanstd{78.4}{0.1} / \meanstd{78.9}{0.1} &  \meanstd{74.9}{0.2} / \meanstd{75.1}{0.1} &  \meanstd{78.2}{0.1} / \meanstd{78.6}{0.1} \\
\bottomrule
\end{tabular}}
\end{table}

\begin{figure}[H]
\centering
\includegraphics[width=0.97\textwidth]{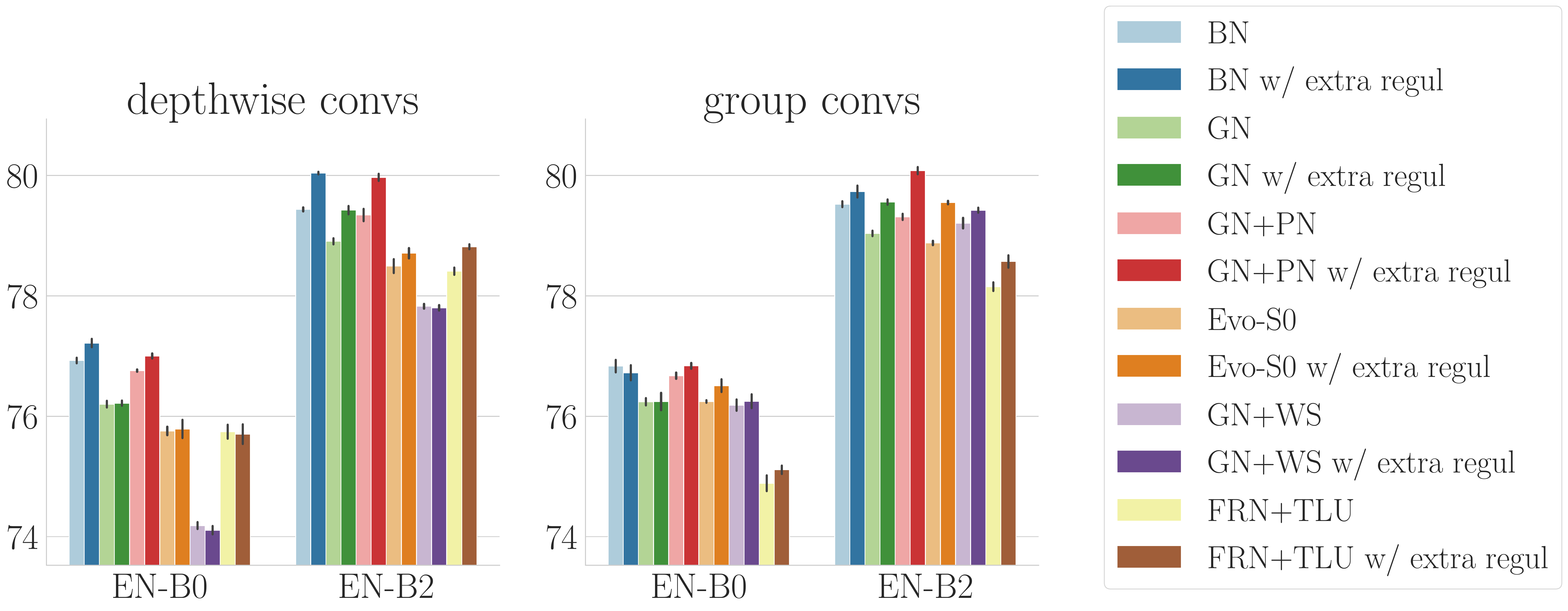}
  \caption{\textbf{BN vs. batch-independent approaches}. Validation accuracies (\%) of EfficientNets trained on ImageNet with BN and various batch-independent approaches, without and with extra regularization. } 
  \label{fig:results_en} 
\end{figure}

\subsubsection{Training accuracies}
\label{sec:training_acc_imagenet}

In Tables~\ref{table:rn50_train},~\ref{table:rn_rnx_train},~\ref{table:en_train}, we complement the results of Tables~\ref{table:rn50},~\ref{table:rn_rnx},~\ref{table:en} with training accuracies.

We stress that these training accuracies are highly dependent on the strength of applied regularization. This leads us to: (i) always separate the training accuracies obtained without and with extra regularization; (ii) report only the training accuracies obtained with batch-independent approaches, given that training accuracies obtained with BN would not be comparable due to BN's inherent regularization.
 
As visible in Tables~\ref{table:rn50_train},~\ref{table:rn_rnx_train},~\ref{table:en_train}, GN+PN outperforms alternative batch-independent approaches in terms of training accuracy on ImageNet. This applies both to training without extra regularization and to training with extra regularization. This suggests that, on larger datasets, GN+PN would outperform these alternative batch-independent approaches in terms of both training and validation accuracies \citep{Kolesnikov20,Brock21b}.

In \Tableref{table:en_train}, the fact that with extra regularization EfficientNets-B2 reach lower training accuracies than EfficientNets-B0 is explained by the different level of applied regularization (we add CutMix when training EfficientNets-B2).

\begin{table}[H]
   \centering
\caption{\textbf{Training accuracies in ResNet-50}. Networks are trained on ImageNet with LN, GN, GN+WS with $G$ groups, either without or with PN added on top (plain vs. +PN). Results are formatted as X with X the training accuracy at the final epoch (mean and $1\sigma$ interval in \%). We report separately the results without extra regularization (top) and with extra regularization (bottom). 
}
\scalebox{0.86}
{\begin{tabular}{llccc}
\\[-9pt]
\toprule
   &   &      &  \multicolumn{2}{c}{RN50} \\
   &   &  $G$ &                                   plain &  +PN \\
\midrule
\multirow{4}*{without extra regul} 
& LN      &  1  &  \meanstd{75.7}{0.1} &  \meanstd{79.9}{0.1} \\
& GN      &  8  &  \meanstd{77.2}{0.1} &  \meanstdbf{80.3}{0.1} \\
& GN      &  32  &  \meanstd{77.0}{0.0} &  \meanstd{79.2}{0.2} \\
\cmidrule(lr){2-5}
& GN+WS   &  8  &  \meanstd{80.1}{0.0} &  \meanstdbf{80.4}{0.0} \\
\midrule
\multirow{4}*{with extra regul} 
& LN      &  1  &  \meanstd{71.8}{0.1} &  \meanstd{75.8}{0.0} \\
& GN      &  8  &  \meanstd{73.3}{0.1} &  \meanstdbf{76.2}{0.1} \\
& GN      &  32 &  \meanstd{73.1}{0.1} &  \meanstd{75.1}{0.1} \\
\cmidrule(lr){2-5}
& GN+WS   &  8  &  \meanstd{75.8}{0.0} &  \meanstdbf{76.3}{0.0} \\
\bottomrule
\end{tabular}}
\label{table:rn50_train}
\end{table}

\begin{table}[H]
   \centering
   \caption{\textbf{Training accuracies in ResNets and ResNeXts}. Networks are trained on ImageNet with GN, GN+PN. Results are formatted as in \Tableref{table:rn50_train}.}
\scalebox{0.85}
{\begin{tabular}{llcccc}
\\[-9pt]
\toprule
 & &                RN50 &               RN101 &               RNX50 &              RNX101 \\
\midrule
\multirow{2}*{without extra regul} 
& GN    &  \meanstd{77.0}{0.0} &  \meanstd{79.9}{0.1} &  \meanstd{79.6}{0.1} &  \meanstd{81.6}{0.0} \\
& GN+PN &  \meanstdbf{80.3}{0.1} &  \meanstdbf{83.5}{0.0} &  \meanstdbf{84.1}{0.1} &  \meanstdbf{86.2}{0.0} \\
\midrule
\multirow{2}*{with extra regul} 
& GN    &  \meanstd{73.1}{0.1} &  \meanstd{76.5}{0.0} &  \meanstd{76.0}{0.1} &  \meanstd{78.6}{0.1} \\
& GN+PN &  \meanstdbf{76.2}{0.1} &  \meanstdbf{79.7}{0.1} &  \meanstdbf{79.8}{0.0} &  \meanstdbf{82.7}{0.0} \\
\bottomrule
\end{tabular}}
\label{table:rn_rnx_train}
\end{table}

\begin{table}[H]
   \centering
  \caption{\textbf{Training accuracies in EfficientNets}. Networks are trained on ImageNet with various batch-independent approaches: GN, GN+PN, Evo-S0, GN+WS, FRN+TLU. Results are formatted as in \Tableref{table:rn50_train}.}
\scalebox{0.85}
{\begin{tabular}{llcccc}
\\[-9pt]
\toprule
 & &   \multicolumn{2}{c}{depthwise convs} & \multicolumn{2}{c}{group convs} \\
    \cmidrule(lr){3-4} \cmidrule(lr){5-6} 
 & &                                 EN-B0 &                                 EN-B2 &                                EN-B0 &                                EN-B2 \\
\midrule
\multirow{5}*{without extra regul}
& GN      &  \meanstd{75.4}{0.0}   &  \meanstd{80.9}{0.1}   &  \meanstd{74.7}{0.0}   &  \meanstd{80.1}{0.1}   \\
& GN+PN   &  \meanstdbf{77.3}{0.0}   &  \meanstdbf{82.7}{0.0}   &  \meanstdbf{75.8}{0.0}   &  \meanstdbf{81.4}{0.1}   \\
& Evo-S0  &  \meanstd{74.6}{0.2}   &  \meanstd{79.8}{0.2}   &  \meanstd{75.1}{0.0}   &  \meanstd{80.4}{0.1}   \\
& GN+WS   &  \meanstd{71.4}{0.0}   &  \meanstd{77.6}{0.0}   &  \meanstd{74.5}{0.0}   &  \meanstd{80.2}{0.1}   \\
& FRN+TLU &  \meanstd{75.0}{0.1}   &  \meanstd{80.4}{0.0}   &  \meanstd{72.9}{0.1}   &  \meanstd{78.5}{0.1}   \\
\midrule
\multirow{5}*{with extra regul} 
& GN      &  \meanstd{71.2}{0.1} &  \meanstd{66.2}{0.1} &  \meanstd{70.5}{0.1} &  \meanstd{65.6}{0.1} \\
& GN+PN   &  \meanstdbf{72.8}{0.0} &  \meanstdbf{67.8}{0.1} &  \meanstdbf{71.5}{0.1} &  \meanstdbf{66.7}{0.0} \\
& Evo-S0  &  \meanstd{70.2}{0.2} &  \meanstd{64.4}{0.3} &  \meanstd{70.8}{0.1} &  \meanstd{65.6}{0.1} \\
& GN+WS   &  \meanstd{67.3}{0.1} &  \meanstd{63.4}{0.1} &  \meanstd{70.4}{0.1} &  \meanstd{65.4}{0.0} \\
& FRN+TLU &  \meanstd{70.4}{0.3} &  \meanstd{65.1}{0.2} &  \meanstd{68.8}{0.2} &  \meanstd{64.0}{0.1} \\
\bottomrule
\end{tabular}}
\label{table:en_train}
\end{table}

\subsubsection{Effect of omitting PN's additional parameters}
\label{sec:pn_noadditional}

In Tables~\ref{table:rn_rnx_pn_noadditional},~\ref{table:en_pn_noadditional} and \Figref{fig:results_pn_noadditional}, we report results with PN's additional parameters $\tilde{\vbeta}\pl$, $\tilde{\vgamma}\pl$ set to 0. In that case, $\tilde{\vbeta}\pl$, $\tilde{\vgamma}\pl$ can be equivalently omitted and the proxy variable $Y\pl$ can be simply considered as a standard Gaussian variable in each channel $c$, i.e. $Y\pl\uc \sim \mathcal{N}(0,1)$ (cf our implementation of \Appendref{sec:model_code}).

As visible in Tables~\ref{table:rn_rnx_pn_noadditional},~\ref{table:en_pn_noadditional} and \Figref{fig:results_pn_noadditional}, the omission of PN's additional parameters $\tilde{\vbeta}\pl$, $\tilde{\vgamma}\pl$ is indeed harmful. However, the drop of performance that results from omitting $\tilde{\vbeta}\pl$, $\tilde{\vgamma}\pl$ in GN+PN is very small (in average less than $0.1\%$ in validation accuracy).

Given that the omission of PN's additional parameters $\tilde{\vbeta}\pl$, $\tilde{\vgamma}\pl$ leads to slight benefits in terms of computational requirements and simplicity of implementation, this variant of PN with $\tilde{\vbeta}\pl$, $\tilde{\vgamma}\pl$ omitted might sometimes be a better trade-off.

\begin{table}[H]
\centering
\caption{\textbf{Effect of omitting PN's additional parameters in ResNets and ResNeXts}. Networks are trained on ImageNet with BN and GN, GN+PN with $\tilde{\vbeta}\pl$, $\tilde{\vgamma}\pl$ included, and GN+PN with $\tilde{\vbeta}\pl$, $\tilde{\vgamma}\pl$ omitted. Results are formatted as in \Tableref{table:rn50_detailed}.}
\label{table:rn_rnx_pn_noadditional}
\scalebox{0.85}{
\begin{tabular}{lcccc}
\\[-9pt]
\toprule
  &                                     RN50 &                                    RN101 &                                    RNX50 &                                   RNX101 \\
\midrule
BN        &  \meanstdbf{76.3}{0.1} / \meanstd{75.8}{0.2} &  \meanstd{77.9}{0.1} / \meanstdbf{78.0}{0.1} &  \meanstdbf{77.6}{0.1} / \meanstd{77.2}{0.1} &  \meanstd{78.7}{0.1} / \meanstdbf{78.9}{0.1} \\
\midrule
GN        &  \meanstd{75.4}{0.1} / \meanstd{75.3}{0.1} &  \meanstd{77.0}{0.1} / \meanstd{77.4}{0.1} &  \meanstd{76.2}{0.2} / \meanstd{76.6}{0.1} &  \meanstd{77.4}{0.2} / \meanstd{78.1}{0.1} \\
GN+PN with $\tilde{\vbeta}\pl$, $\tilde{\vgamma}\pl$ included &  \meanstd{76.3}{0.1} / \meanstdbf{76.7}{0.0} &  \meanstd{77.6}{0.2} / \meanstdbf{78.6}{0.2} &  \meanstd{76.7}{0.1} / \meanstdbf{77.8}{0.2} &  \meanstd{77.7}{0.2} / \meanstd{79.0}{0.1} \\
GN+PN with $\tilde{\vbeta}\pl$, $\tilde{\vgamma}\pl$ omitted  &  \meanstd{76.3}{0.0} / \meanstd{76.7}{0.1} &  \meanstd{77.5}{0.0} / \meanstd{78.5}{0.1} &  \meanstd{76.5}{0.1} / \meanstd{77.6}{0.1} &  \meanstd{77.5}{0.1} / \meanstdbf{79.0}{0.1} \\
\bottomrule
\end{tabular}}
\end{table}

\begin{table}[H]
\centering
\caption{\textbf{Effect of omitting PN's additional parameters in EfficientNets}. Networks are trained on ImageNet with BN and GN, GN+PN with $\tilde{\vbeta}\pl$, $\tilde{\vgamma}\pl$ included, and GN+PN with $\tilde{\vbeta}\pl$, $\tilde{\vgamma}\pl$ omitted. Results are formatted as in \Tableref{table:rn50_detailed}.}
\label{table:en_pn_noadditional}
\scalebox{0.85}{
\begin{tabular}{lcccc}
\\[-9pt]
\toprule
  &   \multicolumn{2}{c}{depthwise convs} & \multicolumn{2}{c}{group convs} \\
    \cmidrule(lr){2-3} \cmidrule(lr){4-5} 
  &                                 EN-B0 &                                 EN-B2 &                                EN-B0 &                                EN-B2 \\
\midrule
BN        &  \meanstd{76.9}{0.1} / \meanstdbf{77.2}{0.1} &  \meanstd{79.4}{0.0} / \meanstdbf{80.0}{0.0} &  \meanstdbf{76.8}{0.1} / \meanstd{76.7}{0.2} &  \meanstd{79.5}{0.1} / \meanstdbf{79.7}{0.1} \\
\midrule
GN        &  \meanstd{76.2}{0.1} / \meanstd{76.2}{0.1} &  \meanstd{78.9}{0.1} / \meanstd{79.4}{0.1} &  \meanstd{76.2}{0.1} / \meanstd{76.2}{0.2} &  \meanstd{79.0}{0.1} / \meanstd{79.6}{0.1} \\
GN+PN with $\tilde{\vbeta}\pl$, $\tilde{\vgamma}\pl$ included &  \meanstd{76.8}{0.0} / \meanstdbf{77.0}{0.1} &  \meanstd{79.3}{0.1} / \meanstdbf{80.0}{0.1} &  \meanstd{76.7}{0.1} / \meanstdbf{76.8}{0.1} &  \meanstd{79.3}{0.1} / \meanstdbf{80.1}{0.1} \\
GN+PN with $\tilde{\vbeta}\pl$, $\tilde{\vgamma}\pl$ omitted   &  \meanstd{76.6}{0.2} / \meanstd{77.0}{0.1} &  \meanstd{79.2}{0.0} / \meanstd{79.9}{0.1} &  \meanstd{76.7}{0.1} / \meanstd{76.7}{0.1} &  \meanstd{79.3}{0.1} / \meanstd{80.0}{0.2} \\
\bottomrule
\end{tabular}}
\end{table}

\begin{figure}[H]
\centering
\includegraphics[width=0.95\textwidth]{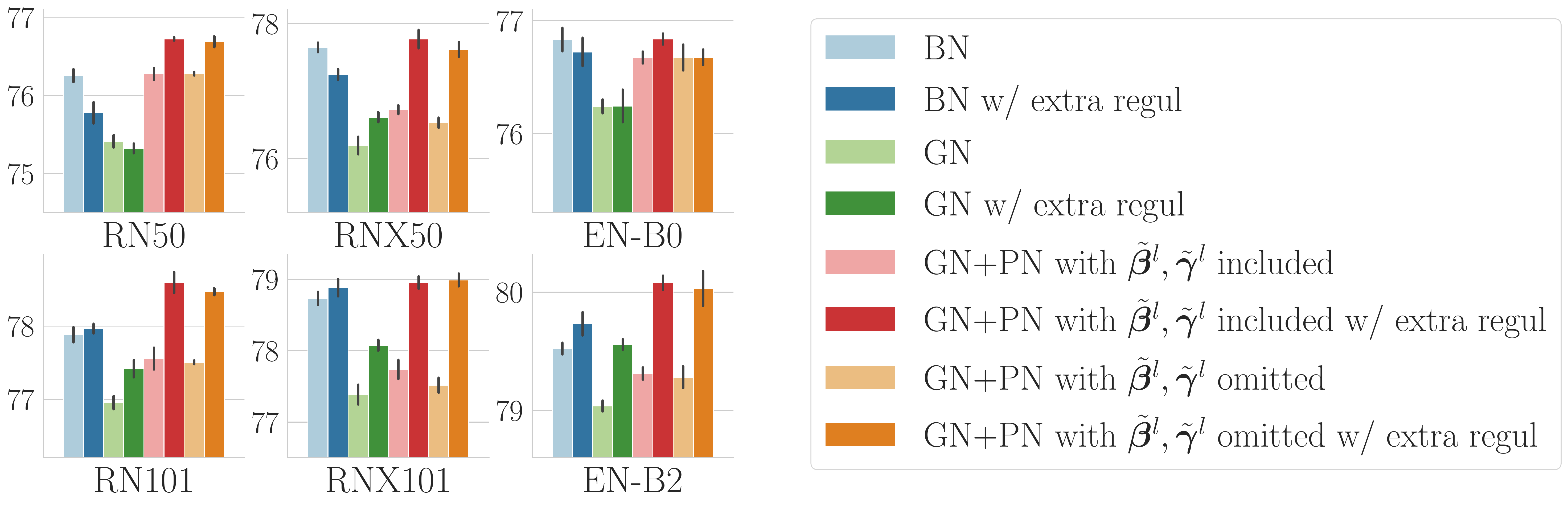}
  \caption{\textbf{Effect of omitting PN's additional parameters}. Validation accuracies (\%) of ResNets, ResNeXts and EfficientNets trained on ImageNet with BN and GN, GN+PN with $\tilde{\vbeta}\pl$, $\tilde{\vgamma}\pl$ included, and GN+PN with $\tilde{\vbeta}\pl$, $\tilde{\vgamma}\pl$ omitted, without and with extra regularization. EfficientNets are considered in the variant with group convolutions \citep{Masters21}.} 
  \label{fig:results_pn_noadditional} 
\end{figure}

\subsubsection{Effect of changing the choice of the extra regularization}
\label{sec:alternative_regul}

In \Tableref{table:alternative_regul} and \Figref{fig:results_alternative_regul}, we report results in ResNets and ResNeXts with a change in the choice of the extra regularization. When using extra regularization, instead of using label smoothing \citep{Szegedy16}, dropout \citep{Srivastava14} and stochastic depth \citep{Huang16}, we use Mixup \citep{Zhang18} in all networks, and in ResNet-101 and ResNeXt-101, we additionally use CutMix \cite{Yun19} (cf \Appendref{sec:experimental_setup}).

We reach similar conclusions with the results of \Tableref{table:alternative_regul} and \Figref{fig:results_alternative_regul} as with the results of \Tableref{table:rn_rnx_detailed}: (i) BN is matched or outperformed by GN+PN, except for a small gap of performance in ResNeXt-50 (this gap of performance might be due to the imperfect \say{abstraction away} of regularization); (ii) good performance remains tied to the combination of both an efficient normalization and an efficient regularization.

\begin{table}[H]
\centering
\caption{\textbf{Effect of changing the choice of the extra regularization}. ResNets and ResNeXts are trained on ImageNet with BN and GN, GN+PN. Results are formatted as X~/~Z with X,~Z the validation accuracies (mean and $1\sigma$ interval in \%) without extra regularization and with an extra regularization other than the one used in \Tableref{table:rn_rnx_detailed}, respectively.}
\label{table:alternative_regul}
\scalebox{0.85}{
\begin{tabular}{lcccc}
\\[-9pt]
\toprule
  &                                     RN50 &                                    RN101 &                                    RNX50 &                                   RNX101 \\
\midrule
BN    &  \meanstd{76.3}{0.1} / \meanstdbf{76.3}{0.0} &  \meanstd{77.9}{0.1} / \meanstdbf{78.1}{0.1} &  \meanstd{77.6}{0.1} / \meanstdbf{78.0}{0.0} &  \meanstd{78.7}{0.1} / \meanstdbf{79.5}{0.0} \\
\midrule
GN    &  \meanstd{75.4}{0.1} / \meanstd{75.9}{0.1} &  \meanstd{77.0}{0.1} / \meanstd{77.7}{0.1} &  \meanstd{76.2}{0.2} / \meanstd{76.7}{0.1} &  \meanstd{77.4}{0.2} / \meanstd{78.3}{0.1} \\
GN+PN &  \meanstd{76.3}{0.1} / \meanstdbf{77.0}{0.0} &  \meanstd{77.6}{0.2} / \meanstdbf{78.9}{0.1} &  \meanstd{76.7}{0.1} / \meanstdbf{77.6}{0.1} &  \meanstd{77.7}{0.2} / \meanstdbf{79.6}{0.1} \\
\bottomrule
\end{tabular}}
\end{table}

\begin{figure}[H]
\centering
\includegraphics[width=0.9\textwidth]{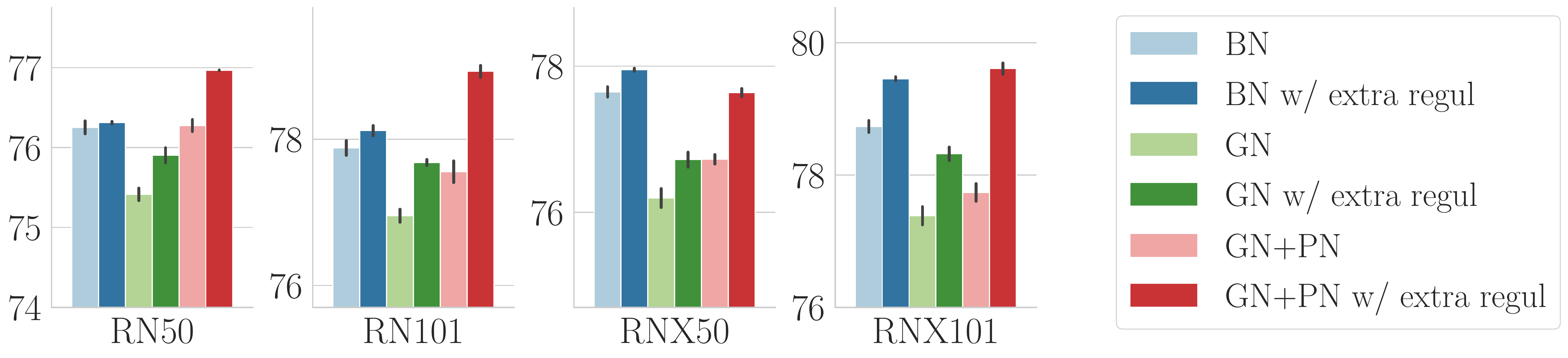}
  \caption{\textbf{Effect of changing the choice of the extra regularization}. Validation accuracies (\%) of ResNets and ResNeXts trained on ImageNet with BN and GN, GN+PN, without extra regularization and with an extra regularization other than the one used in \Tableref{table:rn_rnx_detailed}.} 
  \label{fig:results_alternative_regul} 
\end{figure}

\subsection{Results on CIFAR-10 and CIFAR-100}
\label{sec:cifar_results}

In this section, we report results on CIFAR with different sizes of ResNets: ResNet-20 (RN20), ResNet-32 (RN32), ResNet-44 (RN44), ResNet-56 (RN56), ResNet-110 (RN110).

We report results on CIFAR-10 in \Tableref{table:cifar_10}, and results on CIFAR-100 in \Tableref{table:cifar_100}. We further provide a visualization of these results in \Figref{fig:cifar}.

While slightly underperforming BN on CIFAR-10, GN+PN tends to slightly outperform BN on CIFAR-100. As a possible reason, BN's regularization could be more beneficial on the \say{easy} task of CIFAR-10 than on the \say{harder} task of CIFAR-100. To the extent that BN's regularization can be seen as a reduction of the network's effective capacity, such a reduction of the network's effective capacity could be more harmful for tasks that require more capacity, i.e. for harder tasks.

\begin{table}[H]
\centering
\caption{\textbf{BN vs. GN, GN+PN on CIFAR-10}. ResNets are trained with BN and GN, GN+PN. Results are formatted as X~/~Y with X,~Y the validation accuracies (mean and $1\sigma$ interval in \%) without and with extra regularization, respectively.}
\label{table:cifar_10}
\scalebox{0.82}{
\begin{tabular}{lccccc}
\\[-9pt]
\toprule
  &                                     RN20 &                                     RN32 &                                     RN44 &                                     RN56 &                                    RN110 \\
\midrule
BN    &  \meanstd{91.6}{0.3} / \meanstdbf{91.8}{0.2} &  \meanstd{92.4}{0.1} / \meanstdbf{92.7}{0.2} &  \meanstd{92.7}{0.2} / \meanstdbf{93.1}{0.2} &  \meanstd{93.0}{0.1} / \meanstdbf{93.4}{0.2} &  \meanstd{93.5}{0.1} / \meanstdbf{93.7}{0.2} \\
\midrule
GN    &  \meanstd{90.8}{0.2} / \meanstd{90.7}{0.1} &  \meanstd{91.5}{0.2} / \meanstd{91.5}{0.1} &  \meanstd{91.8}{0.2} / \meanstd{92.0}{0.1} &  \meanstd{92.2}{0.2} / \meanstd{92.2}{0.2} &  \meanstd{92.6}{0.2} / \meanstd{92.9}{0.3} \\
GN+PN &  \meanstd{91.4}{0.2} / \meanstdbf{91.6}{0.3} &  \meanstd{92.3}{0.2} / \meanstdbf{92.5}{0.2} &  \meanstd{92.8}{0.2} / \meanstdbf{92.9}{0.2} &  \meanstd{92.9}{0.2} / \meanstdbf{93.2}{0.2} &  \meanstd{93.2}{0.1} / \meanstdbf{93.6}{0.1} \\
\bottomrule
\end{tabular}}
\end{table}

\begin{table}[H]
\centering
\caption{\textbf{BN vs. GN, GN+PN on CIFAR-100}. ResNets are trained with BN and GN, GN+PN. Results are formatted as in \Tableref{table:cifar_10}.}
\label{table:cifar_100}
\scalebox{0.85}{
\begin{tabular}{lccccc}
\\[-9pt]
\toprule
  &                                     RN20 &                                     RN32 &                                     RN44 &                                     RN56 &                                    RN110 \\
\midrule
BN    &  \meanstdbf{66.8}{0.3} / \meanstd{65.1}{0.2} &  \meanstd{68.2}{0.3} / \meanstdbf{68.7}{0.2} &  \meanstd{69.2}{0.4} / \meanstdbf{70.5}{0.2} &  \meanstd{70.1}{0.2} / \meanstdbf{71.4}{0.3} &  \meanstd{71.7}{0.3} / \meanstdbf{73.3}{0.3} \\
\midrule
GN    &  \meanstd{65.0}{0.3} / \meanstd{61.7}{0.3} &  \meanstd{66.5}{0.4} / \meanstd{65.3}{0.4} &  \meanstd{67.3}{0.6} / \meanstd{67.0}{0.3} &  \meanstd{67.8}{0.4} / \meanstd{68.1}{0.5} &  \meanstd{69.5}{0.3} / \meanstd{70.2}{0.4} \\
GN+PN &  \meanstd{66.3}{0.4} / \meanstdbf{66.7}{0.2} &  \meanstd{67.8}{0.4} / \meanstdbf{69.5}{0.2} &  \meanstd{68.9}{0.3} / \meanstdbf{70.8}{0.4} &  \meanstd{69.8}{0.3} / \meanstdbf{71.7}{0.4} &  \meanstd{71.4}{0.2} / \meanstdbf{73.1}{0.4} \\
\bottomrule
\end{tabular}}
\end{table}

\begin{figure}[H]
\centering
\includegraphics[width=0.81\textwidth]{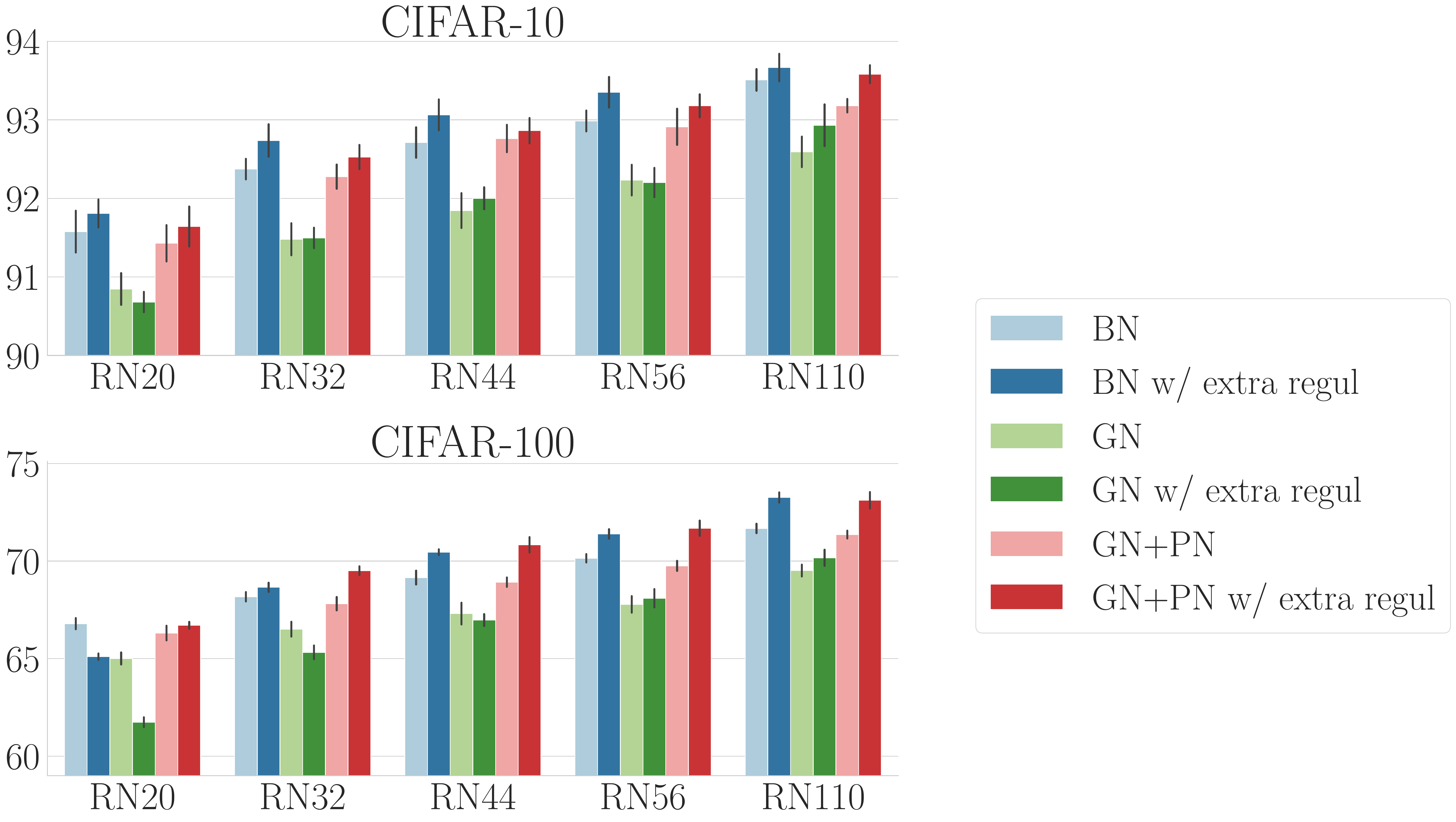}
  \caption{\textbf{BN vs. GN, GN+PN}. Validation accuracies (\%) of ResNets trained on CIFAR-10 (top) and CIFAR-100 (bottom) with BN and GN, GN+PN, without and with extra regularization.} 
  \label{fig:cifar} 
\end{figure}

\section{Implementation of Proxy Norm}
\label{sec:model_code}
In this section, we provide a Tensorflow 1 implementation of the proxy-normalized activation step when PN's additional parameters $\tilde{\vbeta}\pl$,~$\tilde{\vgamma}\pl$ are set to zero, i.e. omitted (cf \Secref{sec:pn}).

\begin{minted}{python}
import tensorflow as tf
import numpy as np
from scipy.special import erfinv
def uniformly_sampled_gaussian(num_rand):
    rand = 2 * (np.arange(num_rand) + 0.5) / float(num_rand) - 1
    return np.sqrt(2) * erfinv(rand)

def proxy_norm_act(y,
                   activation_fn=tf.nn.relu,
                   proxy_epsilon=0.03,
                   num_samples=256):
    """
    TensorFlow 1 implementation of the proxy normalized activation step.
    
    Following the same convention as in the main text of this paper, 
    the affine transform is applied in this step rather than in the normalization step.

    :param y: 4D activation tensor after the normalization step
    :param activation_fn: activation function
    :param proxy_epsilon: PN's numerical stability constant (should not be too low)
    :param num_samples: number of samples in the proxy distribution
    :return tilde_z: 4D activation tensor after the proxy-normalized activation step
    """
    def create_channelwise_variable(name, init):
        num_channels = int(y.get_shape()[-1])
        return tf.get_variable(name,
                               dtype=y.dtype,
                               shape=[1, 1, 1, num_channels],
                               initializer=tf.constant_initializer(init))
    # shift and scale parameters after the norm
    beta = create_channelwise_variable('beta', 0.0)
    gamma = create_channelwise_variable('gamma', 1.0)

    # activation step
    z = gamma * y + beta  # affine transform
    z = activation_fn(z)   # activation function 

    # proxy normalization
    proxy_y = tf.constant(uniformly_sampled_gaussian(num_samples), y.dtype)
    proxy_y = tf.reshape(proxy_y, [num_samples, 1, 1, 1])
    proxy_z = gamma * proxy_y + beta  # affine transform on proxy distribution
    proxy_z = activation_fn(proxy_z)  # activation function on proxy distribution
    # compute proxy statistics in float32
    proxy_mean, proxy_var = tf.nn.moments(
        tf.cast(proxy_z, tf.float32), axes=[0], keepdims=True)
    proxy_mean = tf.cast(proxy_mean, y.dtype)
    inv_proxy_std = tf.cast(tf.rsqrt(proxy_var + proxy_epsilon), y.dtype)
    # normalize z according to proxy statistics
    tilde_z = (z - proxy_mean) * inv_proxy_std
    return tilde_z



\end{minted}

\section{Proofs of results other than Theorems \ref{thm:ln_collapse}, \ref{thm:in_variability}, \ref{thm:pn_iterative}}
\label{sec:proof_others}

\subsection{Layer-wise power equals one}
\label{sec:powers_one}

\begin{Proposition}
\label{prop:powers_one}
If $\sigma^{}_{I_{\rvx,c}}(\rvx\pl)\neq0$ for all $\rvx\in\D$ and $c\in\{1,\dots,C\ul \}$, then it holds that $\pwr(\rvy\pl)=1$ for any choice of $\Norm \in \{ \BN, \LN, \IN, \GN \}$.
\end{Proposition}

\begin{proof}\textbf{.} The proof proceeds by distinguishing each case in $\Norm \in \{ \BN, \LN, \IN, \GN \}$.

\paragraph{Case of BN.}
If we fix a channel $c$, the assumption $\sigma^{}\uc(\rvx\pl)\neq0$ implies that
 \begin{align*}
\pwr\uc(\rvy\pl) & = \frac{\E_{\rvx,\alpha}\Big[\big(\rvx\pl_{\alpha,c} - \mu\uc(\rvx\pl) \big)^2 \Big]}{\sigma\uc(\rvx\pl)^2} = \frac{\sigma\uc(\rvx\pl)^2}{\sigma\uc(\rvx\pl)^2} = 1.
 \end{align*}

We immediately get $\pwr(\rvy\pl)=\E\uc\Big[ \pwr\uc(\rvy\pl) \Big]=1$.

\paragraph{Case of GN.}
Let us fix $\rvx\in \D$ and let us denote $\mathcal{G}_g$ for $g \in \{ 1, \dots, G \}$ the $G$ groups of channels and $I^{(g)}_{\rvx}=\{ \rvx, c \in \mathcal{G}_g \}$ for $g \in \{ 1, \dots, G \}$ the $G$ conditional sets of standardization.

The assumption $\sigma^{}_{I^{(g)}_{\rvx}}(\rvx\pl)\neq0$ implies for any $g$ that
\begin{align*}
\pwr_{I^{(g)}_{\rvx}}(\rvy\pl) & = \frac{\E_{\alpha,c|c\in \mathcal{G}_g}\Big[\big(\rvx\pl_{\alpha,c} - \mu_{I^{(g)}_{\rvx}}(\rvx\pl) \big)^2 \Big]}{\sigma_{I^{(g)}_{\rvx}}(\rvx\pl)^2} =
 \frac{\sigma_{I^{(g)}_{\rvx}}(\rvx\pl)^2}{\sigma_{I^{(g)}_{\rvx}}(\rvx\pl)^2} = 1.
 \end{align*}

This implies
\begin{align*}
\pwr_{\rvx}(\rvy\pl) & = \frac{1}{C\ul} \sum_c \pwr_{\rvx,c}(\rvy\pl) =  \frac{1}{C\ul} \sum_g  \sum_{c\in \mathcal{G}_g} \pwr_{\rvx,c}(\rvy\pl) \\
& = \frac{1}{C\ul} \sum_g |\mathcal{G}_g| \pwr_{I^{(g)}_{\rvx}}(\rvy\pl) = \frac{1}{C\ul} \sum_g |\mathcal{G}_g| =1, 
 \end{align*}
where we used $\pwr_{I^{(g)}_{\rvx}}(\rvy\pl) = \frac{1}{|\mathcal{G}_g|} \sum_{c\in \mathcal{G}_g} \pwr_{\rvx,c}(\rvy\pl)$.

We immediately get $\pwr(\rvy\pl)=\E_\rvx\Big[ \pwr_\rvx(\rvy\pl) \Big]=1$.

\paragraph{Cases of LN and IN.} The cases of LN and IN immediately follow from the cases of GN with $G=1$ group and $G=C\ul$ groups. \qed
\end{proof}

\subsection{Channel-wise collapse implies channel-wise linearity}
\label{sec:channelwise_linearity}

Some additional notations are required in this section. We denote $\Theta\pl\equiv(\vomega^1,\vbeta^1,\vgamma^1,\dots,\vomega\pl,\vbeta\pl,\vgamma\pl)$ the aggregated model parameters up to layer~$l$.

We further define the linearized post-activations $\wtz\pl$ as
\begin{align}
\forall \alpha, c: \quad & \wtz\pl_{\alpha,c} = \widetilde{\lambda}\uc \ty\pl_{\alpha,c},  \label{eq:linearized1} \\ 
\forall c: \quad & \widetilde{\lambda}\uc = \argmin_{\lambda\uc} \E_{\rvx,\alpha}\Big[ \big( \rvz\pl_{\alpha,c} - \lambda\uc \ty\pl_{\alpha,c}  \big)^2 \Big] = \argmin_{\lambda\uc} \E_{\rvx,\alpha}\Big[ \big( \phi(\ty\pl_{\alpha,c}) - \lambda\uc \ty\pl_{\alpha,c} \big)^2 \Big].  \label{eq:linearized2}
\end{align}
The linearized post-activations $\wtz\pl$ are the channel-wise linear best-fit of $\rvz\pl=\phi(\ty\pl)$ using $\ty\pl$.

We start by proving that the inequality $\pwr\uc(\ty\pl) - \pwrone\uc(\ty\pl) \leq \tilde{\eta} \pwrone\uc(\ty\pl)$ for sufficiently small $\tilde{\eta}$ implies channel-wise linearity (\Propref{prop:channelwise_linearity1}). We then prove that the inequality $\pwr\uc(\ty\pl) - \pwrone\uc(\ty\pl) \leq \widetilde{\eta} \pwr\uc(\ty\pl)$ for sufficiently small $\widetilde{\eta}$ implies channel-wise-linearity (\Propref{prop:channelwise_linearity2}). 

\begin{Proposition}
\label{prop:channelwise_linearity1}
If we fix some $d \in \sN^{*}$, there exists $\tilde{\eta}>0$ such that for any choice of $(\phi,H,W,\D,\Theta\pl)$, it holds that
\begin{align*}
\Big( HW|\D|=d \Big) \land \Big( \pwr\uc(\ty\pl) - \pwrone\uc(\ty\pl) \leq \tilde{\eta} \pwrone\uc(\ty\pl) \Big)  \implies \rvz\pl_{\alpha,c} = \wtz\pl_{\alpha,c} \;  \forall \rvx,\alpha,
 \end{align*}
where $\wtz\pl$ are the linearized post-activations defined in \Eqref{eq:linearized1}, (\ref{eq:linearized2}) and $\land$ is the logical \say{and}.
\end{Proposition}

\begin{proof}\textbf{.} Any positive homogeneous $\phi$ satisfies $\phi(r)=r\phi(1)$ and $\phi(-r)=r\phi(-1)$ for any $r\geq0$. This means that any positive homogeneous $\phi$ is: (i) fully determined by its values at $+1$ and $-1$; (ii) linear on the intervals $(-\infty,0]$ and $[0,+\infty)$.

A sufficient condition for the linearity with respect to $\rvx,\alpha$ in channel $c$ is therefore a constant sign of $\ty\pl_{\alpha,c}$ for all $\rvx,\alpha$. Let us see that this constant sign is implied by a sufficiently severe channel-wise collapse.

We start by proving the result with the two distinct conditionalities: (i) $\sigma\uc(\ty\pl)=0$ and (ii) $\sigma\uc(\ty\pl)>0$.

\paragraph{Conditionality $\sigma\uc(\ty\pl)=0$.}
If $\sigma\uc(\ty\pl)=0$, then $\ty\pl_{\alpha,c}=\mu\uc(\ty\pl)$, $\forall \rvx,\alpha$.

Let us then define $\lambda\uc$ such that $\lambda\uc=0$ if $\mu\uc(\ty\pl)=0$, and $\lambda\uc=\frac{\phi(\mu\uc(\ty\pl))}{\mu\uc(\ty\pl)}$ otherwise.

For any choice of positive homogeneous $\phi$, it holds that $\phi(0)=0$. Combined with the definition of $\lambda\uc$, this implies $\phi(\mu\uc(\ty\pl))=\lambda\uc\mu\uc(\ty\pl)$ and thus $\forall \rvx,\alpha$:
\begin{align*}
\rvz\pl_{\alpha,c}=\phi(\ty\pl_{\alpha,c})=\phi(\mu\uc(\ty\pl))=\lambda\uc\mu\uc(\ty\pl)= \lambda\uc\ty\pl_{\alpha,c}.
\end{align*}

Given the definition of the linearized post-activations $\wtz\pl$, this means 
\begin{align*}
\E_{\rvx,\alpha}\Big[ \big(\rvz\pl_{\alpha,c} - \wtz\pl_{\alpha,c} \big)^2 \Big]\leq \E_{\rvx,\alpha}\Big[ \big(\rvz\pl_{\alpha,c} - \lambda\uc\ty\pl_{\alpha,c} \big)^2 \Big]=0.
\end{align*}

This immediately implies $\forall \rvx,\alpha$: $\rvz\pl_{\alpha,c}=\wtz\pl_{\alpha,c}$. Thus, for any $(\phi,H,W,\D,\Theta\pl)$ such that $\sigma\uc(\ty\pl)=0$, it holds that $\forall \rvx,\alpha$: $\rvz\pl_{\alpha,c}=\wtz\pl_{\alpha,c}$. 

More concisely, it holds for any choice of $(\phi,H,W,\D,\Theta\pl)$ that
\begin{align}
\sigma\uc(\ty\pl) = 0 
\implies \rvz\pl_{\alpha,c}=\wtz\pl_{\alpha,c} \;  \forall \rvx,\alpha. \label{eq:channelwise_linearity1}
\end{align}

\paragraph{Conditionality $\sigma\uc(\ty\pl)>0$.}
We start by fixing $(\phi,H,W,\D,\Theta\pl)$. For any given $k>0$, Chebyshev's inequality implies
\begin{align*}
\P_{\rvx,\alpha}\Big[ |\ty\pl_{\alpha,c}-\mu\uc(\ty\pl)| \geq k\sigma\uc(\ty\pl) \Big] & \leq \frac{1}{k^2}, \\
\P_{\rvx,\alpha}\Big[ \big(\ty\pl_{\alpha,c}-\mu\uc(\ty\pl)\big)^2 \geq k^2\sigma\uc(\ty\pl)^2 \Big] & \leq \frac{1}{k^2}.
\end{align*}

Thus, if $\pwr\uc(\ty\pl) - \pwrone\uc(\ty\pl) \leq \eta \pwrone\uc(\ty\pl)$ for some $\eta>0$, it holds for any given $k>0$ that
\begin{align*}
1-\frac{1}{k^2} 
& \leq \P_{\rvx,\alpha}\Big[ \big(\ty\pl_{\alpha,c}-\mu\uc(\ty\pl)\big)^2 < k^2\sigma\uc(\ty\pl)^2 \Big]  
 = \P_{\rvx,\alpha}\Big[ \big(\ty\pl_{\alpha,c}-\mu\uc(\ty\pl)\big)^2 < k^2 \big(\pwr\uc(\ty\pl) - \pwrone\uc(\ty\pl)\big) \Big] \\
& \leq \P_{\rvx,\alpha}\Big[ \big(\ty\pl_{\alpha,c}-\mu\uc(\ty\pl)\big)^2 < k^2\eta\pwrone\uc(\ty\pl) \Big] 
 = \P_{\rvx,\alpha}\Big[ \big(\ty\pl_{\alpha,c}-\mu\uc(\ty\pl)\big)^2 < k^2\eta\mu\uc(\ty\pl)^2 \Big]  \\
& \leq \P_{\rvx,\alpha}\Big[ |\ty\pl_{\alpha,c}-\mu\uc(\ty\pl)| < k \sqrt{\eta} |\mu\uc(\ty\pl)| \Big]. 
\end{align*}

Choosing $k=\frac{1}{\sqrt{\eta}}$, we get
\begin{align*}
\P_{\rvx,\alpha}\Big[ |\ty\pl_{\alpha,c}-\mu\uc(\ty\pl)| \geq |\mu\uc(\ty\pl)| \Big] \leq \eta.
\end{align*}

Now if we suppose that $\eta$ is such that $\frac{1}{HW|\D|}>\eta>0$, we get
\begin{align}
\P_{\rvx,\alpha}\Big[ |\ty\pl_{\alpha,c}-\mu\uc(\ty\pl)| \geq |\mu\uc(\ty\pl)| \Big] < \frac{1}{HW|\D|}. \label{eq:channelwise_linearity2}
\end{align}

\Eqref{eq:channelwise_linearity2} could not hold if there existed $\rvx\in \D$ and $\alpha \in \{1,\dots,H\} \times \{1,\dots,W\}$ such that $|\ty\pl_{\alpha,c}-\mu\uc(\ty\pl)| \geq |\mu\uc(\ty\pl)|$. Consequently, $|\ty\pl_{\alpha,c}-\mu\uc(\ty\pl)| < |\mu\uc(\ty\pl)|$ for all $\rvx,\alpha$, implying that there exists a tensor $\rvr\pl\in \R^{H\times W\times C\ul}$ that implicitly depends on $\rvx$ such that $\forall \rvx,\alpha$:
\begin{align}
\ty\pl_{\alpha,c} = \rvr\pl_{\alpha,c} \mu\uc(\ty\pl), \qquad  \rvr\pl_{\alpha,c} \geq 0. \label{eq:channelwise_linearity3}
\end{align}

Now if we combine $\sigma\uc(\ty\pl)>0$ with $\pwr\uc(\ty\pl)-\pwrone\uc(\ty\pl) \leq \eta \pwrone\uc(\ty\pl)$, we deduce that $\pwrone\uc(\ty\pl)>0$ and thus that $\mu\uc(\ty\pl) \neq 0$. Combining this with \Eqref{eq:channelwise_linearity3}, we get $\forall \rvx,\alpha$:
\begin{align*}
\rvz\pl_{\alpha,c} = \phi(\ty\pl_{\alpha,c}) = \rvr\pl_{\alpha,c} \phi(\mu\uc(\ty\pl)) 
= \frac{\phi(\mu\uc(\ty\pl))}{\mu\uc(\ty\pl)} \rvr\pl_{\alpha,c} \mu\uc(\ty\pl)
= \lambda\uc \ty\pl_{\alpha,c},
\end{align*}
where we defined $\lambda\uc\equiv\frac{\phi(\mu\uc(\ty\pl))}{\mu\uc(\ty\pl)}$.

Given the definition of the linearized post-activations $\wtz\pl$, this means 
\begin{align*}
\E_{\rvx,\alpha}\Big[ \big(\rvz\pl_{\alpha,c} - \wtz\pl_{\alpha,c} \big)^2 \Big]\leq \E_{\rvx,\alpha}\Big[ \big(\rvz\pl_{\alpha,c} - \lambda\uc\ty\pl_{\alpha,c} \big)^2 \Big]=0.
\end{align*}

This immediately implies $\forall \rvx,\alpha$: $\rvz\pl_{\alpha,c}=\wtz\pl_{\alpha,c}$. 

Thus, if we fix some $d \in \sN^{*}$ and if we define $\tilde{\eta}=\frac{1}{2d}$, it holds for any choice of $(\phi,H,W,\D,\Theta\pl)$ such that (i) $HW|\D|=d$, (ii) $\pwr\uc(\ty\pl) - \pwrone\uc(\ty\pl) \leq \tilde{\eta} \pwrone\uc(\ty\pl)$, (iii) $\sigma\uc(\ty\pl)>0$, that $\forall \rvx,\alpha$: $\rvz\pl_{\alpha,c}=\wtz\pl_{\alpha,c}$. 

More concisely, it holds for any choice of $(\phi,H,W,\D,\Theta\pl)$ that
\begin{align}
\Big( HW|\D| = d \Big) 
\land \Big( \pwr\uc(\ty\pl) - \pwrone\uc(\ty\pl) \leq \tilde{\eta} \pwrone\uc(\ty\pl) \Big)
\land \Big( \sigma\uc(\ty\pl) > 0 \Big) 
\implies \rvz\pl_{\alpha,c}=\wtz\pl_{\alpha,c} \;  \forall \rvx,\alpha. \label{eq:channelwise_linearity4}
\end{align}

\paragraph{General case.} To wrap up, if we fix some $d \in \sN^{*}$ and if we reuse the definition $\tilde{\eta}=\frac{1}{2d}$, it holds for any choice of $(\phi,H,W,\D,\Theta\pl)$ that
\begin{align}
& \Big( HW|\D| = d \Big) 
\land \Big( \pwr\uc(\ty\pl) - \pwrone\uc(\ty\pl) \leq\tilde{\eta} \pwrone\uc(\ty\pl) \Big)\nonumber \\
& \qquad \implies \Bigg( \Big( HW|\D| = d \Big) \land \Big( \pwr\uc(\ty\pl) - \pwrone\uc(\ty\pl) \leq \tilde{\eta} \pwrone\uc(\ty\pl) \Big) \land \Big( \sigma\uc(\ty\pl) = 0 \Big) \Bigg)\nonumber \\
& \qquad \qquad \lor \Bigg( \Big( HW|\D| = d \Big) \land \Big( \pwr\uc(\ty\pl) - \pwrone\uc(\ty\pl) \leq \tilde{\eta} \pwrone\uc(\ty\pl) \Big) \land \Big( \sigma\uc(\ty\pl) > 0 \Big) \Bigg) \nonumber \\
& \qquad \implies \Big( \rvz\pl_{\alpha,c}=\wtz\pl_{\alpha,c} \;  \forall \rvx,\alpha \Big) \lor \Big( \rvz\pl_{\alpha,c}=\wtz\pl_{\alpha,c} \;  \forall \rvx,\alpha \Big) \label{eq:channelwise_linearity5} \\
& \qquad \implies \Big( \rvz\pl_{\alpha,c}=\wtz\pl_{\alpha,c} \;  \forall \rvx,\alpha \Big), \nonumber 
\end{align}
where \Eqref{eq:channelwise_linearity5} is obtained using \Eqref{eq:channelwise_linearity1} and \Eqref{eq:channelwise_linearity4} and $\land$, $\lor$ are the logical \say{and} and \say{or}. \qed
\end{proof}

\begin{Proposition}
If we fix some $d \in \sN^{*}$, there exists $\widetilde{\eta}>0$ such that for any choice of $(\phi,H,W,\D,\Theta\pl)$, it holds that
\begin{align*}
\Big( HW|\D|=d \Big) \land \Big( \pwr\uc(\ty\pl) - \pwrone\uc(\ty\pl) \leq \widetilde{\eta} \pwr\uc(\ty\pl) \Big)  \implies \rvz\pl_{\alpha,c} = \wtz\pl_{\alpha,c} \;  \forall \rvx,\alpha,
 \end{align*}
where $\wtz\pl$ are the linearized post-activations defined in \Eqref{eq:linearized1}, (\ref{eq:linearized2}) and $\land$ is the logical \say{and}.
\label{prop:channelwise_linearity2}
\end{Proposition}

\begin{proof}\textbf{.}
We start by noting that for any $1>\eta>0$:
\begin{align*}
 \pwr\uc(\ty\pl) - \pwrone\uc(\ty\pl) \leq \eta \pwr\uc(\ty\pl) & \iff  \pwr\uc(\ty\pl) - \pwrone\uc(\ty\pl) \leq \eta \Big( \pwr\uc(\ty\pl) - \pwrone\uc(\ty\pl) + \pwrone\uc(\ty\pl) \Big) \\
& \iff \pwr\uc(\ty\pl) - \pwrone\uc(\ty\pl) \leq \frac{\eta}{1-\eta} \pwrone\uc(\ty\pl).
 \end{align*}

Thus, if we fix some $d \in \sN^{*}$ and if we define $\tilde{\eta}=\frac{1}{2d}$ and $\widetilde{\eta}=\frac{\tilde{\eta}}{1+\tilde{\eta}}$, it holds for any choice of $(\phi,H,W,\D,\Theta\pl)$ that
\begin{align}
& \Big( HW|\D| = d \Big) 
\land \Big( \pwr\uc(\ty\pl) - \pwrone\uc(\ty\pl) \leq\widetilde{\eta} \pwr\uc(\ty\pl) \Big)  \nonumber \\
& \qquad \iff \Big( HW|\D| = d \Big) \land \Big( \pwr\uc(\ty\pl) - \pwrone\uc(\ty\pl) \leq\frac{\widetilde{\eta}}{1-\widetilde{\eta}} \pwrone\uc(\ty\pl) \Big) \nonumber \\
& \qquad \iff \Big( HW|\D| = d \Big) \land \Big( \pwr\uc(\ty\pl) - \pwrone\uc(\ty\pl) \leq\tilde{\eta} \pwrone\uc(\ty\pl) \Big)\nonumber \\
& \qquad \implies \Big( \rvz\pl_{\alpha,c}=\wtz\pl_{\alpha,c} \;  \forall \rvx,\alpha \Big), \label{eq:channelwise_linearity6}
\end{align}
where \Eqref{eq:channelwise_linearity6} is obtained using \Propref{prop:channelwise_linearity1} and $\land$ is the logical \say{and}. \qed
\end{proof}

\subsection{Alteration of expressivity with IN}
\label{sec:alteration_IN}

In this section, we first prove that, for any dataset $\D$, networks without normalization can express mappings arbitrarily close to the identity (\Propref{prop:identity_nonorm}). We then prove that, in general, networks with IN cannot express mappings arbitrarily close to the identity (\Propref{prop:identity_in}).

\begin{Proposition}
\label{prop:identity_nonorm}
Lift any assumptions on $\phi$ and suppose instead that $\phi$ is non-polynomial. Further suppose that each layer $l$ up to depth $L$ implements the following two steps $\forall \alpha,c$:
\begin{align}
\rvy\pl_{\alpha,c} & = (\vomega\pl \ast \rvz\plm)_{\alpha,c} + \rvb\pl\uc, \label{eq_identity_nonorm1} \\
 \rvz\pl_{\alpha,c} & = \phi\big(\rvy\pl_{\alpha,c} \big), \label{eq_identity_nonorm2} 
\end{align}
where $\rvz\po\equiv\rvx$, and $\vomega\pl \in \R^{K\ul \times K\ul \times C\ulm \times C\ul}$ and $\rvb\pl \in \R^{C\ul}$ are the weights and biases at layer $l$.

Now fix a layer $l\in\{1,\dots,L \}$, the spatial extents $H$, $W$, the widths $C\uo$, $C\ul$ assumed equal at layer~$0$ and layer~$l$, and the dataset $\D$. Further denote $\Phi\ul$ the network mapping from $\rvx$ to $\rvy\pl$ such that $\rvy\pl=\Phi\ul(\rvx)$. 

Then for any $\epsilon>0$, there exists a choice of intermediate widths $(C\uk)_{1\leq k <l}$ and model parameters $(\vomega^1,\rvb^1,\dots,\vomega\pl,\rvb\pl)$ such that
\begin{align*}
\max_{\rvx\in \D} ||\Phi\ul(\rvx) - \rvx || \leq \epsilon.
\end{align*}
\end{Proposition}

\begin{proof}\textbf{.} The proof proceeds in multiple steps of increasing generality.

\paragraph{Case of unit spatial extent of activations and kernels.} When $H$, $W$ are equal to $1$ and $K\uk$ is equal to one at every layer $k$, the propagation of \Eqref{eq_identity_nonorm1}, (\ref{eq_identity_nonorm2}) becomes strictly equivalent to the propagation in a fully-connected network. 

If $l$ is the first layer in the network ($l=1$), we may obtain the strict equality $\Phi\ul=\identity$ by choosing the reshaped matricial version $\mW\pl\in\R^{C\uo \times C\ul}$ of $\vomega\pl$ as the identity and $\rvb\pl$ as zero.

Otherwise ($l\geq 2$), we may apply the universal approximation theorem \citep{Leshno93}. Given the assumption of non-polynomial activation function $\phi$, this means that for any $\epsilon>0$, there exists a choice of intermediate widths $(C\uk)_{1\leq k <l}$ and model parameters $(\vomega^1,\rvb^1,\dots,\vomega\pl,\rvb\pl)$ such that
\begin{align}
\max_{\rvx\in \D} ||\Phi\ul(\rvx) - \rvx || \leq \epsilon. \label{eq_identity_nonorm3} 
\end{align}

\paragraph{Case of unit spatial extent of kernels.} When $K\uk$ is equal to one at every layer $k$, the propagation of \Eqref{eq_identity_nonorm1}, (\ref{eq_identity_nonorm2}) \say{occurs} strictly independently for each spatial position $\alpha$.

Let us then consider a neural network that takes an input $\bar{\rvx}\in \R^{1\times 1\times C\uo}$ and provides \mbox{$\bar{\rvy}\pl$, $\bar{\rvz}\pl$} at every layer $l$ by implementing the same steps as \Eqref{eq_identity_nonorm1}, (\ref{eq_identity_nonorm2}). Let us denote $\Psi\ul$ the network mapping from $\bar{\rvx}$ to $\bar{\rvy}\pl$ such that $\bar{\rvy}\pl=\Psi\ul(\bar{\rvx})$. And let us further denote $\overline{\D}=\{ \bar{\rvx}^{(\rvx,\alpha)} \}_{\rvx\in \D, \alpha\in \{1,\dots,H\} \times \{1,\dots,W\}}$, where $\bar{\rvx}^{(\rvx,\alpha)}\in \R^{1\times1\times C\uo}$ denotes the reshaped version of $\rvx_{\alpha,:}\in \R^{C\uo}$ for any $\rvx$, $\alpha$.

If we fix any $\epsilon>0$ and if we apply \Eqref{eq_identity_nonorm3} from the previous case with $\frac{1}{\sqrt{H W}} \epsilon$, we get that there exists a choice of intermediate widths \mbox{$(C\uk)_{1\leq k <l}$ and model parameters $(\vomega^1,\rvb^1,\dots,\vomega\pl,\rvb\pl)$ such that}
\begin{align}
\max_{\bar{\rvx}\in \overline{\D}} ||\Psi\ul(\bar{\rvx}) - \bar{\rvx} || 
= \max_{\rvx\in \D,\alpha \in \{1,\dots,H\} \times \{1,\dots,W\}} ||\Psi\ul(\bar{\rvx}^{(\rvx,\alpha)}) - \bar{\rvx}^{(\rvx,\alpha)} || 
\leq \frac{1}{\sqrt{H W}} \epsilon.  \label{eq_identity_nonorm4} 
\end{align}
Let us then fix $(C\uk)_{1\leq k <l}$ and $(\vomega^1,\rvb^1,\dots,\vomega\pl,\rvb\pl)$ such that \Eqref{eq_identity_nonorm4} holds.

Due to the independence of spatial positions, the mapping $\Phi\ul$ is such that $\Phi\ul(\rvx)_{\alpha,:}$ is a reshaped version of $\Psi\ul(\bar{\rvx}^{(\rvx,\alpha)})$ for any $\rvx$, $\alpha$. This means that $\forall \rvx\in \D$ and $\forall \alpha \in \{1,\dots,H\} \times \{1,\dots,W\}$:
\begin{align*}
||\Phi\ul(\rvx)_{\alpha,:} - \rvx_{\alpha,:} || & = || \Psi\ul(\bar{\rvx}^{(\rvx,\alpha)}) - \bar{\rvx}^{(\rvx,\alpha)}|| \leq \frac{1}{\sqrt{H W}} \epsilon, \\
||\Phi\ul(\rvx) - \rvx ||^2 & = \sum_{ \alpha \in \{1,\dots,H\} \times \{1,\dots,W\}} ||\Phi\ul(\rvx)_{\alpha,:} - \rvx_{\alpha,:} ||^2 
  \leq \sum_{ \alpha \in \{1,\dots,H\} \times \{1,\dots,W\}}  \frac{1}{H W} \epsilon^2  \leq \epsilon^2.
\end{align*}
This immediately implies
\begin{align*}
\max_{\rvx\in \D} ||\Phi\ul(\rvx) - \rvx || \leq \epsilon.
\end{align*}

\paragraph{General case.} Let us consider the neural network that takes $\rvx\in \D$ as input and provides $\bar{\rvy}\pl$, $\bar{\rvz}\pl$ at every layer $l$ by implementing \Eqref{eq_identity_nonorm1}, (\ref{eq_identity_nonorm2}) with weights $\bar{\vomega}\pl \in \R^{1 \times 1 \times C\ulm \times C\ul}$, biases $\bar{\rvb}\pl \in \R^{C\ul}$ and activation function $\phi$. Let us then denote $\Psi\ul$ the network mapping from $\rvx$ to $\bar{\rvy}\pl$ such that $\bar{\rvy}\pl=\Psi\ul(\rvx)$.

If we fix any $\epsilon>0$, we get from the previous case that there exists a choice of intermediate widths $(C\uk)_{1\leq k <l}$ and model parameters $(\bar{\vomega}^1,\bar{\rvb}^1,\dots,\bar{\vomega}\pl,\bar{\rvb}\pl)$ such that
\begin{align}
\max_{\rvx\in \D} ||\Psi\ul(\rvx) - \rvx || \leq \epsilon.  \label{eq_identity_nonorm5}
\end{align}
Let us then fix $(C\uk)_{1\leq k <l}$ and $(\bar{\vomega}^1,\bar{\rvb}^1,\dots,\bar{\vomega}\pl,\bar{\rvb}\pl)$ such that \Eqref{eq_identity_nonorm5} holds. Let us further define the weights and biases $\vomega\pk$, $\rvb\pk$ at each layer $k$ such that $\forall h,w,c,c'$:
\begin{align*}
\vomega\pk_{h,w,c,c'} & \equiv \left\{
    \begin{array}{cc}
        & \text{if the multi-index $(h,w,c,c')$ in the weights} \\
        \bar{\vomega}\pk_{1,1,c,c'} & \text{associates spatial positions $\alpha$ in the convolution input $\rvz\plm_{\alpha,c}$} \\
        & \text{to the same spatial positions $\alpha$ in the convolution output $\rvy\pl_{\alpha,c'}$}, \\
        \\
        0 & \text{otherwise,} 
    \end{array}
\right. \\
\rvb\pk_{c'} & \equiv \bar{\rvb}\pk_{c'}.
\end{align*}

Then it holds that $\Phi\ul=\Psi\ul$, which in turn implies
\begin{align*}
\max_{\rvx\in \D} ||\Phi\ul(\rvx) - \rvx || \leq \epsilon.  \tag*{\qed}
\end{align*}
\end{proof}

\begin{Proposition}
\label{prop:identity_in}
Suppose that the neural network implements \Eqref{eq:propagation1}, (\ref{eq:propagation2}), (\ref{eq:propagation3}) in every layer up to depth $L$ and suppose $\Norm=\IN$.

Further fix a layer $l\in\{1,\dots,L\}$, the spatial extents $H$, $W$, the widths $C\uo$, $C\ul$ assumed equal at layer~$0$ and layer~$l$, and any dataset $\D$ such that there exists at least one channel in which the inputs of $\D$ do not all share the same statistics of instance mean, i.e.
\begin{align*}
\exists c,\; \exists \rvx', \rvx'' \in \D: \qquad \E_\alpha[\rvx'_{\alpha,c}] \neq \E_\alpha[\rvx''_{\alpha,c}].
\end{align*}
Then there exists $\epsilon>0$ such that for any choice of intermediate widths $(C\uk)_{1\leq k <l}$ and model parameters $(\vomega^1,\vbeta^1,\vgamma^1,\dots,\vomega\pl,\vbeta\pl,\vgamma\pl)$, it holds that
\begin{align*}
\max_{\rvx\in \D} ||\Phi\ul(\rvx) - \rvx || > \epsilon,
\end{align*}
where $\Phi\ul$ denotes the network mapping from $\rvx$ to $\ty\pl$ such that $\ty\pl=\Phi\ul(\rvx)$, $\forall \rvx$. 
\end{Proposition}

\begin{proof}\textbf{.} 
Let us proceed by contradiction and suppose that for any $\epsilon>0$, there exists a choice of intermediate widths $(C\uk)_{1\leq k <l}$ and model parameters $(\vomega^1,\vbeta^1,\vgamma^1,\dots,\vomega\pl,\vbeta\pl,\vgamma\pl)$ such that
\begin{align}
\max_{\rvx\in \D} ||\Phi\ul(\rvx) - \rvx || \leq \epsilon. \label{eq_identity_in1}
\end{align}
Given the assumption on $\D$, there exists some channel $c$ and some inputs $\rvx', \rvx'' \in \D$ such that
\begin{align}
\E_\alpha[\rvx'_{\alpha,c}] \neq \E_\alpha[\rvx''_{\alpha,c}]. \label{eq_identity_in2}
\end{align}
Let us then fix $c$ and $\rvx', \rvx'' \in \D$ satisfying \Eqref{eq_identity_in2}, and let us define $\eta \equiv \big| \E_\alpha[\rvx'_{\alpha,c}] - \E_\alpha[\rvx''_{\alpha,c}] \big| > 0$.

Applying \Eqref{eq_identity_in1} with $\epsilon=\frac{\sqrt{HW}}{4} \eta$, we get that there exists a choice of intermediate widths $(C\uk)_{1\leq k <l}$ and model parameters $(\vomega^1,\vbeta^1,\vgamma^1,\dots,\vomega\pl,\vbeta\pl,\vgamma\pl)$ such that
\begin{align}
\max_{\rvx\in\D} ||\Phi\ul(\rvx) - \rvx|| & \leq \frac{\sqrt{HW}}{4} \eta. \label{eq_identity_in3}
\end{align}

Let us then fix $(C\uk)_{1\leq k <l}$ and $(\vomega^1,\vbeta^1,\vgamma^1,\dots,\vomega\pl,\vbeta\pl,\vgamma\pl)$ such that \Eqref{eq_identity_in3} holds. This gives
\begin{align}
\sum_\alpha \big(\Phi\ul(\rvx')_{\alpha,c} - \rvx'_{\alpha,c}\big)^2 & \leq ||\Phi\ul(\rvx') - \rvx'||^2 \leq \Big( \frac{\sqrt{HW}}{4} \eta \Big)^2, \nonumber \\
\sum_\alpha \big(\Phi\ul(\rvx'')_{\alpha,c} - \rvx''_{\alpha,c}\big)^2 & \leq ||\Phi\ul(\rvx'') - \rvx''||^2 \leq \Big( \frac{\sqrt{HW}}{4} \eta \Big)^2, \nonumber \\
\E_\alpha \Big[ \big(\Phi\ul(\rvx')_{\alpha,c} - \rvx'_{\alpha,c}\big)^2 \Big]
+ \E_\alpha \Big[ \big(\Phi\ul(\rvx'')_{\alpha,c} - \rvx''_{\alpha,c}\big)^2 \Big] & \leq \frac{2}{HW} \Big( \frac{\sqrt{HW}}{4} \eta \Big)^2. \label{eq_identity_in4}
\end{align}

At the same time, for any input $\rvx$, it holds that
\begin{align}
\E_\alpha \Big[\big(\Phi\ul(\rvx)_{\alpha,c} - \rvx_{\alpha,c}\big)^2 \Big] 
& = \E_\alpha\Big[\Phi\ul(\rvx)_{\alpha,c} - \rvx_{\alpha,c} \Big]^2  
+ \Var_\alpha\Big[\Phi\ul(\rvx)_{\alpha,c} - \rvx_{\alpha,c} \Big] \nonumber \\
& \geq \E_\alpha\Big[\Phi\ul(\rvx)_{\alpha,c} - \rvx_{\alpha,c} \Big]^2. \label{eq_identity_in5}
\end{align}

Using $\forall a,b$: $(a-b)^2 \leq 2a^2 + 2b^2$, combined with \Eqref{eq_identity_in5} and \Eqref{eq_identity_in4}, we get
\begin{align}
\Bigg( \E_\alpha\Big[\Phi\ul(\rvx')_{\alpha,c} - \rvx'_{\alpha,c} \Big] 
& - \E_\alpha\Big[\Phi\ul(\rvx'')_{\alpha,c} - \rvx''_{\alpha,c} \Big] \Bigg)^2 \nonumber \\
& \leq 2 \E_\alpha\Big[\Phi\ul(\rvx')_{\alpha,c} - \rvx'_{\alpha,c} \Big]^2
+ 2 \E_\alpha\Big[\Phi\ul(\rvx'')_{\alpha,c} - \rvx''_{\alpha,c} \Big]^2 \nonumber \\
& \leq 2\E_\alpha \Big[ \big(\Phi\ul(\rvx')_{\alpha,c} - \rvx'_{\alpha,c}\big)^2 \Big]
+ 2\E_\alpha \Big[ \big(\Phi\ul(\rvx'')_{\alpha,c} - \rvx''_{\alpha,c}\big)^2 \Big] \nonumber \\
& \leq \frac{4}{HW} \Big( \frac{\sqrt{HW}}{4} \eta \Big)^2. \label{eq_identity_in6}
\end{align}

Next, we note that with IN all inputs $\rvx$ are associated to the same instance means in each channel of $\ty\pl=\Phi\ul(\rvx)$. This means in particular that
\begin{align}
\E_\alpha\Big[\Phi\ul(\rvx')_{\alpha,c}\Big] & = \E_\alpha\Big[\Phi\ul(\rvx'')_{\alpha,c}\Big]. \label{eq_identity_in7}
\end{align}

Combining \Eqref{eq_identity_in6} with \Eqref{eq_identity_in7}, we get
\begin{align*}
\Big( \E_\alpha[\rvx'_{\alpha,c}] - \E_\alpha[\rvx''_{\alpha,c}] \Big)^2 
& \leq \frac{4}{HW} \Big(\frac{\sqrt{HW}}{4} \eta \Big)^2, \\
\big| \E_\alpha[\rvx'_{\alpha,c}] - \E_\alpha[\rvx''_{\alpha,c}] \big| 
& \leq \frac{2}{\sqrt{HW}}\frac{\sqrt{HW}}{4} \eta = \frac{\eta}{2}.
\end{align*}

Since we earlier defined $\eta$ as $\eta \equiv \big| \E_\alpha[\rvx'_{\alpha,c}] - \E_\alpha[\rvx''_{\alpha,c}] \big| > 0$, we reach a contradiction. \qed
\end{proof}

\section{Proof of \Thmref{thm:ln_collapse}}
\label{sec:ln_collapse}

\subsection{Additional notations}

Some additional notations are required in this section. 

\paragraph{Model parameters.} We introduce the notations $\theta\pl\equiv(\vomega\pl,\vbeta\pl,\vgamma\pl)$ for the model parameters at layer~$l$ and $\Theta\pl\equiv(\vomega^1,\vbeta^1,\vgamma^1,\dots,\vomega\pl,\vbeta\pl,\vgamma\pl)$ for the aggregated model parameters up to layer~$l$.

\paragraph{Activation tensors.} For each layer $l$, we define the tensors $\hz\plm$, $\hx\pl$, $\hy\pl$, $\cy\pl$, $\cz\pl$ such that $\forall \alpha,c$:
\begin{align}
\hz\plm_{\alpha,c} & = \sqrt{\frac{\pwr(\rvz\plm)}{\pwr_\rvx(\rvz\plm)}} \rvz\plm_{\alpha,c}, \label{eq:additional1} \\
\hx\pl_{\alpha,c} & = \sqrt{\frac{\pwr(\rvz\plm)}{\pwr_\rvx(\rvz\plm)}} \rvx\pl_{\alpha,c} = (\vomega\pl \ast \hz\plm)_{\alpha,c},  \label{eq:additional2} \\
\hy\pl_{\alpha,c} & = \frac{1}{\omega\sqrt{\pwr(\rvz\plm)}} \hx\pl_{\alpha,c}, \label{eq:additional3} \\
\cy\pl_{\alpha,c} & = \vgamma\pl\uc \hy\pl_{\alpha,c} + \vbeta\pl\uc,\label{eq:additional4} \\
\cz\pl_{\alpha,c} & = \phi(\cy\pl_{\alpha,c}), \label{eq:additional5}
\end{align}
with the convention that, if $\pwr_\rvx(\rvz\plm)=0$, then $\forall \alpha,c$: $\hz\plm_{\alpha,c}=0$, $\hx\pl_{\alpha,c}=0$ and if $\pwr(\rvz\plm)=0$, then $\forall \alpha,c$: $\hy\pl_{\alpha,c}=0$.

\paragraph{Moments.} We introduce the notation $\varrho(\rvy\pl)$ for the ratio of the traces of the covariance matrix and Gram matrix of the activation vectors $(\rvy\pl_{\alpha,1},\dots,\rvy\pl_{\alpha,C\ul})^\mathrm{T}$ with respect to the randomness from $(\rvx,\alpha)$, i.e.
 \begin{align*}
\varrho(\rvy\pl) & \equiv \frac{\pwr(\rvy\pl)-\pwrone(\rvy\pl)}{\pwr(\rvy\pl)} \leq 1,
 \end{align*}
with the convention that, if $\pwr(\rvy\pl)=0$, then $\varrho(\rvy\pl)=0$.

We extend the definition of the terms $\pwrone\uc(\rvy\pl)$, $\pwrtwo\uc(\rvy\pl)$, $\pwrthree\uc(\rvy\pl)$, $\pwrfour\uc(\rvy\pl)$ and $\pwrone(\rvy\pl)$, $\pwrtwo(\rvy\pl)$, $\pwrthree(\rvy\pl)$, $\pwrfour(\rvy\pl)$, $\varrho(\rvy\pl)$ to all the other activation tensors of layer $l$.

\subsection{Required Lemmas}

\begin{Lemma}
\label{lemma:ln_collapse1}
Fix a layer $l\geq1$, $\nu_{\vomega}$, $\nu_{\vbeta}$, $\nu_{\vgamma}$, $\D$ in \Defref{def:random} and model parameters $\Theta\plm$ up to layer $l-1$ such that $\pwr_\rvx(\rvz\plm)>0$, $\forall \rvx$. Further suppose $\Norm=\LN$ and suppose that the convolution of \Eqref{eq:propagation1} uses periodic boundary conditions.

Then for any $\eta>0$ and any $\delta>0$, there exists $N'(\eta,\delta)\in \mathbb{N}^*$ independent of $\Theta\plm$, $l$ such that if $C\ul \geq N'(\eta,\delta)$, it holds for random nets of \Defref{def:random} that
 \begin{align*}
\P_{\theta\ul}\Big[ |\varrho(\hy\pl) - \varrho(\hz\plm)| \leq \eta   \Big]  &\geq 1-\delta,  \\
\P_{\theta\ul}\Big[ |\varrho(\cz\pl) - \rho \chi(\hz\plm) \varrho(\hz\plm)| \leq \eta \Big]  &\geq 1-\delta, 
 \end{align*}
where $\hz\plm$, $\hy\pl$, $\cz\pl$ are defined in \Eqref{eq:additional1}, (\ref{eq:additional3}), (\ref{eq:additional5}), and where $\rho=\frac{\gamma^2}{\gamma^2 + \beta^2}<1$ and $\chi(\hz\plm)\in \R^+$ is dependent on $\Theta\plm$ but independent of $\theta\pl$ such that $\chi(\hz\plm)\leq1$ in general and $\chi(\hz\plm)=1$ if $\phi=\identity$.
\end{Lemma}

\begin{proof}\textbf{.}
First noting that $\hx\pl=\vomega\pl*\hz\plm$, we define $\hr\plm~\in~\R^{H \times W \times K\ul^2 C\ulm}$ the \say{receptive field} tensor containing at each spatial position $\alpha \in \{1, \dots, H \} \times \{1, \dots, W\}$ the $K\ul^2 C\ulm$ elements of $\hz\plm$ belonging to the receptive field of $(\hx\pl_{\alpha,1},\dots,\hx\pl_{\alpha,C\ul})^\mathrm{T}$.

For a fixed fan-in element $c'$ originating from channel $c$, the assumption of periodic boundary conditions implies that $\hr\plm_{\alpha,c'}$ has the same distribution as $\hz\plm_{\alpha,c}$ with respect to $(\rvx,\alpha)$, implying
 \begin{align*}
  \pwr_{c'}(\hr\plm) & = \pwr\uc(\hz\plm), &  
  \pwrone_{c'}(\hr\plm) & = \pwrone\uc(\hz\plm).
 \end{align*}

Since the number of fan-in elements $c'$ originating from channel $c$ is equal to $K\ul^2$ for any choice of $c$, it follows that
 \begin{align}
  \pwr(\hr\plm) &= \pwr(\hz\plm), &
  \pwrone(\hr\plm) & = \pwrone(\hz\plm). \label{eq:ln_collapse1_1}
 \end{align}

Now if we denote $\mW\pl \in \R^{C\ul \times K\ul^2 C\ulm}$ the reshaped matricial form of $\vomega\pl$, we obtain that $\forall\alpha$: $(\hx\pl_{\alpha,1},\dots,\hx\pl_{\alpha,C\ul})^\mathrm{T}=\mW\pl (\hr\plm_{\alpha,1},\dots,\hr\plm_{\alpha,K\ul^2 C\ulm})^\mathrm{T}$, implying $\forall c$:
\begin{align}
\mu\uc(\hx\pl) & = \E_{\rvx,\alpha}\Big[ \hx\pl_{\alpha,c} \Big] = \E_{\rvx,\alpha}\Bigg[ \sum_{c'} \mW\pl_{cc'} \hr\plm_{\alpha,c'} \Bigg], \label{eq:ln_collapse1_2} \\
\pwr\uc(\hx\pl) & = \E_{\rvx,\alpha}\Big[ (\hx\pl_{\alpha,c})^2 \Big] 
				  = \E_{\rvx,\alpha}\Bigg[ \Big(\sum_{c'} \mW\pl_{cc'} \hr\plm_{\alpha,c'}\Big)^2 \Bigg], \label{eq:ln_collapse1_3}\\
\pwrone\uc(\hx\pl) & = \mu\uc(\hx\pl)^2
				  = \E_{\rvx,\alpha}\Bigg[ \sum_{c'} \mW\pl_{cc'} \hr\plm_{\alpha,c'} \Bigg]^2. \label{eq:ln_collapse1_4}
\end{align}

Further expanding \Eqref{eq:ln_collapse1_2}, (\ref{eq:ln_collapse1_3}), (\ref{eq:ln_collapse1_4}), we get $\forall c$:
\begin{align}
\mu\uc(\hx\pl) & = \sum_{c'} \mW\pl_{cc'} \E_{\rvx,\alpha}\Big[\hr\plm_{\alpha,c'} \Big], \nonumber \\
\pwr\uc(\hx\pl) & = \sum_{c',c''} \mW\pl_{cc'} \mW\pl_{cc''} \E_{\rvx,\alpha}\Big[ \hr\plm_{\alpha,c'} \hr\plm_{\alpha,c''} \Big], \nonumber\\
\pwrone\uc(\hx\pl) &  = \sum_{c',c''} \mW\pl_{cc'} \mW\pl_{cc''} \E_{\rvx,\alpha}\Big[ \hr\plm_{\alpha,c'} \Big] \E_{\rvx,\alpha}\Big[ \hr\plm_{\alpha,c''} \Big], \nonumber \\
\pwr\uc(\hx\pl) - \pwrone\uc(\hx\pl) & = \sum_{c',c''} \mW\pl_{cc'} \mW\pl_{cc''} \Bigg( \E_{\rvx,\alpha}\Big[ \hr\plm_{\alpha,c'} \hr\plm_{\alpha,c''} \Big] - \E_{\rvx,\alpha}\Big[ \hr\plm_{\alpha,c'} \Big] \E_{\rvx,\alpha}\Big[ \hr\plm_{\alpha,c''} \Big]\Bigg). \nonumber
\end{align}

Since the components of $\sqrt{K\ul^2 C\ulm} \mW\pl$ are sampled i.i.d. from the fixed distribution $\nu_\vomega$ which is assumed symmetric around zero, we get $\forall c$:
\begin{align}
\E_{\theta\pl} \Big[ \mu\uc(\hx\pl) \Big] & = 0,  \label{eq:ln_collapse1_5} \\
\E_{\theta\pl} \Big[ \pwr\uc(\hx\pl) \Big] & = \sum_{c'} \E_{\theta\pl} \Big[ (\mW\pl_{cc'})^2 \Big] \E_{\rvx,\alpha}\Big[ (\hr\plm_{\alpha,c'})^2  \Big] \nonumber \\
& = \sum_{c'} \frac{\omega^2}{K\ul^2 C\ulm} \pwr_{c'}(\hr\plm) \nonumber \\
& = \omega^2 \pwr(\hz\plm),  \label{eq:ln_collapse1_6} \\
\E_{\theta\pl} \Big[ \pwr\uc(\hx\pl) - \pwrone\uc(\hx\pl)  \Big] & = \sum_{c'} \E_{\theta\pl} \Big[ (\mW\pl_{cc'})^2 \Big] \Bigg( \E_{\rvx,\alpha}\Big[ (\hr\plm_{\alpha,c'})^2 \Big] - \E_{\rvx,\alpha}\Big[ \hr\plm_{\alpha,c'} \Big]^2 \Bigg) \nonumber \\
& = \sum_{c'} \frac{\omega^2}{K\ul^2 C\ulm} \Bigg( \pwr_{c'}(\hr\plm) - \pwrone_{c'}(\hr\plm) \Bigg) \nonumber \\
& = \omega^2 \Bigg( \pwr(\hz\plm) - \pwrone(\hz\plm) \Bigg),  \label{eq:ln_collapse1_7}
\end{align}
where we recall that $\omega\equiv\E_{\theta\pl}\Big[ \big(\sqrt{K\ul^2 C\ulm} \mW\pl_{cc'}\big)^2 \Big]^{1/2}>0$ is the $\normltwo$ norm (i.e. the root mean square) of $\sqrt{K\ul^2 C\ulm} \mW\pl_{cc'}\sim \nu_{\vomega}$, and where we used \Eqref{eq:ln_collapse1_1} in \Eqref{eq:ln_collapse1_6} and \Eqref{eq:ln_collapse1_7}.

Let us now bound $\E_{\theta\pl} \Big[ \mu\uc(\hx\pl)^2 \Big]$, $\E_{\theta\pl} \Big[ \pwr\uc(\hx\pl)^2 \Big]$, with the aim of bounding $\Var_{\theta\pl} \Big[ \mu\uc(\hy\pl) \Big]$, $\Var_{\theta\pl} \Big[ \pwr\uc(\hy\pl) \Big]$ later on. We start by expanding $\mu\uc(\hx\pl)^2$ and $\pwr\uc(\hx\pl)^2$ as
\begin{align*}
\mu\uc(\hx\pl)^2 & = \pwr\uc(\hx\pl) - \Big(\pwr\uc(\hx\pl) - \pwrone\uc(\hx\pl) \Big), \\
\pwr\uc(\hx\pl)^2 & = \Bigg(\sum_{c'} (\mW\pl_{cc'})^2 \E_{\rvx,\alpha}\Big[ (\hr\plm_{\alpha,c'})^2 \Big] + 
\sum_{c'}\sum_{c'' \neq c'} \mW\pl_{cc'} \mW\pl_{cc''} \E_{\rvx,\alpha}\Big[ \hr\plm_{\alpha,c'} \hr\plm_{\alpha,c''}  \Big] \Bigg)^2 \\
 & = \Bigg(\sum_{c'} (\mW\pl_{cc'})^2 \E_{\rvx,\alpha}\Big[ (\hr\plm_{\alpha,c'})^2 \Big] + 
2 \sum_{c'}\sum_{c'' < c'} \mW\pl_{cc'} \mW\pl_{cc''} \E_{\rvx,\alpha}\Big[ \hr\plm_{\alpha,c'} \hr\plm_{\alpha,c''}  \Big] \Bigg)^2.
\end{align*}

For $\E_{\theta\pl} \Big[\mu\uc(\hx\pl)^2\Big]$, we get from \Eqref{eq:ln_collapse1_6} and \Eqref{eq:ln_collapse1_7} that 
\begin{align}
\E_{\theta\pl} \Big[\mu\uc(\hx\pl)^2\Big] 
& = \omega^2 \pwr(\hz\plm) - \omega^2 \Big( \pwr(\hz\plm) - \pwrone(\hz\plm) \Big) \nonumber \\
& = \omega^2 \pwrone(\hz\plm) \nonumber \\
& \leq \omega^2 \pwr(\hz\plm). \label{eq:ln_collapse1_8}
\end{align}

As for $\E_{\theta\pl} \Big[\pwr\uc(\hx\pl)^2\Big]$, given that only terms in $(\mW\pl_{cc'})^4$ and $(\mW\pl_{cc'})^2 (\mW\pl_{cc''})^2$ remain when taking the expectation over $\theta\pl$, we get
\begin{align}
\E_{\theta\pl} & \Big[\pwr\uc(\hx\pl)^2\Big] \nonumber \\
& = \E_{\theta\pl}\Bigg[ \Bigg(\sum_{c'} (\mW\pl_{cc'})^2 \E_{\rvx,\alpha}\Big[ (\hr\plm_{\alpha,c'})^2 \Big] \Bigg)^2  
+ 4 \sum_{c'}\sum_{c'' < c'} (\mW\pl_{cc'})^2 (\mW\pl_{cc''})^2 \E_{\rvx,\alpha}\Big[ \hr\plm_{\alpha,c'} \hr\plm_{\alpha,c''}  \Big]^2 \Bigg] \nonumber \\
& \leq \E_{\theta\pl}\Bigg[ \Bigg(\sum_{c'} (\mW\pl_{cc'})^2 \pwr_{c'}(\hr\plm) \Bigg)^2 
+ 4 \sum_{c'}\sum_{c'' < c'} (\mW\pl_{cc'})^2 (\mW\pl_{cc''})^2 \pwr_{c'}(\hr\plm) \pwr_{c''}(\hr\plm) \Bigg] \label{eq:ln_collapse1_9} \\
& \leq  \E_{\theta\pl}\Bigg[ \sum_{c'} (\mW\pl_{cc'})^4 \pwr_{c'}(\hr\plm)^2 
+ 6 \sum_{c'}\sum_{c'' < c'} (\mW\pl_{cc'})^2 (\mW\pl_{cc''})^2 \pwr_{c'}(\hr\plm) \pwr_{c''}(\hr\plm) \Bigg] \nonumber \\
& \leq \sum_{c'} \E_{\theta\pl}\Big[ (\mW\pl_{cc'})^4 \Big] \pwr_{c'}(\hr\plm)^2 
+ 3 \sum_{c'}\sum_{c'' \neq c'} \E_{\theta\pl}\Big[ (\mW\pl_{cc'})^2 (\mW\pl_{cc''})^2 \Big] \pwr_{c'}(\hr\plm) \pwr_{c''}(\hr\plm)\nonumber \\
& \leq \sum_{c'} \E_{\theta\pl}\Big[ (\mW\pl_{cc'})^4 \Big] \pwr_{c'}(\hr\plm)^2 
+ 3 \sum_{c'}\sum_{c'' \neq c'} \E_{\theta\pl}\Big[ (\mW\pl_{cc'})^4 \Big] \pwr_{c'}(\hr\plm) \pwr_{c''}(\hr\plm), \label{eq:ln_collapse1_10}
\end{align}
where \Eqref{eq:ln_collapse1_9} and \Eqref{eq:ln_collapse1_10} are obtained using Cauchy-Schwarz inequality combined with $\E_{\rvx,\alpha}\big[ (\hr\plm_{\alpha,c'})^2 \big]=\pwr_{c'}(\hr\plm)$ and $\E_{\rvx,\alpha}\big[ (\hr\plm_{\alpha,c''})^2 \big]=\pwr_{c''}(\hr\plm)$.

We may further process \Eqref{eq:ln_collapse1_10} to get
\begin{align}
 \E_{\theta\pl} \Big[\pwr\uc(\hx\pl)^2\Big] 
& \leq 3 \sum_{c'} \E_{\theta\pl}\Big[ (\mW\pl_{cc'})^4 \Big] \pwr_{c'}(\hr\plm)^2  
+ 3 \sum_{c'}\sum_{c'' \neq c'} \E_{\theta\pl}\Big[ (\mW\pl_{cc'})^4 \Big] \pwr_{c'}(\hr\plm) \pwr_{c''}(\hr\plm) \nonumber \\
& \leq 3 \sum_{c',c''} \E_{\theta\pl}\Big[ (\mW\pl_{cc'})^4 \Big] \pwr_{c'}(\hr\plm) \pwr_{c''}(\hr\plm) \nonumber \\ 
& \leq 3 \E_{\theta\pl}\Big[ (\mW\pl_{c,1})^4 \Big] \big(K\ul^2 C\ulm\big)^2 \pwr(\hr\plm)^2 \nonumber \\
& \leq 3 \E_{\theta\pl}\Big[ \big(\sqrt{K\ul^2 C\ulm} \mW\pl_{c,1}\big)^4 \Big] \pwr(\hz\plm)^2 \nonumber \\
& \leq 3 \tilde{\omega}^4 \pwr(\hz\plm)^2, \label{eq:ln_collapse1_11}
\end{align}
where $\tilde{\omega}\equiv\E_{\theta\pl}\Big[ \big(\sqrt{K\ul^2 C\ulm} \mW\pl_{c,1}\big)^4 \Big]^{1/4}\geq\omega>0$ is the $\normlfour$ norm of $\sqrt{K\ul^2 C\ulm} \mW\pl_{c,1}\sim \nu_{\vomega}$.

Now we turn to $\hy\pl$ and $\cy\pl$ defined in \Eqref{eq:additional3} and \Eqref{eq:additional4}.

Due to $\mu\uc(\hy\pl)=\frac{1}{w\sqrt{\pwr(\rvz\plm)}}\mu\uc(\hx\pl)$, $\pwr\uc(\hy\pl)=\frac{1}{w^2\pwr(\rvz\plm)}\pwr\uc(\hx\pl)$ and $\pwrone\uc(\hy\pl)=\frac{1}{w^2\pwr(\rvz\plm)}\pwrone\uc(\hx\pl)$ and due to $\pwr(\hz\plm)=\pwr(\rvz\plm)\E_\rvx \Big[ \frac{1}{\pwr_\rvx(\rvz\plm)}\pwr_\rvx(\rvz\plm) \Big]=\pwr(\rvz\plm)$, \Eqref{eq:ln_collapse1_5}, (\ref{eq:ln_collapse1_6}), (\ref{eq:ln_collapse1_7}) imply
\begin{align}
\E_{\theta\pl} \Big[ \mu\uc(\hy\pl) \Big] & = 0, \nonumber \\
\E_{\theta\pl} \Big[ \pwr\uc(\hy\pl) \Big] & = 1,\label{eq:ln_collapse1_12} \\
\E_{\theta\pl} \Big[ \pwr\uc(\hy\pl) - \pwrone\uc(\hy\pl) \Big] & = \frac{\pwr(\hz\plm)-\pwrone(\hz\plm)}{\pwr(\hz\plm)}=\varrho(\hz\plm). \label{eq:ln_collapse1_13} 
\end{align}

Using \Eqref{eq:ln_collapse1_8} and \Eqref{eq:ln_collapse1_11}, we further get
\begin{align}
\Var_{\theta\pl} \Big[\mu\uc(\hy\pl) \Big] & = \E_{\theta\pl} \Big[\mu\uc(\hy\pl)^2 \Big] \leq 1, \nonumber \\
\Var_{\theta\pl} \Big[\pwr\uc(\hy\pl) \Big] & \leq \E_{\theta\pl} \Big[\pwr\uc(\hy\pl)^2 \Big] \leq 3 \tilde{\omega}^4 \omega^{-4}, \label{eq:ln_collapse1_14} \\
\Var_{\theta\pl} \Big[\pwr\uc(\hy\pl) - \pwrone\uc(\hy\pl) \Big] & \leq \E_{\theta\pl} \Big[\big(\pwr\uc(\hy\pl) - \pwrone\uc(\hy\pl)\big)^2 \Big]
\leq \E_{\theta\pl} \Big[\pwr\uc(\hy\pl)^2 \Big] \leq 3 \tilde{\omega}^4 \omega^{-4}. \nonumber
\end{align}

The terms $\mu\uc(\hy\pl)$, $\pwr\uc(\hy\pl)$ and $\pwr\uc(\hy\pl)-\pwrone\uc(\hy\pl)$ being i.i.d. with respect to $\theta\pl$ in the different channels $c$, we get
\begin{align}
\E_{\theta\pl} \Big[ \mu(\hy\pl) \Big] & = 0, & 
\Var_{\theta\pl} \Big[\mu(\hy\pl) \Big] & \leq \frac{1}{C\ul}, \label{eq:ln_collapse1_15} \\
\E_{\theta\pl} \Big[ \pwr(\hy\pl) \Big] & = 1, & 
\Var_{\theta\pl} \Big[\pwr(\hy\pl) \Big] & \leq \frac{3 \tilde{\omega}^4 \omega^{-4}}{C\ul}, \label{eq:ln_collapse1_16} \\
\E_{\theta\pl} \Big[ \pwr(\hy\pl) -\pwrone(\hy\pl)\Big] & = \varrho(\hz\plm), & 
\Var_{\theta\pl} \Big[\pwr(\hy\pl)-\pwrone(\hy\pl) \Big] & \leq \frac{3 \tilde{\omega}^4 \omega^{-4}}{C\ul}. \label{eq:ln_collapse1_17}
\end{align}

Combining \Eqref{eq:ln_collapse1_16} and \Eqref{eq:ln_collapse1_17} with Chebyshev's inequality, we get for any $1>\eta>0$ that
\begin{align*}
\P_{\theta\pl} \Bigg[ |\pwr(\hy\pl)-1| > \frac{\eta}{1+\eta} \Bigg] & \leq \Bigg(\frac{1+\eta}{\eta}\Bigg)^2 \frac{3 \tilde{\omega}^4 \omega^{-4}}{C\ul}, \\
\P_{\theta\pl} \Bigg[ |\pwr(\hy\pl)-1| > \frac{\eta}{1-\eta} \Bigg] & \leq \Bigg(\frac{1-\eta}{\eta}\Bigg)^2 \frac{3 \tilde{\omega}^4 \omega^{-4}}{C\ul}, \\
\P_{\theta\pl} \Bigg[ \Big|\pwr(\hy\pl)-\pwrone(\hy\pl)-\varrho(\hz\plm)\Big| > \eta \Bigg] & \leq \Bigg(\frac{1}{\eta}\Bigg)^2 \frac{3 \tilde{\omega}^4 \omega^{-4}}{C\ul}.
\end{align*}

Thus, for any $1>\eta>0$ and any $\delta>0$, there exists $N_1(\eta,\delta)\in \mathbb{N}^*$ independent of $\Theta\plm$, $l$, such that if $C\ul\geq N_1(\eta,\delta)$, it holds that
\begin{align}
\P_{\theta\pl} \Bigg[ \pwr(\hy\pl)  \geq \frac{1}{1+\eta}  \Bigg]
 \geq \P_{\theta\pl} \Bigg[ |\pwr(\hy\pl) -1| \leq \frac{\eta}{1+\eta} \Bigg] 
& \geq 1-\delta,  \nonumber \\
\P_{\theta\pl} \Bigg[ \pwr(\hy\pl)  \leq \frac{1}{1-\eta}  \Bigg]
 \geq \P_{\theta\pl} \Bigg[ |\pwr(\hy\pl) -1| \leq \frac{\eta}{1-\eta} \Bigg] 
& \geq 1-\delta,  \nonumber \\
\P_{\theta\pl} \Bigg[ \big|\pwr(\hy\pl)-\pwrone(\hy\pl) -\varrho(\hz\plm)\big| \leq \eta\Bigg]
& \geq  1-\delta. \nonumber
\end{align}

Thus, if $C\ul\geq N_1(\eta,\delta)$ with $1>\eta>0$, it holds with probability greater than $1-3\delta$ with respect to $\theta\pl$ that (i) $(1+\eta) \pwr(\hy\pl)\geq1$, (ii) $(1-\eta) \pwr(\hy\pl)\leq1$, (iii) $|\pwr(\hy\pl)-\pwrone(\hy\pl) -\varrho(\hz\plm)| \leq \eta$, implying
\begin{align*}
\pwr(\hy\pl)-\pwrone(\hy\pl) & \leq \varrho(\hz\plm) + \eta \\
& \leq ( \varrho(\hz\plm) + \eta) (1+\eta) \pwr(\hy\pl) \\
& \leq ( \varrho(\hz\plm) + \varrho(\hz\plm) \eta + \eta +\eta^2 ) \pwr(\hy\pl) \\
& \leq (\varrho(\hz\plm) + 3\eta ) \pwr(\hy\pl), \\
\pwr(\hy\pl)-\pwrone(\hy\pl) &  \geq \max\Big(0,\varrho(\hz\plm) - \eta\Big) \\
& \geq \max\Big(0,\varrho(\hz\plm) - \eta\Big) (1-\eta) \pwr(\hy\pl) \\
& \geq \max\Big(0, (\varrho(\hz\plm) - \eta) (1-\eta) \pwr(\hy\pl) \Big)\\
& \geq (\varrho(\hz\plm) - \eta) (1-\eta) \pwr(\hy\pl)\\
& \geq (\varrho(\hz\plm) - \varrho(\hz\plm) \eta - \eta +\eta^2 ) \pwr(\hy\pl) \\
& \geq (\varrho(\hz\plm) - 2\eta ) \pwr(\hy\pl),
\end{align*}
where we used $\varrho(\hz\plm) \leq 1$ and $\eta^2\leq \eta$ due to $\eta< 1$ , as well as $\pwr(\hy\pl)-\pwrone(\hy\pl)\geq0$.

Given that $(1+\eta) \pwr(\hy\pl)\geq1 \implies  \pwr(\hy\pl)>0$, it follows that if $C\ul\geq N_1(\eta,\delta)$ with $1>\eta>0$, it holds with probability greater than $1-3\delta$ with respect to $\theta\pl$ that
\begin{align*}
\varrho(\hy\pl) & \leq \varrho(\hz\plm) + 3\eta, \\
\varrho(\hy\pl) & \geq \varrho(\hz\plm) - 2\eta, \\
|\varrho(\hy\pl) - \varrho(\hz\plm)| & \leq 3\eta.
\end{align*}

Now, let $N_2$ be defined independently of $\Theta\plm$, $l$ as $N_2(\eta,\delta)=N_1\Big(\min\big(\frac{\eta}{3}, \frac{1}{2}\big),\frac{\delta}{3}\Big)$, $\forall\eta>0$, $\forall\delta>0$. Then for any $\eta>0$ and any $\delta>0$, if $C\ul\geq N_2(\eta,\delta)$, it holds that
\begin{align}
\P_{\theta\ul}\Bigg[|\varrho(\hy\pl) - \varrho(\hz\plm)| \leq \eta \Bigg] & \geq \P_{\theta\ul}\Bigg[ |\varrho(\hy\pl) - \varrho(\hz\plm)| \leq 3 \min\Big(\frac{\eta}{3}, \frac{1}{2}\Big)   \Bigg] \nonumber \\
& \geq 1-3\frac{\delta}{3} \nonumber \\
& \geq  1-\delta.  \label{eq:ln_collapse1_18}
\end{align}

Let us apply a similar approach with respect to $\cz\pl$, first noting that $\forall \alpha,c$:
\begin{align}
(\cy\pl_{\alpha,c})^2 & = \Big(\vgamma\pl\uc \hy\pl_{\alpha,c} + \vbeta\pl\uc \Big)^2 = (\vgamma\pl\uc)^2 (\hy\pl_{\alpha,c})^2 + (\vbeta\pl\uc)^2 + 2\vgamma\pl\uc \vbeta\pl\uc \hy\pl_{\alpha,c}, \nonumber \\
 \pwr\uc(\cy\pl) & = (\vgamma\pl\uc)^2 \pwr\uc(\hy\pl) + (\vbeta\pl\uc)^2 + 2\vgamma\pl\uc \vbeta\pl\uc \mu\uc(\hy\pl), \nonumber \\
 \E_{\theta\pl} \Big[ \pwr\uc(\cy\pl) \Big] & =  \E_{\theta\pl} \Big[ (\vgamma\pl\uc)^2 \pwr\uc(\hy\pl) + (\vbeta\pl\uc)^2 \Big] = \gamma^2  \E_{\theta\pl} \Big[ \pwr\uc(\hy\pl) \Big] + \beta^2, \label{eq:ln_collapse1_19}  
\end{align}
where $\gamma$, $\beta$ are the $\normltwo$ norms (i.e. the root mean squares) of $\vgamma\pl\uc\sim \nu_\vgamma$ and $\vbeta\pl\uc\sim \nu_\vbeta$, and where we used the fact that $\vgamma\uc\pl$ is independent from $\pwr\uc(\hy\pl)$ with respect to $\theta\pl$, while $\vbeta\pl\uc$ is independent from $\vgamma\uc\pl \mu\uc(\hy\pl)$ with respect to $\theta\pl$, with the distribution of $\vbeta\pl\uc$ symmetric around zero.

At the same time, $\forall \alpha,c$:
\begin{align}
\Big(\cy\pl_{\alpha,c}-\mu\uc(\cy\pl) \Big)^2 & = \Big(\vgamma\pl\uc \hy\pl_{\alpha,c} + \vbeta\pl\uc - \big(\vgamma\pl\uc \mu\uc(\hy\pl) + \vbeta\pl\uc \big) \Big)^2
= (\vgamma\pl\uc)^2  \Big( \hy\pl_{\alpha,c} - \mu\uc(\hy\pl) \Big)^2, \nonumber \\
 \pwr\uc(\cy\pl)-\pwrone\uc(\cy\pl) & = (\vgamma\pl\uc)^2 \Big( \pwr\uc(\hy\pl)-\pwrone\uc(\hy\pl) \Big), \nonumber  \\
 \E_{\theta\pl} \Big[ \pwr\uc(\cy\pl)-\pwrone\uc(\cy\pl) \Big] & = \gamma^2  \E_{\theta\pl} \Big[ \pwr\uc(\hy\pl)-\pwrone\uc(\hy\pl) \Big], \label{eq:ln_collapse1_20} 
\end{align}
where we used the fact that $\vgamma\pl\uc$ is independent from $\pwr\uc(\hy\pl)-\pwrone\uc(\hy\pl)$ with respect to $\theta\pl$.

Using \Eqref{eq:ln_collapse1_12} and \Eqref{eq:ln_collapse1_13}, \Eqref{eq:ln_collapse1_19} and \Eqref{eq:ln_collapse1_20} imply that $ \forall c$:
\begin{align}
 \E_{\theta\pl} \Big[ \pwr\uc(\cy\pl) \Big] & = \gamma^2  + \beta^2, \label{eq:ln_collapse1_21} \\
\E_{\theta\pl} \Big[ \pwr\uc(\cy\pl)-\pwrone\uc(\cy\pl) \Big] & = \gamma^2 \varrho(\hz\plm). \label{eq:ln_collapse1_22}
\end{align}

Now, we may bound $\E_{\theta\pl} \Big[ \pwr\uc(\cy\pl)^2 \Big]$ using $\forall \alpha, c$:
\begin{align}
(\cy\pl_{\alpha,c})^2 = \Big(\vgamma\pl\uc \hy\pl_{\alpha,c} + \vbeta\pl\uc \Big)^2 & \leq 2 \Big( (\vgamma\pl\uc)^2 (\hy\pl_{\alpha,c})^2 + (\vbeta\pl\uc)^2 \Big), \nonumber \\
\pwr\uc(\cy\pl) & \leq 2 \Big( (\vgamma\pl\uc)^2 \pwr\uc(\hy\pl) + (\vbeta\pl\uc)^2 \Big),\nonumber  \\
\pwr\uc(\cy\pl)^2 & \leq  8 \Big( (\vgamma\pl\uc)^4 \pwr\uc(\hy\pl)^2 + (\vbeta\pl\uc)^4 \Big), \nonumber \\
\E_{\theta\pl} \Big[ \pwr\uc(\cy\pl)^2 \Big] & \leq 8 \Big( \tilde{\gamma}^4 \E_{\theta\pl} \Big[ \pwr\uc(\hy\pl)^2 \Big] + \tilde{\beta}^4 \Big), \label{eq:ln_collapse1_23}
\end{align}
where we defined $\tilde{\gamma}\equiv\E_{\theta\pl}\Big[ \big(\vgamma\pl_{c}\big)^4 \Big]^{1/4}>\gamma>0$ and $\tilde{\beta}\equiv\E_{\theta\pl}\Big[ \big(\vbeta\pl_{c}\big)^4 \Big]^{1/4}>\beta>0$ the $\normlfour$ norms of $\vgamma\pl_{c}\sim \nu_{\vgamma}$ and $\vbeta\pl_{c}\sim \nu_{\vbeta}$, and where we used twice $(a+b)^2 \leq 2(a^2 + b^2)$, $\forall a,b$.

Using \Eqref{eq:ln_collapse1_14}, we then get $\forall c$:
\begin{align}
\E_{\theta\pl} \Big[\pwr\uc(\cy\pl)^2 \Big] & \leq 24 \tilde{\gamma}^4 \tilde{\omega}^4 \omega^{-4} + 8\tilde{\beta}^4. \label{eq:ln_collapse1_24} 
\end{align}

Next, we consider $\cz\pl$. We adopt the notations $\cy\plplus$, $\cy\plminus$ for the positive and negative parts of $\cy\pl$ such that $\forall \alpha,c$:
\begin{align*}
\cy\plplus_{\alpha,c} & = \max(\cy\pl_{\alpha,c},0), &
\cy\plminus_{\alpha,c} & = \max(-\cy\pl_{\alpha,c},0). 
\end{align*}

The positive homogeneity of $\phi$ implies that $\forall \alpha,c$: 
\begin{align*}
\cz\pl_{\alpha,c}& = \phi(\cy\pl_{\alpha,c})=\phi(1)\cdot\cy\plplus_{\alpha,c} + \phi(-1)\cdot\cy\plminus_{\alpha,c}, &  \quad
(\cz\pl_{\alpha,c})^2 & = \phi(1)^2 \cdot (\cy\plplus_{\alpha,c})^2 + \phi(-1)^2 \cdot (\cy\plminus_{\alpha,c})^2.
\end{align*}

For any $c$, this implies for $\mu\uc(\cz\pl)^2$ and $\pwr\uc(\cz\pl)$ that
\begin{align}
\mu\uc(\cz\pl)^2 & = \Big( \phi(1)  \E_{\rvx,\alpha}\big[ \cy\plplus_{\alpha,c} \big] + \phi(-1) \E_{\rvx,\alpha}\big[ \cy\plminus_{\alpha,c}  \big] \Big)^2 \nonumber  \\
& = \Big(\phi(1) \mu\uc(\cy\plplus) + \phi(-1) \mu\uc(\cy\plminus)\Big)^2 \nonumber  \\
& = \phi(1)^2 \mu\uc(\cy\plplus)^2 + \phi(-1)^2 \mu\uc(\cy\plminus)^2 + 2 \phi(1)\phi(-1) \mu\uc(\cy\plplus) \mu\uc(\cy\plminus), \label{eq:ln_collapse1_25}  \\
\pwr\uc(\cz\pl) & = \E_{\rvx,\alpha}\Big[ \phi(1)^2 (\cy\plplus_{\alpha,c})^2 + \phi(-1)^2 (\cy\plminus_{\alpha,c})^2 \Big] \nonumber \\
& = \phi(1)^2 \pwr\uc(\cy\plplus) + \phi(-1)^2 \pwr\uc(\cy\plminus). \label{eq:ln_collapse1_26} 
\end{align}

Now turning to $\cy\pl$, we have $\forall \alpha,c$:
\begin{align*}
\cy\pl_{\alpha,c} & = \cy\plplus_{\alpha,c} - \cy\plminus_{\alpha,c}, &
(\cy\pl_{\alpha,c})^2 & = (\cy\plplus_{\alpha,c})^2 + (\cy\plminus_{\alpha,c})^2.
\end{align*}

For any $c$, this implies for $\mu\uc(\cy\pl)^2$ and $\pwr\uc(\cy\pl)$ that
\begin{align}
\mu\uc(\cy\pl)^2 & = \Big( \E_{\rvx,\alpha}\big[ \cy\plplus_{\alpha,c} \big] - \E_{\rvx,\alpha}\big[ \cy\plminus_{\alpha,c} \big] \Big)^2 \nonumber\\
& = \Big(\mu\uc(\cy\plplus) - \mu\uc(\cy\plminus)\Big)^2 \nonumber\\
&  = \mu\uc(\cy\plplus)^2 + \mu\uc(\cy\plminus)^2 - 2\mu\uc(\cy\plplus) \mu\uc(\cy\plminus), \label{eq:ln_collapse1_27} \\
\pwr\uc(\cy\pl) & = \E_{\rvx,\alpha}\Big[ (\cy\plplus_{\alpha,c})^2 + (\cy\plminus_{\alpha,c})^2 \Big] \nonumber \\
&  = \pwr\uc(\cy\plplus) + \pwr\uc(\cy\plminus). \label{eq:ln_collapse1_28}
\end{align}
At this point, we note that $\hy\pl$ and $-\hy\pl$ have the same distribution with respect to $\theta\pl$ by symmetry around zero of $\nu_\vomega$. From this and the symmetry around zero of $\nu_\vbeta$, we deduce that $\cy\pl$ and $-\cy\pl$ have the same distribution with respect to $\theta\pl$. In turn, this implies that $\cy\plplus$ and $\cy\plminus$ have the same distribution with respect to $\theta\pl$. 

Combined with \Eqref{eq:ln_collapse1_25} and \Eqref{eq:ln_collapse1_26}, we deduce that $\forall c$:
\begin{align*}
\E_{\theta\pl} \Big[\mu\uc(\cz\pl)^2\Big]  & = \Big(\phi(1)^2 + \phi(-1)^2\Big)  \E_{\theta\pl}\Big[\mu\uc(\cy\plplus)^2 \Big] + 2 \phi(1)\phi(-1) \E_{\theta\pl}\Big[\mu\uc(\cy\plplus) \mu\uc(\cy\plminus)\Big] \\
& = F_\phi \E_{\theta\pl}\Big[\mu\uc(\cy\plplus)^2 + \mu\uc(\cy\plminus)^2 \Big] + 2 \phi(1)\phi(-1) \E_{\theta\pl}\Big[\mu\uc(\cy\plplus) \mu\uc(\cy\plminus)\Big], \\
\E_{\theta\pl} \Big[\pwr\uc(\cz\pl)\Big] & = \Big(\phi(1)^2 + \phi(-1)^2\Big)  \E_{\theta\pl}\Big[\pwr\uc(\cy\plplus) \Big] \\
& = F_\phi \E_{\theta\pl}\Big[\pwr\uc(\cy\plplus) + \pwr\uc(\cy\plminus) \Big],
\end{align*}
where we defined $F_\phi\equiv \frac{\phi(1)^2 + \phi(-1)^2}{2}>0$, with the strict positivity of $F_\phi$ following from the assumption that $\phi$ is nonzero.

Given $|\phi(1)\phi(-1)|\leq \frac{\phi(1)^2 + \phi(-1)^2}{2} \implies \phi(1)\phi(-1) \geq -F_\phi$, and given \mbox{$\mu\uc(\cy\plplus) \mu\uc(\cy\plminus)\geq0$}, we deduce that $\forall c$:
\begin{align}
\E_{\theta\pl}\Big[\mu\uc(\cz\pl)^2\Big] 
& \geq F_\phi \E_{\theta\pl}\Big[\mu\uc(\cy\plplus)^2 + \mu\uc(\cy\plminus)^2  - 2\mu\uc(\cy\plplus) \mu\uc(\cy\plminus)\Big] \nonumber \\
& \geq F_\phi \E_{\theta\pl}\Big[\mu\uc(\cy\pl)^2 \Big] \nonumber  \\
& \geq F_\phi \E_{\theta\pl}\Big[\pwrone\uc(\cy\pl) \Big], \label{eq:ln_collapse1_29}  \\
\E_{\theta\pl}\Big[\pwr\uc(\cz\pl)\Big]  
& = F_\phi \E_{\theta\pl}\Big[\pwr\uc(\cy\pl) \Big] \nonumber  \\
& = F_\phi \big(\gamma^2  + \beta^2\big), \label{eq:ln_collapse1_30}  \\
\E_{\theta\pl}\Big[\pwr\uc(\cz\pl)-\pwrone\uc(\cz\pl)\Big]  
& = \E_{\theta\pl}\Big[\pwr\uc(\cz\pl)-\mu\uc(\cz\pl)^2\Big]  \nonumber  \\
& \leq F_\phi \E_{\theta\pl}\Big[\pwr\uc(\cy\pl)-\pwrone\uc(\cy\pl) \Big] \nonumber  \\
& \leq F_\phi \gamma^2 \varrho(\hz\plm). \label{eq:ln_collapse1_31} 
\end{align}
where we used \Eqref{eq:ln_collapse1_27}, (\ref{eq:ln_collapse1_28}) and \Eqref{eq:ln_collapse1_21}, (\ref{eq:ln_collapse1_22}).

Let us now define $\chi(\hz\plm)\in\R^+$ independently of $c$ as
\begin{align*}
\chi(\hz\plm) & \equiv \left\{
    \begin{array}{cl}
        \frac{\E_{\theta\pl}\big[\pwr\uc(\cz\pl)-\pwrone\uc(\cz\pl)\big]}{F_\phi \gamma^2 \varrho(\hz\plm)} & \mbox{if } \varrho(\hz\plm)>0, \\
        1& \mbox{otherwise.} 
    \end{array}
\right.
\end{align*}
We note that $\chi(\hz\plm)$ is independent of $\theta\pl$ and that $\chi(\hz\plm) \leq 1$ in general, and $\chi(\hz\plm)=1$ if $\phi=\identity$ since the inequalities of \Eqref{eq:ln_collapse1_29} and \Eqref{eq:ln_collapse1_31} become equalities when $\phi=\identity$.

Given this definition of $\chi(\hz\plm)$, we may rewrite $\E_{\theta\pl}\Big[\pwr\uc(\cz\pl)-\pwrone\uc(\cz\pl)\Big]$ for any $c$ as 
\begin{align}
\E_{\theta\pl}\Big[\pwr\uc(\cz\pl)-\pwrone\uc(\cz\pl)\Big] 
& = \chi(\hz\plm) F_\phi \gamma^2 \varrho(\hz\plm). \label{eq:ln_collapse1_32}
\end{align}

Now let us bound $\E_{\theta\pl} \Big[\pwr\uc(\cz\pl)^2 \Big]$ with the goal of bounding $\Var_{\theta\pl} \Big[\pwr\uc(\cz\pl) \Big]$, $\Var_{\theta\pl} \Big[\pwr\uc(\cz\pl)-\pwrone\uc(\cz\pl) \Big]$. We get from \Eqref{eq:ln_collapse1_26} that $\forall c$:
\begin{align}
\pwr\uc(\cz\pl) & \leq \Big(\phi(1)^2 + \phi(-1)^2\Big) \Big(\pwr\uc(\cy\plplus) + \pwr\uc(\cy\plminus)\Big) \nonumber \\
& \leq 2F_\phi \pwr\uc(\cy\pl), \nonumber  \\
\E_{\theta\pl} \Big[\pwr\uc(\cz\pl)^2 \Big] 
& \leq  4F^2_\phi \E_{\theta\pl} \Big[\pwr\uc(\cy\pl)^2 \Big] \nonumber \\
& \leq  4F^2_\phi \Big(24 \tilde{\gamma}^4 \tilde{\omega}^4 \omega^{-4} + 8\tilde{\beta}^4\Big), \nonumber 
\end{align}
where we used \Eqref{eq:ln_collapse1_24}.

This gives for $\Var_{\theta\pl} \Big[\pwr\uc(\cz\pl) \Big]$, $\Var_{\theta\pl} \Big[\pwr\uc(\cz\pl)-\pwrone\uc(\cz\pl) \Big]$ that
\begin{align}
\Var_{\theta\pl} \Big[\pwr\uc(\cz\pl) \Big] 
& \leq \E_{\theta\pl} \Big[\pwr\uc(\cz\pl)^2 \Big]  \nonumber \\
& \leq  4F^2_\phi \Big(24 \tilde{\gamma}^4 \tilde{\omega}^4 \omega^{-4} + 8\tilde{\beta}^4\Big), \label{eq:ln_collapse1_33} \\
\Var_{\theta\pl} \Big[\pwr\uc(\cz\pl)-\pwrone\uc(\cz\pl) \Big] 
& \leq \E_{\theta\pl} \Big[ \big( \pwr\uc(\cz\pl)-\pwrone\uc(\cz\pl) \big)^2 \Big]\nonumber  \\
& \leq \E_{\theta\pl} \Big[ \pwr\uc(\cz\pl)^2 \Big] \nonumber \\
& \leq  4F^2_\phi \Big(24 \tilde{\gamma}^4 \tilde{\omega}^4 \omega^{-4} + 8\tilde{\beta}^4\Big). \label{eq:ln_collapse1_34}
\end{align}

Using \Eqref{eq:ln_collapse1_30}, (\ref{eq:ln_collapse1_32}),(\ref{eq:ln_collapse1_33}), (\ref{eq:ln_collapse1_34}) and the fact that the terms $\pwr\uc(\cz\pl)$ and $\pwr\uc(\cz\pl)-\pwrone\uc(\cz\pl)$ are i.i.d. in the different channels $c$, we get
\begin{align}
\E_{\theta\pl}\Big[\pwr(\cz\pl)\Big]  
& = F_\phi \big(\gamma^2  + \beta^2\big), &
\Var_{\theta\pl}\Big[\pwr(\cz\pl)\Big] 
& \leq \frac{4F^2_\phi\Big(24 \tilde{\gamma}^4 \tilde{\omega}^4 \omega^{-4} + 8\tilde{\beta}^4\Big)}{C\ul}, \label{eq:ln_collapse1_35} \\
\E_{\theta\pl}\Big[\pwr(\cz\pl)-\pwrone(\cz\pl)\Big]  
& = \chi(\hz\plm) F_\phi \gamma^2 \varrho(\hz\plm), &
\Var_{\theta\pl}\Big[\pwr(\cz\pl)-\pwrone(\cz\pl)\Big] 
& \leq \frac{4F^2_\phi\Big(24 \tilde{\gamma}^4 \tilde{\omega}^4 \omega^{-4} + 8\tilde{\beta}^4\Big)}{C\ul}. \nonumber
\end{align}

Now if we define $\bz\pl$ such that $\forall \alpha,c$: $\bz\pl_{\alpha,c}=\frac{\cz\pl_{\alpha,c}}{\sqrt{F_\phi (\gamma^2  + \beta^2)}}$, we get
\begin{align*}
\E_{\theta\pl}\Big[\pwr(\bz\pl)\Big]  
& = 1, &
\Var_{\theta\pl}\Big[\pwr(\bz\pl)\Big] 
& \leq \frac{1}{C\ul} \frac{4\Big(24 \tilde{\gamma}^4 \tilde{\omega}^4 \omega^{-4} + 8\tilde{\beta}^4\Big)}{\big(\gamma^2  + \beta^2\big)^2}, \\
\E_{\theta\pl}\Big[\pwr(\bz\pl)-\pwrone(\bz\pl)\Big]  
& = \rho \chi(\hz\plm) \varrho(\hz\plm), &
\Var_{\theta\pl}\Big[\pwr(\bz\pl)-\pwrone(\bz\pl)\Big] 
& \leq \frac{1}{C\ul} \frac{4\Big(24 \tilde{\gamma}^4 \tilde{\omega}^4 \omega^{-4} + 8\tilde{\beta}^4\Big)}{\big(\gamma^2  + \beta^2\big)^2},
\end{align*}
where $\rho=\frac{\gamma^2}{\gamma^2 + \beta^2}$.

The reasoning that yielded \Eqref{eq:ln_collapse1_18} from \Eqref{eq:ln_collapse1_16} and \Eqref{eq:ln_collapse1_17} can be immediately transposed by replacing $\hy\pl$ by $\bz\pl$, $\varrho(\hz\plm)$ by $\rho\chi(\hz\plm)\varrho(\hz\plm)$ and $3 \tilde{\omega}^4 \omega^{-4}$ by $\frac{4\big(24 \tilde{\gamma}^4 \tilde{\omega}^4 \omega^{-4} + 8\tilde{\beta}^4\big)}{\big(\gamma^2  + \beta^2\big)^2}$.

Consequently, for any $\eta>0$ and any $\delta>0$, there exists $N_3(\eta,\delta)\in \mathbb{N}^*$ independent of $\Theta\plm$, $l$, such that if $C\ul\geq N_3(\eta,\delta)$, it holds that
\begin{align*}
\P_{\theta\ul}\Bigg[|\varrho(\bz\pl) - \rho\chi(\hz\plm)\varrho(\hz\plm)| \leq \eta \Bigg] \geq 1-\delta, \\
\P_{\theta\ul}\Bigg[|\varrho(\cz\pl) - \rho\chi(\hz\plm)\varrho(\hz\plm)| \leq \eta \Bigg] \geq 1-\delta,
\end{align*}
where we used the fact that $\varrho(\cz\pl)=\varrho(\bz\pl)$.

Let us finally define $N'$ independently of $\Theta\pl$, $l$ as $N'(\eta,\delta)=\max\big(N_2(\eta,\delta),N_3(\eta,\delta)\big)$, $\forall\eta>0$, $\forall\delta>0$. Then for any $\eta>0$ and any $\delta>0$, if $C\ul\geq N'(\eta,\delta)$, it holds that
 \begin{align*}
\P_{\theta\ul}\Big[ |\varrho(\hy\pl) - \varrho(\hz\plm)| \leq \eta   \Big]  &\geq 1-\delta,  \\
\P_{\theta\ul}\Big[ |\varrho(\cz\pl) - \rho \chi(\hz\plm) \varrho(\hz\plm)| \leq \eta \Big]  &\geq 1-\delta,
 \end{align*}
where we recall that $\chi(\hz\plm) \leq 1$ in general, and that $\chi(\hz\plm)=1$ if $\phi=\identity$.  \qed
\end{proof}

\begin{Lemma}
\label{lemma:ln_collapse2}
Fix a layer $l\geq1$, $\nu_{\vomega}$, $\nu_{\vbeta}$, $\nu_{\vgamma}$, $\D$ in \Defref{def:random} and model parameters $\Theta\plm$ up to layer $l-1$ such that $\pwr_\rvx(\rvz\plm)>0$, $\forall \rvx$. Further suppose $\Norm=\LN$ and suppose and that the convolution of \Eqref{eq:propagation1} uses periodic boundary conditions.

Then for any $\eta>0$ and any $\delta>0$, there exists $N''(\eta,\delta)\in \mathbb{N}^*$ independent of $\Theta\plm$, $l$ such that if $C\ul \geq N''(\eta,\delta)$, it holds for random nets of \Defref{def:random} with probability greater than $1-\delta$ with respect to $\theta\pl$ that
 \begin{align*}
|\pwrone(\hy\pl)-\pwrone(\rvy\pl)| & \leq \eta, &
|\pwrone(\cz\pl)-\pwrone(\rvz\pl)| & \leq \eta,  \\
|\pwr_\rvx(\hy\pl)-\pwr_\rvx(\rvy\pl)| & \leq \eta, \quad \forall \rvx\in \D, &
|\pwr_\rvx(\cz\pl)-\pwr_\rvx(\rvz\pl)| & \leq \eta, \quad \forall \rvx\in \D, \\
|\pwr_\rvx(\hy\pl)-1| & \leq \eta, \quad \forall \rvx\in \D, &
\big|\pwr_\rvx(\cz\pl)-F_\phi \big(\gamma^2  + \beta^2\big)\big| & \leq \eta, \quad \forall \rvx\in \D,
 \end{align*}
where $F_\phi \equiv \frac{\phi(1)^2 + \phi(-1)^2}{2}>0$ and $\hy\pl$, $\cz\pl$ are defined in \Eqref{eq:additional3} and \Eqref{eq:additional5}.
\end{Lemma}

\begin{proof}\textbf{.} Let us start by noting that $\forall \alpha,c$:
\begin{align*}
\hx\pl_{\alpha,c} & = \sqrt{\frac{\pwr(\rvz\plm)}{\pwr_\rvx(\rvz\plm)}} \rvx\pl_{\alpha,c}, &
\hy\pl_{\alpha,c} & = \frac{1}{\omega\sqrt{\pwr(\rvz\plm)}} \hx\pl_{\alpha,c} = \frac{1}{\omega\sqrt{\pwr_\rvx(\rvz\plm)}} \rvx\pl_{\alpha,c}.
\end{align*}
This implies that $\hy\pl$ only depends on $\rvx$ and not on other inputs in the dataset.

Thus, \Eqref{eq:ln_collapse1_15}, (\ref{eq:ln_collapse1_16}), (\ref{eq:ln_collapse1_35}) still hold when considering the moments conditioned on $\rvx$, such that $\forall \rvx$:
\begin{align}
\E_{\theta\pl} \Big[ \mu_\rvx(\hy\pl) \Big] & = 0, & 
\Var_{\theta\pl} \Big[\mu_\rvx(\hy\pl) \Big] & \leq \frac{1}{C\ul}, \label{eq:ln_collapse2_1} \\
\E_{\theta\pl} \Big[ \pwr_\rvx(\hy\pl) \Big] & = 1, & 
\Var_{\theta\pl} \Big[\pwr_\rvx(\hy\pl) \Big] & \leq \frac{3 \tilde{\omega}^4 \omega^{-4}}{C\ul}, \label{eq:ln_collapse2_2} \\
\E_{\theta\pl}\Big[\pwr_\rvx(\cz\pl)\Big]  
& = F_\phi \big(\gamma^2  + \beta^2\big), &
\Var_{\theta\pl}\Big[\pwr_\rvx(\cz\pl)\Big] 
& \leq \frac{4F^2_\phi\Big(24 \tilde{\gamma}^4 \tilde{\omega}^4 \omega^{-4} + 8\tilde{\beta}^4\Big)}{C\ul}. \label{eq:ln_collapse2_3}
\end{align}

Combining \Eqref{eq:ln_collapse2_1}, (\ref{eq:ln_collapse2_2}), (\ref{eq:ln_collapse2_3}) with Chebyshev's inequality, we get for any $\eta>0$ and any $\delta>0$ that there exists $N_4(\eta,\delta)\in \mathbb{N}^*$ independent of $\Theta\plm$, $l$ such that, if $C\ul\geq N_4(\eta,\delta)$, it holds for any $\rvx$ that
\begin{align}
\P_{\theta\pl}\Big[ |\mu_\rvx(\hy\pl)| \leq \eta \Big] & \geq 1-\delta, \\ 
\P_{\theta\pl}\Big[ |\pwr_\rvx(\hy\pl)-1| \leq \eta \Big] & \geq 1-\delta, \\
\P_{\theta\pl}\Bigg[ \big|\pwr_\rvx(\cz\pl)-F_\phi(\gamma^2 + \beta^2)\big| \leq \eta \Bigg] & \geq 1-\delta.
\end{align}

Next we turn to $\pwr_\rvx(\hy\pl - \rvy\pl)$, $\pwr_\rvx(\cz\pl - \rvz\pl)$. Given that $\rvy\pl_{\alpha,c}=\frac{\rvx\pl_{\alpha,c}-\mu_\rvx(\rvx\pl)}{\sigma_\rvx(\rvx\pl)}=\frac{\hy\pl_{\alpha,c}-\mu_\rvx(\hy\pl)}{\sigma_\rvx(\hy\pl)}$, $\forall\rvx,\alpha,c$, we deduce $\forall \rvx,\alpha, c$:
\begin{align*}
\hy\pl_{\alpha,c} & = \sigma_\rvx(\hy\pl) \rvy\pl_{\alpha,c} + \mu_\rvx(\hy\pl) \\
& = \rvy\pl_{\alpha,c} + (\sigma_\rvx(\hy\pl) -1 ) \rvy\pl_{\alpha,c} + \mu_\rvx(\hy\pl), \\
\cy\pl_{\alpha,c}  & = \vgamma\pl\uc \rvy\pl_{\alpha,c} + \vbeta\pl\uc + \vgamma\pl\uc ( \sigma_\rvx(\hy\pl) -1 ) \rvy\pl_{\alpha,c} + \vgamma\pl\uc \mu_\rvx(\hy\pl) \\
& = \ty\pl_{\alpha,c} + \vgamma\pl\uc ( \sigma_\rvx(\hy\pl) -1 ) \rvy\pl_{\alpha,c} + \vgamma\pl\uc \mu_\rvx(\hy\pl).
\end{align*}

Now let us fix $\rvx$ and bound $\pwr_\rvx(\hy\pl - \rvy\pl)$ and $\pwr_\rvx(\cz\pl - \rvz\pl)$. We start by noting that
\begin{align*}
(\hy\pl_{\alpha,c} - \rvy\pl_{\alpha,c})^2
& \leq 2 (\sigma_\rvx(\hy\pl) -1 )^2 (\rvy\pl_{\alpha,c})^2 + 2 \mu_\rvx(\hy\pl)^2, \\
(\cy\pl_{\alpha,c} - \ty\pl_{\alpha,c})^2
& \leq 2 (\sigma_\rvx(\hy\pl) -1)^2 (\vgamma\pl\uc \rvy\pl_{\alpha,c})^2 + 2 \mu_\rvx(\hy\pl)^2 (\vgamma\pl\uc)^2, \\
(\cz\pl_{\alpha,c} - \rvz\pl_{\alpha,c})^2
& = \big(\phi(\cy\pl_{\alpha,c}) - \phi(\ty\pl_{\alpha,c})\big)^2 \\
& \leq 2F_\phi (\cy\pl_{\alpha,c} - \ty\pl_{\alpha,c})^2 \\
& \leq 4F_\phi (\sigma_\rvx(\hy\pl) -1)^2 (\vgamma\pl\uc \rvy\pl_{\alpha,c})^2 + 4F_\phi \mu_\rvx(\hy\pl)^2 (\vgamma\pl\uc)^2,
\end{align*}
where we used $(a+b)^2\leq 2a^2+2b^2$, $\forall a,b$ and $\big(\phi(a)-\phi(b)\big)^2\leq (F'_\phi)^2 (a-b)^2 \leq 2F_\phi (a-b)^2$, $\forall a,b$, by $F'_\phi$-Lipschitzness of $\phi$ with $F'_\phi=\max\big(|\phi(1)|,|\phi(-1)|\big) \leq \sqrt{2F_\phi}$.

We deduce for $\pwr_\rvx(\hy\pl - \rvy\pl)$ and $\pwr_\rvx(\cz\pl - \rvz\pl)$ that
\begin{align}
\pwr_\rvx(\hy\pl - \rvy\pl)
& \leq 2 ( \sigma_\rvx(\hy\pl) -1 )^2 \pwr_\rvx(\rvy\pl) + 2 \mu_\rvx(\hy\pl)^2 \nonumber \\
& \leq 2\Big( (\sigma_\rvx(\hy\pl) -1)^2  + \mu_\rvx(\hy\pl)^2 \Big),  \label{eq:ln_collapse2_4}  \\
\pwr_\rvx(\cz\pl - \rvz\pl) 
& \leq 4F_\phi (\sigma_\rvx(\hy\pl) -1 )^2 \E\uc\Big[ (\vgamma\pl\uc)^2 \pwr_{\rvx,c}(\rvy\pl) \Big] + 4F_\phi \mu_\rvx(\hy\pl)^2 \E\uc\Big[ (\vgamma\pl\uc)^2\Big] \nonumber \\
& \leq 4F_\phi \Big( (\sigma_\rvx(\hy\pl) -1 )^2 + \mu_\rvx(\hy\pl)^2 \Big) \E\uc\Big[ (\vgamma\pl\uc)^2 \big( \pwr_{\rvx,c}(\rvy\pl) +1 \big) \Big], \label{eq:ln_collapse2_5} 
\end{align}
where we used $\pwr_\rvx(\rvy\pl)\leq 1$.

Next, let us bound the expectation over $\theta\pl$ of $\E\uc\Big[ (\vgamma\pl\uc)^2 \big(\pwr_{\rvx,c}(\rvy\pl)+1\big) \Big]$:
\begin{align*}
\E_{\theta\pl} \Big[ \E\uc\Big[ (\vgamma\pl\uc)^2 \big(\pwr_{\rvx,c}(\rvy\pl) + 1\big) \Big] \Big] 
& = \E\uc \Big[ \E_{\theta\pl} \Big[ (\vgamma\pl\uc)^2 \big(\pwr_{\rvx,c}(\rvy\pl)+ 1\big) \Big] \Big] \nonumber \\
& = \E\uc \Big[ \E_{\theta\pl} \Big[ (\vgamma\pl\uc)^2 \Big] \E_{\theta\pl} \Big[ \pwr_{\rvx,c}(\rvy\pl) +1 \Big] \Big] \nonumber  \\
& = \gamma^2 \E_{\theta\pl} \Big[ \E\uc \Big[ \pwr_{\rvx,c}(\rvy\pl) + 1 \Big] \Big] \nonumber \\
& = \gamma^2 \E_{\theta\pl} \Big[ \pwr_\rvx(\rvy\pl) +1 \Big] \\
& \leq 2\gamma^2,
\end{align*}
where we used the independence of $\vgamma\pl\uc$ and $\pwr_{\rvx,c}(\rvy\pl)$ \mbox{with respect to $\theta\pl$ for any $c$, and again $\pwr_\rvx(\rvy\pl)\leq 1$.}

Markov's inequality then gives for any $\delta>0$ that
\begin{align}
\P_{\theta\pl} \Bigg[ \E\uc\Big[ (\vgamma\pl\uc)^2 \big(\pwr_{\rvx,c}(\rvy\pl)+1\big) \Big] \geq \frac{2\gamma^2}{\delta} \Bigg] & \leq 2\gamma^2 \frac{\delta}{2\gamma^2} = \delta. \nonumber
\end{align}

Thus, for any $1\geq\eta>0$ and any $\delta>0$, if $C\ul\geq N_4(\eta,\delta)$, it holds for any $\rvx$ with probability greater than $1-4\delta$ with respect to $\theta\pl$ that
\begin{align}
|\mu_\rvx(\hy\pl)| & \leq \eta, \label{eq:ln_collapse2_6}  \\ 
|\pwr_\rvx(\hy\pl)-1| & \leq \eta, \label{eq:ln_collapse2_7}  \\
\big|\pwr_\rvx(\cz\pl)-F_\phi(\gamma^2 + \beta^2)\big| & \leq \eta, \label{eq:ln_collapse2_8}  \\
\E\uc\Big[ (\vgamma\pl\uc)^2 \big(\pwr_{\rvx,c}(\rvy\pl)+1\big) \Big] & \leq \frac{2\gamma^2}{\delta}. \label{eq:ln_collapse2_9}
\end{align}

If both inequalities of \Eqref{eq:ln_collapse2_6} and \Eqref{eq:ln_collapse2_7} hold with $1\geq\eta>0$, then
\begin{align*}
|\sigma_\rvx(\hy\pl)-1| 
& \leq |\sigma_\rvx(\hy\pl)-1| |\sigma_\rvx(\hy\pl)+1| = |\sigma_\rvx(\hy\pl)^2-1| \\
& \leq |\pwr_\rvx(\hy\pl)-1| + \mu_\rvx(\hy\pl)^2 \\
& \leq \eta + \eta^2 \leq 2\eta, \\
(\sigma_\rvx(\hy\pl)-1)^2 
& \leq 4\eta^2 \leq 4\eta, \\
\mu_\rvx(\hy\pl)^2 & \leq \eta^2 \leq \eta,
\end{align*}
where we used $\eta^2 \leq \eta$ for $1\geq\eta>0$.

Injecting this into \Eqref{eq:ln_collapse2_4} and \Eqref{eq:ln_collapse2_5}, we get that, if $C\ul\geq N_4(\eta,\delta)$ with $1\geq\eta>0$, it holds for any $\rvx$ with probability greater than $1-4\delta$ with respect to $\theta\pl$ that
\begin{align*}
|\pwr_\rvx(\hy\pl)-1| & \leq \eta, \\
\pwr_\rvx(\hy\pl - \rvy\pl) & \leq 10 \eta,  \\
\big|\pwr_\rvx(\cz\pl)-F_\phi(\gamma^2 + \beta^2)\big| & \leq \eta, \\
\pwr_\rvx(\cz\pl - \rvz\pl) & \leq 20F_\phi\eta \frac{2\gamma^2}{\delta}.
\end{align*}

Let us define $N_5$ independently of $\Theta\plm$, $l$ as $N_5(\eta,\delta)=N_4\Big( \min(\frac{\eta}{10}, \frac{\eta}{20F_\phi} \frac{\delta}{2\gamma^2}, 1),\frac{\delta}{4|\D|}\Big)$, $\forall\eta>0$, $\forall\delta>0$. Then $\forall\eta>0$, $\forall\delta>0$ if $C\ul\geq N_5(\eta,\delta)$, it holds with probability greater than $1-\delta$ with respect to $\theta\pl$ that
\begin{align}
|\pwr_\rvx(\hy\pl)-1| & \leq \eta, \quad \forall \rvx \in \D, &
 \big|\pwr_\rvx(\cz\pl)-F_\phi(\gamma^2 + \beta^2)\big| & \leq \eta, \quad \forall \rvx \in \D, \label{eq:ln_collapse2_10} \\
\pwr_\rvx(\hy\pl - \rvy\pl) & \leq \eta, \quad \forall \rvx \in \D,  &
\pwr_\rvx(\cz\pl - \rvz\pl) & \leq \eta, \quad \forall \rvx \in \D. \label{eq:ln_collapse2_11}
\end{align}

If \Eqref{eq:ln_collapse2_10} and \Eqref{eq:ln_collapse2_11} hold, then
\begin{align*} 
\big| \pwrone(\hy\pl)-\pwrone(\rvy\pl) \big|
& = \big| \pwrone(\hy\pl)-\pwrone(\hy\pl+\rvy\pl-\hy\pl) \big| \\
& \leq \pwrone(\rvy\pl-\hy\pl) + 2 \Big| \E\uc\Big[ \mu_{c}(\hy\pl)\mu_{c}(\rvy\pl-\hy\pl) \Big] \Big| \\
& \leq \pwrone(\rvy\pl-\hy\pl) + 2 \E\uc\Big[ |\mu_{c}(\hy\pl)| |\mu_{c}(\rvy\pl-\hy\pl)| \Big] \\
& \leq \pwrone(\rvy\pl-\hy\pl) + 2 \sqrt{\E\uc\Big[ \mu_{c}(\hy\pl)^2 \Big]  \E\uc\Big[ \mu_{c}(\rvy\pl-\hy\pl)^2 \Big]} \\
& \leq \pwr(\rvy\pl-\hy\pl) + 2 \sqrt{\pwr(\hy\pl) \pwr(\rvy\pl-\hy\pl)} \\
& \leq \E_\rvx\Big[\pwr_\rvx(\rvy\pl-\hy\pl)\Big] + 2 \sqrt{\E_\rvx\Big[\pwr_\rvx(\hy\pl)\Big] \E_\rvx\Big[\pwr_\rvx(\rvy\pl-\hy\pl)\Big]} \\
& \leq \eta + 2 \sqrt{(1+\eta)\eta}, 
\end{align*}
where we used $\pwrone(\rvy\pl-\hy\pl)\leq \pwr(\rvy\pl-\hy\pl)$ and $\pwrone(\hy\pl)\leq \pwr(\hy\pl)$, as well as Jensen's inequality and Cauchy-Schwartz inequality.

Similarly, if \Eqref{eq:ln_collapse2_10} and \Eqref{eq:ln_collapse2_11} hold, then $\forall \rvx$:
\begin{align*} 
\big|\pwr_\rvx(\hy\pl) - \pwr_\rvx(\rvy\pl)\big| 
& = \big| \pwr_\rvx(\hy\pl)-\pwr_\rvx(\hy\pl+\rvy\pl-\hy\pl) \big| \\
& \leq \pwr_\rvx(\rvy\pl-\hy\pl) + 2 \mu_\rvx(|\hy\pl| |\rvy\pl-\hy\pl|) \\
& \leq \pwr_\rvx(\rvy\pl-\hy\pl) + 2 \sqrt{\pwr_\rvx(\hy\pl) \pwr_\rvx(\rvy\pl-\hy\pl)} \\
& \leq \eta + 2 \sqrt{(1+\eta)\eta}.
\end{align*}

A similar calculation with $\cz\pl$, $\rvz\pl$ shows that if \Eqref{eq:ln_collapse2_10} and \Eqref{eq:ln_collapse2_11} hold, then
\begin{align*} 
\big| \pwrone(\cz\pl)-\pwrone(\rvz\pl) \big| 
& \leq \eta + 2 \sqrt{\eta}\sqrt{F_\phi(\gamma^2+\beta^2)+\eta}, \\ 
\big|\pwr_\rvx(\cz\pl) - \pwr_\rvx(\rvz\pl)\big| 
& \leq \eta + 2 \sqrt{\eta}\sqrt{F_\phi(\gamma^2+\beta^2)+\eta}, \quad \forall \rvx.
\end{align*}

Given that the three terms $\eta$, $\eta + 2 \sqrt{(1+\eta)\eta}$ and $\eta + 2 \sqrt{\eta}\sqrt{F_\phi(\gamma^2+\beta^2)+\eta}$ converge to $0$ as $\eta\to0$, it follows that there exists a mapping $h$ such that for any $\eta>0$: $h(\eta)>0$ and
\begin{align*} 
 h(\eta) & \leq \eta, \\
 h(\eta) + 2 \sqrt{(1+h(\eta))h(\eta)} &  \leq \eta, &\\
 h(\eta) + 2 \sqrt{h(\eta)}\sqrt{F_\phi(\gamma^2+\beta^2)+h(\eta)} & \leq \eta.
\end{align*}

Let us finally define $N''$ independently of $\Theta\pl$, $l$ as $N''(\eta,\delta)=N_5(h(\eta),\delta)$, $\forall\eta>0$, $\forall\delta>0$. Then $\forall\eta>0$, $\forall\delta>0$, if $C\ul\geq N''(\eta,\delta)$, it holds with probability greater than $1-\delta$ with respect to $\theta\pl$ that 
 \begin{align*}
|\pwrone(\hy\pl)-\pwrone(\rvy\pl)| & \leq \eta,  &
|\pwrone(\cz\pl)-\pwrone(\rvz\pl)| & \leq \eta,  \\
|\pwr_\rvx(\hy\pl)-\pwr_\rvx(\rvy\pl)| & \leq \eta, \quad \forall \rvx\in \D, &
|\pwr_\rvx(\cz\pl)-\pwr_\rvx(\rvz\pl)| & \leq \eta, \quad \forall \rvx\in \D, \\
|\pwr_\rvx(\hy\pl)-1| & \leq \eta, \quad \forall \rvx\in \D, &
\big|\pwr_\rvx(\cz\pl)-F_\phi (\gamma^2+\beta^2)\big| & \leq \eta, \quad \forall \rvx\in \D. \tag*{\qed}
 \end{align*}

\end{proof}

\begin{Lemma}
\label{lemma:ln_collapse3}
Fix a layer $l\geq1$, $\nu_{\vomega}$, $\nu_{\vbeta}$, $\nu_{\vgamma}$, $\D$ in \Defref{def:random} and model parameters $\Theta\plm$ up to layer $l-1$ such that $\pwr_\rvx(\rvz\plm)>0$, $\forall \rvx$. Further suppose $\Norm=\LN$ and suppose that the convolution of \Eqref{eq:propagation1} uses periodic boundary conditions.

Then for any $\eta>0$ and any $\delta>0$, there exists $N'''(\eta,\delta)\in \mathbb{N}^*$ independent of $\Theta\plm$, $l$ such that if $C\ul \geq N'''(\eta,\delta)$, it holds for random nets of \Defref{def:random} with probability greater than $1-\delta$ with respect to $\theta\pl$ that
 \begin{align*}
|\varrho(\rvy\pl)-\varrho(\hz\plm)| & \leq \eta, & 
|\varrho(\rvz\pl) - \rho\chi(\hz\plm) \varrho(\hz\plm)| & \leq \eta, \\
|\pwr_\rvx(\rvy\pl)-1| & \leq \eta,\quad \forall \rvx \in \D, &
\big|\pwr_\rvx(\rvz\pl)-F_\phi(\gamma^2+\beta^2)\big| & \leq \eta,\quad \forall \rvx \in \D,
 \end{align*}
where $F_\phi \equiv \frac{\phi(1)^2 + \phi(-1)^2}{2}>0$, $\rho=\frac{\gamma^2}{\gamma^2 + \beta^2}<1$ and $\chi(\hz\plm)\in \R^+$ is dependent on $\Theta\plm$ but independent of $\theta\pl$ such that $\chi(\hz\plm)\leq1$ in general and $\chi(\hz\plm)=1$ if $\phi=\identity$.
\end{Lemma}

\begin{proof}\textbf{.} 
First let us note that
\begin{align*}
 |\pwr_\rvx(\hy\pl)-\pwr_\rvx(\rvy\pl)| & \leq \eta, \quad \forall \rvx\in \D 
 && \implies & |\pwr(\hy\pl)-\pwr(\rvy\pl)| & \leq \eta, \\
|\pwr_\rvx(\hy\pl)-1| & \leq \eta, \quad \forall \rvx\in \D
 && \implies & |\pwr(\hy\pl)-1| & \leq \eta, \\
|\pwr_\rvx(\cz\pl)-\pwr_\rvx(\rvz\pl)| & \leq \eta, \quad \forall \rvx\in \D 
 && \implies & |\pwr(\cz\pl)-\pwr(\rvz\pl)| & \leq \eta, \\
\big|\pwr_\rvx(\cz\pl)-F_\phi (\gamma^2+\beta^2)\big| & \leq \eta, \quad \forall \rvx\in \D
 && \implies & |\pwr(\cz\pl)-F_\phi (\gamma^2+\beta^2)\big| & \leq \eta.
 \end{align*}

Combined with \Lemmaref{lemma:ln_collapse2}, we deduce for any $\eta>0$ and any $\delta>0$ that there exists $N''(\eta,\delta)\in \mathbb{N}^*$ independent of $\Theta\plm$, $l$ such that if $C\ul \geq N''(\eta,\delta)$, it holds with probability greater than $1-\delta$ with respect to $\theta\pl$ that
 \begin{align}
|\pwrone(\hy\pl)-\pwrone(\rvy\pl)| & \leq \eta, &
|\pwrone(\cz\pl)-\pwrone(\rvz\pl)| & \leq \eta, \label{eq:ln_collapse3_1} \\
|\pwr(\hy\pl)-\pwr(\rvy\pl)| & \leq \eta, &
|\pwr(\cz\pl)-\pwr(\rvz\pl)| & \leq \eta, \label{eq:ln_collapse3_2}  \\
|\pwr(\hy\pl)-1| & \leq \eta, &
\big|\pwr(\cz\pl)-F_\phi(\gamma^2+\beta^2)\big| & \leq \eta,  \label{eq:ln_collapse3_3} 
 \end{align}
where $F_\phi \equiv \frac{\phi(1)^2 + \phi(-1)^2}{2}>0$.

If all inequalities of \Eqref{eq:ln_collapse3_1}, (\ref{eq:ln_collapse3_2}), (\ref{eq:ln_collapse3_3}) hold with $\eta\leq\frac{1}{4}F_\phi(\gamma^2+\beta^2)$, then $\varrho(\rvz\pl)-\varrho(\cz\pl)$ may be upper bounded using
\begin{align*}
\pwr(\rvz\pl) -\pwrone(\rvz\pl)
&  \leq \pwr(\cz\pl) -\pwrone(\cz\pl) + |\pwr(\rvz\pl) - \pwr(\cz\pl)| + |\pwrone(\cz\pl)-\pwrone(\rvz\pl)| \nonumber \\
& \leq \varrho(\cz\pl) \pwr(\cz\pl) + 2\eta   \nonumber \\
& \leq \varrho(\cz\pl) \big(F_\phi(\gamma^2+\beta^2)+\eta\big) + 2\eta \nonumber \\
& \leq \varrho(\cz\pl) F_\phi(\gamma^2+\beta^2) + 3\eta, \nonumber\\
\pwr(\rvz\pl)
&\geq  \pwr(\cz\pl)-\eta \nonumber\\
& \geq  F_\phi(\gamma^2+\beta^2) - 2\eta, \nonumber\\
\varrho(\rvz\pl) 
& \leq \frac{\varrho(\cz\pl)F_\phi(\gamma^2+\beta^2) + 3\eta}{F_\phi(\gamma^2+\beta^2) - 2\eta} \\
& \leq \frac{\varrho(\cz\pl) + \frac{3\eta}{F_\phi(\gamma^2+\beta^2)}}{1 - \frac{2\eta}{F_\phi(\gamma^2+\beta^2)}}  \\
& \leq \Big(\varrho(\cz\pl) + \frac{3\eta}{F_\phi(\gamma^2+\beta^2)}\Big)\Big(1 + \frac{8\eta}{F_\phi(\gamma^2+\beta^2)}\Big) \\
& \leq \varrho(\cz\pl) + \frac{35\eta}{F_\phi(\gamma^2+\beta^2)},
\end{align*}
where we used $\varrho(\cz\pl)\leq 1$, as well as $\frac{1}{1-x}\leq 1+4x$ for $x\leq\frac{1}{2}$ and $\Big(\frac{\eta}{F_\phi(\gamma^2+\beta^2)}\Big)^2\leq \frac{\eta}{F_\phi(\gamma^2+\beta^2)}$ for $\frac{\eta}{F_\phi(\gamma^2+\beta^2)}\leq 1$.

Similarly, if all inequalities of \Eqref{eq:ln_collapse3_1}, (\ref{eq:ln_collapse3_2}), (\ref{eq:ln_collapse3_3}) hold with $\eta\leq\frac{1}{4}F_\phi(\gamma^2+\beta^2)$, then $\varrho(\rvz\pl)-\varrho(\cz\pl)$ may be lower bounded using
\begin{align*}
\pwr(\rvz\pl) -\pwrone(\rvz\pl)
&  \geq \pwr(\cz\pl) -\pwrone(\cz\pl) - |\pwr(\rvz\pl) - \pwr(\cz\pl)| - |\pwrone(\cz\pl)-\pwrone(\rvz\pl)| \nonumber \\
& \geq \varrho(\cz\pl) \pwr(\cz\pl) - 2\eta   \nonumber \\
& \geq \varrho(\cz\pl) \big(F_\phi(\gamma^2+\beta^2)-\eta\big) - 2\eta \nonumber \\
& \geq \varrho(\cz\pl)F_\phi(\gamma^2+\beta^2) - 3\eta, \nonumber\\
\pwr(\rvz\pl)
&\leq  \pwr(\cz\pl)+\eta \leq  F_\phi(\gamma^2+\beta^2)+2\eta, \nonumber\\
\varrho(\rvz\pl) 
& \geq \frac{\max\Big(\varrho(\cz\pl)F_\phi(\gamma^2+\beta^2) - 3\eta,0\Big)}{F_\phi(\gamma^2+\beta^2) + 2\eta}  \\
& \geq \frac{\max\Big(\varrho(\cz\pl) - \frac{3\eta}{F_\phi(\gamma^2+\beta^2)},0\Big)}{1+\frac{2\eta}{F_\phi(\gamma^2+\beta^2)}} \\
& \geq \max\Bigg(\Big(\varrho(\cz\pl) - \frac{3\eta}{F_\phi(\gamma^2+\beta^2)}\Big) \Big(1-\frac{2\eta}{F_\phi(\gamma^2+\beta^2)}\Big),0\Bigg) \\
& \geq \varrho(\cz\pl) - \frac{5\eta}{F_\phi(\gamma^2+\beta^2)},
\end{align*}
where we used $\varrho(\cz\pl)\leq 1$, as well as $\frac{1}{1+x}\geq 1-x \geq0$ for $0 \leq x\leq1$.

We deduce that if all inequalities of \Eqref{eq:ln_collapse3_1}, (\ref{eq:ln_collapse3_2}), (\ref{eq:ln_collapse3_3}) hold with $\eta\leq\frac{1}{4}F_\phi(\gamma^2+\beta^2)$, then
\begin{align}
|\varrho(\rvz\pl)-\varrho(\cz\pl)| \leq  \frac{35\eta}{F_\phi(\gamma^2+\beta^2)}. \label{eq:ln_collapse3_4} 
\end{align}

The reasoning that yielded \Eqref{eq:ln_collapse3_4} from \Eqref{eq:ln_collapse3_1}, (\ref{eq:ln_collapse3_2}), (\ref{eq:ln_collapse3_3}) can be immediately transposed by replacing $\rvz\pl$ by $\rvy\pl$, $\cz\pl$ by $\hy\pl$ and $F_\phi(\gamma^2+\beta^2)$ by $1$.

Consequently, if all inequalities of \Eqref{eq:ln_collapse3_1}, (\ref{eq:ln_collapse3_2}), (\ref{eq:ln_collapse3_3}) hold with $\eta\leq\frac{1}{4}$, then
\begin{align*}
|\varrho(\rvy\pl)-\varrho(\hy\pl)| \leq  35\eta. 
\end{align*}

\Lemmaref{lemma:ln_collapse1} also tells us that for any $\eta>0$ and any $\delta>0$, there exists $N'(\eta,\delta)\in \mathbb{N}^*$ independent of $\Theta\plm$, $l$ such that if $C\ul \geq N'(\eta,\delta)$, it holds with probability greater than $1-2\delta$ with respect to $\theta\pl$ that
 \begin{align*}
|\varrho(\hy\pl)-\varrho(\hz\plm)| & \leq \eta, &
|\varrho(\cz\pl) - \rho \chi(\hz\plm) \varrho(\hz\plm)| & \leq \eta,
\end{align*}
where $\rho=\frac{\gamma^2}{\gamma^2 + \beta^2}<1$ and $\chi(\hz\plm)\in \R^+$ is dependent on $\Theta\plm$ but independent of $\theta\pl$ such that $\chi(\hz\plm)\leq1$ in general and $\chi(\hz\plm)=1$ if $\phi=\identity$.

Let us then define $N_6$ independently of $\Theta\plm$, $l$ as 
\begin{align*}
N_6(\eta,\delta)=\max\Bigg(N'(\eta,\delta),N''\Big(\min\big(\eta,\frac{1}{4},\frac{1}{4}F_\phi(\gamma^2+\beta^2)\big),\delta\Big)\Bigg), \quad \forall\eta>0, \quad \forall\delta>0.
\end{align*}

\mbox{Then $\forall \eta>0$, $\forall \delta>0$, if $C\ul\geq N_6(\eta,\delta)$, it holds with probability greater than $1-3\delta$ with respect to $\theta\pl$ that}
\begin{align*}
|\varrho(\rvy\pl) - \varrho(\hz\plm)|
& \leq |\varrho(\rvy\pl) - \varrho(\hy\pl)| + |\varrho(\hy\pl) - \varrho(\hz\plm)|  \\
& \leq 35 \eta + \eta  \\
& \leq 36 \eta,  \\
|\varrho(\rvz\pl) - \rho \chi(\hz\plm)\varrho(\hz\plm)|
& \leq |\varrho(\rvz\pl) - \varrho(\cz\pl)| + |\varrho(\cz\pl) - \rho \chi(\hz\plm)\varrho(\hz\plm)|  \\
& \leq \frac{35\eta}{F_\phi(\gamma^2+\beta^2)} + \eta  \\
& \leq \Bigg(\frac{35 + F_\phi(\gamma^2+\beta^2)}{F_\phi(\gamma^2+\beta^2)} \Bigg) \eta.
\end{align*}

Now, let us define $N_7$ independently of $\Theta\plm$, $l$ as 
\begin{align*}
N_7(\eta,\delta)=N_6\Bigg(\min\Big(\frac{\eta}{36},\frac{F_\phi(\gamma^2+\beta^2)}{35+F_\phi(\gamma^2+\beta^2)} \eta  \Big),\frac{\delta}{3}\Bigg), \quad \forall\eta>0, \quad \forall\delta>0. 
\end{align*}

Then $\forall \eta>0$, $\forall\delta>0$, if $C\ul\geq N_7(\eta,\delta)$, it holds with probability greater than $1-\delta$ with respect to $\theta\pl$ that
\begin{align*}
|\varrho(\rvy\pl) - \varrho(\hz\plm)| & \leq \eta, &
|\varrho(\rvz\pl) - \rho \chi(\hz\plm) \varrho(\hz\plm)| & \leq \eta.
\end{align*}

\Lemmaref{lemma:ln_collapse2} can be used again to deduce for any $\eta>0$ and any $\delta>0$, that if $C\ul\geq N''(\eta,\delta)$, it holds with probability greater than $1-\delta$ with respect to $\theta\pl$ that $\forall\rvx \in \D$:
\begin{align*}
|\pwr_\rvx(\rvy\pl)-1| & \leq |\pwr_\rvx(\rvy\pl)-\pwr_\rvx(\hy\pl)| + |\pwr_\rvx(\hy\pl)-1| \\
& \leq 2 \eta,  \\
|\pwr_\rvx(\rvz\pl)-F_\phi(\gamma^2+\beta^2)| & \leq |\pwr_\rvx(\rvz\pl)-\pwr_\rvx(\cz\pl)| + |\pwr_\rvx(\cz\pl)-F_\phi(\gamma^2+\beta^2)| \\
& \leq 2 \eta. 
 \end{align*}
Let us finally define $N'''$ independently of $\Theta\plm$, $l$ as $N'''(\eta,\delta)=\max \Big( N_7(\eta,\frac{\delta}{2}), N''\big(\frac{\eta}{2},\frac{\delta}{2}\big)\Big)$, $\forall\eta>0$, $\forall\delta>0$. Then, for any $\eta>0$ and any $\delta>0$, if $C\ul\geq N'''(\eta,\delta)$, it holds with probability greater than $1-\delta$ with respect to $\theta\pl$ that
 \begin{align*}
|\varrho(\rvy\pl) - \varrho(\hz\plm)| & \leq \eta, & 
|\varrho(\rvz\pl) - \rho \chi(\hz\plm) \varrho(\hz\plm)| & \leq \eta, \\
|\pwr_\rvx(\rvy\pl)-1| & \leq \eta,\quad \forall \rvx \in \D, &
\big|\pwr_\rvx(\rvz\pl)-F_\phi(\gamma^2+\beta^2)\big| & \leq \eta,\quad \forall \rvx \in \D. \tag*{\qed}
 \end{align*}
\end{proof}

\subsection{Proof of \Thmref{thm:ln_collapse}}
\label{sec:proof_ln_collapse}

\textbf{\Thmref{thm:ln_collapse}}\textbf{.}\emph{
Fix a layer $l\geq 1$ and $\nu_{\vomega}$, $\nu_{\vbeta}$, $\nu_{\vgamma}$, $\D$ in \Defref{def:random}. Further suppose $\Norm=\LN$ and suppose that the convolution of \Eqref{eq:propagation1} uses periodic boundary conditions.}

\emph{Then for any $\eta>0$ and any $\delta>0$, there exists $N(\eta,\delta)\in \mathbb{N}^*$ such that if $\min_{1\leq k \leq l} C\uk~\geq~N(\eta,\delta)$, it holds for random nets of \Defref{def:random} with probability greater than $1-\delta$ with respect to $\Theta\pl$ that 
 \begin{align}
  \pwr(\rvy\pl) - \pwrone(\rvy\pl) & \leq \rho\plm + \eta, &  \pwr(\rvy\pl) & =1,
 \end{align}
where $\rho \equiv \frac{\gamma^2}{\gamma^2 + \beta^2} <1$.}

\begin{proof}\textbf{.} For fixed $\Theta\pkm$ such that $\pwr_\rvx(\rvz\pkm)>0$, $\forall \rvx\in \D$, \Lemmaref{lemma:ln_collapse3} tells us for any $\delta>0$ that, if $C\uk \geq N'''\Big(\min\big(\frac{1}{2},\frac{F_\phi(\gamma^2+\beta^2)}{2}\big),\delta\Big)$ with $F_\phi \equiv \frac{\phi(1)^2 + \phi(-1)^2}{2}>0$, it holds with probability greater than $1-\delta$ with respect to $\theta\pk$ that $\forall \rvx \in \D$:
 \begin{align*}
|\pwr_\rvx(\rvy\pk)-1| & \leq \frac{1}{2},\qquad
& \pwr_\rvx(\rvy\pk) & \geq 1-\frac{1}{2}>0, \\
\big|\pwr_\rvx(\rvz\pk)-F_\phi(\gamma^2+\beta^2)\big| & \leq \frac{F_\phi(\gamma^2+\beta^2)}{2}, \qquad
&  \pwr_\rvx(\rvz\pk) & \geq F_\phi(\gamma^2+\beta^2) - \frac{F_\phi(\gamma^2+\beta^2)}{2}>0.
 \end{align*}

Let us define the event $A\pkm\equiv\Big\{\pwr_\rvx(\rvy\pkm)>0, \forall \rvx\in\D\Big\} \cap \Big\{\pwr_\rvx(\rvz\pkm)>0, \forall \rvx\in\D\Big\}$ with $A\po\equiv \Big\{\pwr_\rvx(\rvx)>0, \forall \rvx\in\D\Big\}$. 

Given that $N'''$ is independent of $\Theta\pkm$, $k$, we deduce for $C\uk \geq N'''\Big(\min\big(\frac{1}{2},\frac{F_\phi(\gamma^2+\beta^2)}{2}\big),\delta\Big)$ that
 \begin{align*}
\P_{\Theta\pk | A\pkm}\big[  A\pk \big] = \P_{\Theta\pk| A\pkm}\Bigg[\Big(\pwr_\rvx(\rvy\pk)>0, \forall \rvx\in\D\Big) \land \Big(\pwr_\rvx(\rvz\pk)>0, \forall \rvx\in\D\Big) \Bigg] \geq 1-\delta.
 \end{align*}

Using the fact that $\pwr_\rvx(\rvx)>0$, $\forall \rvx\in \D$ by \Defref{def:random}, this implies for $\min_{1\leq k \leq l} C\uk \geq N'''\Big(\min\big(\frac{1}{2},\frac{F_\phi(\gamma^2+\beta^2)}{2}\big),\delta\Big)$  that
 \begin{align*}
\P_{\Theta\pl}\big[  A\pl \big] = \P_{\Theta\pl}\Bigg[\Big(\pwr_\rvx(\rvy\pl)>0, \forall \rvx\in\D\Big) \land \Big(\pwr_\rvx(\rvz\pl)>0, \forall \rvx\in\D\Big) \Bigg] \geq (1-\delta)^l.
 \end{align*}

Thus, for any $\delta>0$ there exists $N_8(\delta)\in\mathbb{N}^*$ such that, if $\min_{1\leq k \leq l} C\uk\geq N_8(\delta)$, it holds with probability greater than $1-\delta$ with respect to $\Theta\pl$ that
\begin{align*}
\Big(\pwr_\rvx(\rvy\pl)>0, \forall \rvx\in\D \Big) \land \Big(\pwr_\rvx(\rvz\pl)>0, \forall \rvx\in\D \Big).
\end{align*}

Given that $\Big(\pwr_\rvx(\rvy\pl)>0, \forall \rvx\in\D \Big) \implies \pwr(\rvy\pl)=1$, we deduce for any $\delta>0$ that, if $\min_{1\leq k \leq l} C\uk\geq N_8(\delta)$, it holds with probability greater than $1-\delta$ with respect to $\Theta\pl$ that
\begin{align*}
 \pwr(\rvy\pl)=1.
\end{align*}

Now, until further notice, let us fix some $k$ and some $\Theta\pkm$ such that $\pwr_\rvx(\rvz\pkm)>0$, $\forall \rvx\in \D$. 

Using again \Lemmaref{lemma:ln_collapse3}, we get for any $1>\eta>0$ and any $\delta>0$ that, if $C\uk \geq N'''\Big(F_\phi (\gamma^2+\beta^2)\eta,\delta\Big)$, it holds with probability greater than $1-\delta$ with respect to $\theta\pk$ that
 \begin{align}
  F_\phi(\gamma^2+\beta^2) (1-\eta) & \leq \pwr_\rvx(\rvz\pk) \leq F_\phi(\gamma^2+\beta^2) (1+\eta), \quad \forall \rvx \in \D, \nonumber \\
  F_\phi(\gamma^2+\beta^2) (1-\eta) & \leq \pwr(\rvz\pk) \leq F_\phi(\gamma^2+\beta^2) (1+\eta), \label{eq:ln_collapse4_1} \\
\sqrt{\frac{1-\eta}{1+\eta}} & \leq \sqrt{\frac{\pwr(\rvz\pk)}{\pwr_\rvx(\rvz\pk)}} \leq \sqrt{\frac{1+\eta}{1-\eta}}, \quad  \forall \rvx \in \D,\nonumber \\
\Big( \sqrt{\frac{\pwr(\rvz\pk)}{\pwr_\rvx(\rvz\pk)}} - 1 \Big)^2 & \leq g(\eta)\equiv\max\Bigg( \Big(\sqrt{\frac{1-\eta}{1+\eta}}-1\Big)^2, \Big(\sqrt{\frac{1+\eta}{1-\eta}}-1\Big)^2 \Bigg), \quad \forall \rvx \in \D. \label{eq:ln_collapse4_2} 
 \end{align}

If \Eqref{eq:ln_collapse4_2} holds, then
 \begin{align*}
\hz\pk_{\alpha,c}-\rvz\pk_{\alpha,c} & = \Big(\sqrt{\frac{\pwr(\rvz\pk)}{\pwr_\rvx(\rvz\pk)}}-1\Big) \rvz\pk_{\alpha,c}, \\
\pwr_\rvx(\hz\pk-\rvz\pk) & = \Big(\sqrt{\frac{\pwr(\rvz\pk)}{\pwr_\rvx(\rvz\pk)}}-1\Big)^2 \pwr_\rvx(\rvz\pk) \leq g(\eta) \pwr_\rvx(\rvz\pk), \\
\pwr(\hz\pk-\rvz\pk) & \leq g(\eta) \pwr(\rvz\pk).
 \end{align*}

In turn, this implies that if both \Eqref{eq:ln_collapse4_1} and \Eqref{eq:ln_collapse4_2} hold, then
  \begin{align*}
|\pwr(\hz\pk)-\pwr(\rvz\pk)| & = |\pwr(\hz\pk-\rvz\pk+\rvz\pk)-\pwr(\rvz\pk)| \\
& \leq \pwr(\hz\pk-\rvz\pk) + 2\mu(|\hz\pk-\rvz\pk||\rvz\pk|) \\
& \leq \pwr(\hz\pk-\rvz\pk) + 2\sqrt{ \pwr(\hz\pk-\rvz\pk) \pwr(\rvz\pk)} \\
& \leq \Big(g(\eta)  + 2\sqrt{ g(\eta)} \Big) \pwr(\rvz\pk) \\
& \leq  \Big(g(\eta)  + 2\sqrt{ g(\eta)} \Big) F_\phi(\gamma^2+\beta^2) (1+\eta), \\
| \pwrone(\hz\pk)-\pwrone(\rvz\pk) | & = |\pwrone(\hz\pk-\rvz\pk+\rvz\pk)-\pwrone(\rvz\pk)| \\
& \leq \pwrone(\hz\pk-\rvz\pk) + 2 \E\uc\Big[ \big|\mu\uc(\hz\pk-\rvz\pk)\big| \big| \mu\uc(\rvz\pk) \big| \Big] \\
& \leq \pwr(\hz\pk-\rvz\pk) + 2 \sqrt{\pwr(\hz\pk-\rvz\pk) \pwr(\rvz\pk)} \\
& \leq \Big(g(\eta)  + 2\sqrt{ g(\eta)} \Big) \pwr(\rvz\pk) \\
& \leq \Big(g(\eta)  + 2\sqrt{ g(\eta)} \Big) F_\phi(\gamma^2+\beta^2) (1+\eta).
 \end{align*}
Since $\big(g(\eta)  + 2\sqrt{ g(\eta)} \big) F_\phi(\gamma^2+\beta^2) (1+\eta)\to0$ as $\eta\to 0$, we deduce for any $\eta>0$ and any $\delta>0$ that there exists $N_9(\eta,\delta)\in\mathbb{N}^*$ independent of $\Theta\pkm$, $k$ such that if $C\uk \geq N_9(\eta,\delta)$, it holds with probability greater than $1-\delta$ with respect to $\theta\pk$ that 
\begin{align*}
|\pwr_\rvx(\rvz\pk)-F_\phi(\gamma^2+\beta^2)| \leq \eta \quad  \forall \rvx \in \D, \qquad |\pwr(\rvz\pk)-F_\phi(\gamma^2+\beta^2)| & \leq \eta, \\
|\pwr(\hz\pk)-\pwr(\rvz\pk)| & \leq \eta, \\
| \pwrone(\hz\pk)-\pwrone(\rvz\pk) | & \leq \eta.
\end{align*}
Defining $N_{10}$ independently of $\Theta\pkm$, $k$ as $N_{10}(\eta,\delta)=\max\Big(N'''\big(\eta,\frac{\delta}{2}\big), N_9(\eta,\frac{\delta}{2})\Big)$, $\forall\eta>0$, $\forall\delta>0$, we deduce for any $\eta>0$ and any $\delta>0$ that, if $C\uk\geq N_{10}(\eta,\delta)$, it holds with probability greater than $1-\delta$ with respect to $\theta\pk$ that 
\begin{align*}
|\varrho(\rvy\pk) - \varrho(\hz\pkm)| \leq \eta, \\
|\varrho(\rvz\pk) - \rho \chi(\hz\pkm)\varrho(\hz\pkm)| \leq \eta, \\
|\pwr(\rvz\pk)-F_\phi(\gamma^2+\beta^2)| & \leq \eta, \\
|\pwr(\hz\pk)-\pwr(\rvz\pk)| & \leq \eta, \\
| \pwrone(\hz\pk)-\pwrone(\rvz\pk) | & \leq \eta.
\end{align*}
The reasoning that yielded \Eqref{eq:ln_collapse3_4} from \Eqref{eq:ln_collapse3_1}, (\ref{eq:ln_collapse3_2}), (\ref{eq:ln_collapse3_3}) can be immediately transposed by replacing $\cz\pl$ by $\rvz\pk$ and $\rvz\pl$ by $\hz\pk$.

Thus, if $C\uk \geq N_{10}(\eta,\delta)$, it holds with probability greater than $1-\delta$ with respect to $\theta\pk$ that 
\begin{align*}
|\varrho(\rvy\pk) - \varrho(\hz\pkm)| & \leq \eta, \\
|\varrho(\rvz\pk) - \rho \chi(\hz\pkm)\varrho(\hz\pkm)| & \leq \eta, \\
|\varrho(\hz\pk) - \varrho(\rvz\pk)| & \leq \frac{35\eta}{F_\phi(\gamma^2+\beta^2)}, \\
|\varrho(\hz\pk) - \rho \chi(\hz\pkm)\varrho(\hz\pkm)| & \leq \eta + \frac{35\eta}{F_\phi(\gamma^2+\beta^2)} = \frac{F_\phi(\gamma^2+\beta^2) + 35}{F_\phi(\gamma^2+\beta^2)} \eta.
\end{align*}

Defining $N_{11}$ independently of $\Theta\pkm$, $k$ as $N_{11}(\eta,\delta)=N_{10}\Big(\frac{F_\phi(\gamma^2+\beta^2)}{F_\phi(\gamma^2+\beta^2) + 35}\eta,\delta\Big)$, $\forall\eta>0$, $\forall\delta>0$, we deduce for any $\eta>0$ and any $\delta>0$ that, if $C\uk\geq N_{11}(\eta,\delta)$, it holds that 
 \begin{align*}
\P_{\theta\pk} \Big[ |\varrho(\rvy\pk) - \varrho(\hz\pkm)| \leq \eta \Big] \geq 
\P_{\theta\pk} \Bigg[ |\varrho(\rvy\pk) - \varrho(\hz\pkm)| \leq \frac{F_\phi(\gamma^2+\beta^2)}{F_\phi(\gamma^2+\beta^2) + 35}\eta \Bigg] & \geq 1-\delta, \\
\P_{\theta\pk} \Big[ |\varrho(\hz\pk) - \rho \chi(\hz\pkm) \varrho(\hz\pkm)| \leq \eta \Big] & \geq 1-\delta.
\end{align*}

Considering again $k$ and $\Theta\pkm$ as not fixed, we deduce for any $k$ that, if $C\uk\geq N_{11}(\eta,\delta)$, it holds that 
\begin{align*}
\P_{\Theta\pk | A\pkm} \Big[ |\varrho(\rvy\pk) - \varrho(\hz\pkm)| \leq \eta \Big] 
& \geq 1-\delta, \\
 \P_{\Theta\pk | A\pkm} \Big[ |\varrho(\hz\pk) - \rho \chi(\hz\pkm) \varrho(\hz\pkm)| \leq \eta \Big] 
& \geq 1-\delta.
\end{align*}

Defining $N_{12}$ independently of $\Theta\pk$, $k$ as $N_{12}(\eta,\delta)=\max\Big(N_{11}\big(\eta,\frac{\delta}{2}\big), N_8\big(\frac{\delta}{2}\big)\Big)$, $\forall\eta>0$, $\forall\delta>0$, we deduce for any $k$, any $\eta>0$ and any $\delta>0$ that, if $C\uk\geq N_{12}(\eta,\delta)$, it holds that 
\begin{align*}
\P_{\Theta\pk} \Big[ |\varrho(\rvy\pk) - \varrho(\hz\pkm)| \leq \eta \Big] 
& \geq \P_{\Theta\pk| A\pkm} \Big[ |\varrho(\rvy\pk) - \varrho(\hz\pkm)| \leq \eta \Big] \P_{\Theta\pkm} \big[ A\pkm \big] \\
& \geq \P_{\Theta\pk| A\pkm} \Big[ |\varrho(\rvy\pk) - \varrho(\hz\pkm)| \leq \eta \Big] \P_{\Theta\pl} \big[ A\pl \big] \\
& \geq \Big(1-\frac{\delta}{2}\Big) \Big(1-\frac{\delta}{2}\Big) \geq 1-\delta, \\
\P_{\Theta\pk} \Big[ |\varrho(\hz\pk) - \rho \chi(\hz\pkm) \varrho(\hz\pkm)| \leq \eta \Big] 
& \geq \P_{\Theta\pk| A\pkm} \Big[ |\varrho(\hz\pk) - \rho \chi(\hz\pkm) \varrho(\hz\pkm)| \leq \eta \Big] \P_{\Theta\pl} \big[ A\pl \big] \\
& \geq \Big(1-\frac{\delta}{2}\Big) \Big(1-\frac{\delta}{2}\Big) \geq 1-\delta.
\end{align*}

Thus, for any $\eta >0$ and any $\delta >0$, if $\min_{1\leq k \leq l} C\uk \geq N_{12}(\eta,\delta)$, it holds with probability greater than $1-l\delta$ with respect to $\Theta\pl$ that 
\begin{align}
|\varrho(\hz\pk) - \rho \chi(\hz\pkm)\varrho(\hz\pkm)| & \leq \eta,  \quad \forall k\leq l-1, \label{eq:ln_collapse4_3} \\
|\varrho(\rvy\pl) - \varrho(\hz\plm)| & \leq \eta. \label{eq:ln_collapse4_4}
\end{align}

Given $\chi(\hz\pkm)\leq1$, $\forall k$ and given $\varrho(\rvz^0)=\varrho(\rvx)\leq1$, we note that if \Eqref{eq:ln_collapse4_3} and \Eqref{eq:ln_collapse4_4} hold, then
\begin{align}
\varrho(\hz^1) & \leq \rho \chi(\rvz^0) \varrho(\rvz^0) + \eta  \leq \rho + \eta, \nonumber \\
\varrho(\hz^2) & \leq \rho \chi(\hz^1) \varrho(\hz^1) + \eta \leq \rho^2 + \rho \eta + \eta, \nonumber \\
 &\vdots  \nonumber \\
\varrho(\hz\plm) & \leq \rho \chi(\hz^{l-2}) \varrho(\hz^{l-2}) + \eta \leq \rho\plm + \Big(\sum^{l-2}_{k=0} \rho\pk \Big) \eta \leq \rho\plm + \frac{1}{1-\rho} \eta, \nonumber \\
\varrho(\rvy\pl) & \leq \varrho(\hz\plm) + \eta \leq \rho\plm + \frac{1}{1-\rho} \eta + \eta \leq \rho\plm + \Big(\frac{2-\rho}{1-\rho}\Big) \eta. \label{eq:ln_collapse4_5}
\end{align}

Defining $N_{13}$ such that $N_{13}(\eta,\delta)=N_{12}\Big(\frac{1-\rho}{2-\rho}\eta,\frac{1}{l}\delta\Big)$, $\forall \eta>0$, $\forall \delta>0$, we deduce for any $\eta>0$ and any $\delta>0$ that, if $\min_{1\leq k \leq l} C\uk\geq N_{13}(\eta,\delta)$, it holds with probability greater than $1-\delta$ with respect to $\Theta\pl$ that 
\begin{align*}
\varrho(\rvy\pl) & \leq \rho\plm + \eta. 
\end{align*}

Finally, let us define $N$ such that $N(\eta,\delta)=\max\Big(N_8(\frac{\delta}{2}),N_{13}(\eta,\frac{\delta}{2})\Big)$, $\forall\eta>0$, $\forall\delta>0$. Then for any $\eta>0$ and any $\delta>0$, if $\min_{1\leq k \leq l} C\uk \geq N(\eta,\delta)$, it holds with probability greater than $1-\delta$ with respect to $\Theta\pl$ that 
\begin{align*}
\varrho(\rvy\pl) = \pwr(\rvy\pl) - \pwrone(\rvy\pl) & \leq \rho\plm + \eta, &
 \pwr(\rvy\pl) & = 1. \tag*{\qed}
\end{align*}
\end{proof}

\subsection{Case $\phi=\identity$}
\label{sec:proof_ln_collapse_identity}

\begin{Proposition}
Fix a layer $l\geq 1$ and $\nu_{\vomega}$, $\nu_{\vbeta}$, $\nu_{\vgamma}$, $\D$ in \Defref{def:random}, with $\D$ \say{centered} such that $\pwrone(\rvz\po)=\pwrone(\rvx)=0$. Further suppose $\Norm=\LN$, $\phi=\identity$, and suppose that the convolution of \Eqref{eq:propagation1} uses periodic boundary conditions.

Then for any $\eta>0$ and any $\delta>0$, there exists $N(\eta,\delta)\in \mathbb{N}^*$ such that if $\min_{1\leq k \leq l} C\uk\geq N(\eta,\delta)$, it holds for random nets of \Defref{def:random} with probability greater than $1-\delta$ with respect to $\Theta\pl$ that 
 \begin{align*}
  |\pwr(\rvy\pl) - \pwrone(\rvy\pl) - \rho\plm | & \leq \eta, &  \pwr(\rvy\pl) & =1,
 \end{align*}
where $\rho \equiv \frac{\gamma^2}{\gamma^2 + \beta^2} <1$.
\end{Proposition}

\begin{proof}\textbf{.}
Since $\phi=\identity$ is a particular case of positive homogeneous activation function, the whole proof of \Thmref{thm:ln_collapse} still applies. Let us then define $N_8$, $N_{12}$ as in the proof of \Thmref{thm:ln_collapse}.

Then for any $\delta>0$, if $\min_{1\leq k \leq l} C\uk\geq N_8(\delta)$, it holds with probability greater than $1-\delta$ with respect to $\Theta\pl$ that
\begin{align*}
\pwr(\rvy\pl)=1.
\end{align*}

In addition, for any $\eta>0$ and any $\delta >0$, if $\min_{1\leq k \leq l} C\uk \geq N_{12}(\eta,\delta)$, it holds with probability greater than $1-l\delta$ with respect to $\Theta\pl$ that 
\begin{align*}
|\varrho(\hz\pk) - \rho \chi(\hz\pkm)\varrho(\hz\pkm)|=|\varrho(\hz\pk) - \rho \varrho(\hz\pkm)| & \leq \eta,  \quad \forall k\leq l-1, \\
|\varrho(\rvy\pl) - \varrho(\hz\plm)| & \leq \eta,
\end{align*}
where we used the fact that $\chi(\hz\pkm)=1$, $\forall k$ when $\phi=\identity$.

Next we note that: (i) the assumptions that $\pwrone(\rvz\po)=0$ and $\pwr_\rvx(\rvz\po)>0$, $\forall\rvx$ (cf \Defref{def:random}) together imply $\varrho(\rvz^0)=\varrho(\rvx)=1$; (ii) the reasoning yielding \Eqref{eq:ln_collapse4_5} from \Eqref{eq:ln_collapse4_3}, (\ref{eq:ln_collapse4_4}) still applies. 

We deduce that, if $\min_{1\leq k \leq l} C\uk \geq N_{12}(\eta,\delta)$, it holds with probability greater than $1-l\delta$ with respect to $\Theta\pl$ that 
\begin{align}
\varrho(\hz^1) & \geq \rho \varrho(\rvx) - \eta = \rho - \eta, \nonumber  \\
\varrho(\hz^2) & \geq \rho \varrho(\hz^1) - \eta \geq \rho^2 - \rho \eta - \eta, \nonumber  \\
 &\vdots \nonumber \\
\varrho(\hz\plm) & \geq \rho \varrho(\hz^{l-2}) - \eta \geq \rho\plm - \Big(\sum^{l-2}_{k=0} \rho\pk \Big) \eta \geq \rho\plm - \frac{1}{1-\rho} \eta, \nonumber  \\
\varrho(\rvy\pl) & \geq \varrho(\hz\plm) - \eta \geq \rho\plm - \frac{1}{1-\rho} \eta - \eta \geq \rho\plm - \Big(\frac{2-\rho}{1-\rho}\Big) \eta, \nonumber \\
\varrho(\rvy\pl) & \leq \rho\plm + \Big(\frac{2-\rho}{1-\rho}\Big) \eta, \label{eq:ln_collapse5_1} \\
|\varrho(\rvy\pl) - \rho\plm| & \leq \Big(\frac{2-\rho}{1-\rho}\Big) \eta, \nonumber 
\end{align}
where \Eqref{eq:ln_collapse5_1} follows from \Eqref{eq:ln_collapse4_5}.

As in the proof of \Thmref{thm:ln_collapse}, defining $N_{13}$ such that $N_{13}(\eta,\delta)=N_{12}\Big(\frac{1-\rho}{2-\rho}\eta,\frac{1}{l}\delta\Big)$, we deduce for any $\eta>0$ and any $\delta>0$ that, if $\min_{1\leq k \leq l} C\uk\geq N_{13}(\eta,\delta)$, it holds with probability greater than $1-\delta$ with respect to $\Theta\pl$ that 
\begin{align*}
|\varrho(\rvy\pl) - \rho\plm| \leq \eta. 
\end{align*}

As in the proof of \Thmref{thm:ln_collapse}, defining $N$ such that $N(\eta,\delta)=\max\Big(N_8(\frac{\delta}{2}),N_{13}(\eta,\frac{\delta}{2})\Big)$, $\forall\eta>0$, $\forall\delta>0$, we deduce for any $\eta>0$ and any $\delta>0$ that, if $\min_{1\leq k \leq l} C\uk \geq N(\eta,\delta)$, it holds with probability greater than $1-\delta$ with respect to $\Theta\pl$ that 
\begin{align*}
|\varrho(\rvy\pl)-\rho\plm| = |\pwr(\rvy\pl) - \pwrone(\rvy\pl)-\rho\plm| & \leq \eta, &
 \pwr(\rvy\pl) & = 1. \tag*{\qed}
\end{align*}
\end{proof}

\section{Proof of \Thmref{thm:in_variability}}
\label{sec:in_variability}

\textbf{\Thmref{thm:in_variability}} \textbf{.}\emph{
Fix a layer $l\in\{1,\dots,L\}$ and lift any assumptions on $\phi$. Further suppose $\Norm=\IN$, with \Eqref{eq:propagation2} having nonzero denominator at layer $l$ for all inputs and channels.}

\emph{Then it holds that 
\begin{itemize}
\item $\rvy\pl$ is normalized in each channel $c$ with 
 \begin{align*}
\pwrone\uc(\rvy\pl) & = 0, &  \quad
\pwr\uc(\rvy\pl) & =1; 
 \end{align*}
\item $\rvy\pl$ lacks variability in instance statistics in each channel $c$ with
 \begin{align*}
 \pwrtwo\uc(\rvy\pl)  & = 0, & \quad
 \pwrthree\uc(\rvy\pl) & = 1, & \quad
\pwrfour\uc(\rvy\pl) & = 0.
 \end{align*}
\end{itemize}
}

\begin{proof}\textbf{.} With $\Norm=\IN$, if $\forall \rvx,c$: $\sigma{}_{\rvx,c}(\rvx\pl)>0$, then instance statistics are given by
 \begin{align*}
\mu_{\rvx,c}(\rvy\pl) & = \E_{\alpha}\big[\rvy\pl_{\alpha,c}\big] = \frac{\E_{\alpha}\big[\rvx\pl_{\alpha,c}\big] - \mu_{\rvx,c}(\rvx\pl)}{\sigma{}_{\rvx,c}(\rvx\pl)}
= \frac{\mu_{\rvx,c}(\rvx\pl) - \mu_{\rvx,c}(\rvx\pl)}{\sigma{}_{\rvx,c}(\rvx\pl)} = 0, \\
\pwr_{\rvx,c}(\rvy\pl) & = \E_{\alpha}\Big[(\rvy\pl_{\alpha,c})^2 \Big] = \frac{\E_{\alpha}\Big[\big(\rvx\pl_{\alpha,c} - \mu_{\rvx,c}(\rvx\pl)\big)^2 \Big]}{\sigma_{\rvx,c}(\rvx\pl)^2}
= \frac{\sigma_{\rvx,c}(\rvx\pl)^2}{\sigma_{\rvx,c}(\rvx\pl)^2} = 1, \\
\sigma_{\rvx,c}(\rvy\pl) & = \sqrt{\pwr_{\rvx,c}(\rvy\pl) - \mu_{\rvx,c}(\rvy\pl)^2} = 1.
 \end{align*}

In turn, this implies for the different power terms:
 \begin{align*}
\pwr\uc(\rvy\pl) & = \E_{\rvx} \Big[ \pwr_{\rvx,c}(\rvy\pl) \Big] = \E_{\rvx} \big[ 1 \big] =  1,\\
\pwrone\uc(\rvy\pl) & = \E_{\rvx} \Big[ \mu_{\rvx,c}(\rvy\pl) \Big]^2  = \E_{\rvx} \big[ 0 \big]^2 = 0, & \;
\pwrtwo\uc(\rvy\pl) & = \Var_{\rvx} \Big[ \mu_{\rvx,c}(\rvy\pl) \Big] = \Var_{\rvx} \big[ 0 \big] = 0, \\
\pwrthree\uc(\rvy\pl) & = \E_{\rvx} \Big[ \sigma_{\rvx,c}(\rvy\pl) \Big]^2 = \E_{\rvx} \big[ 1 \big]^2 = 1, & \;
\pwrfour\uc(\rvy\pl) & = \Var_{\rvx} \Big[ \sigma_{\rvx,c}(\rvy\pl) \Big] = \Var_{\rvx} \big[ 1 \big] = 0.\tag*{\qed}
 \end{align*}
\end{proof}
\section{Proof of \Thmref{thm:pn_iterative}}
\label{sec:pn_iterative}

\textbf{\Thmref{thm:pn_iterative}} \textbf{.}\emph{
Fix a layer $l\in\{1,\dots,L\}$ and lift any assumptions on $\phi$ and $\rvx$'s distribution. Further suppose that the neural network implements \Eqref{eq:propagation1}, (\ref{eq:propagation2}), (\ref{eq:propagation4}) at every layer up to depth $L$, with $\epsilon=0$ and \Eqref{eq:propagation2}, (\ref{eq:propagation4}) having nonzero denominators for all layers, inputs and channels.
}

\emph{Finally suppose that
\begin{itemize}
\item The proxy-normalized post-activations $\tz\plm$ at layer $l-1$ are normalized in each channel $c$:
 \begin{align*}
\pwrone\uc(\tz\plm) & = 0, &  \quad
\pwr\uc(\tz\plm) & =1; 
 \end{align*}
\item The convolution and normalization steps at layer $l$ do not cause any aggravation of channel-wise collapse and channel imbalance, i.e. $\forall c,c'$:
 \begin{align*}
\pwrone\uc(\rvy\pl) & = 0, &  \quad
\pwr\uc(\rvy\pl) & = \pwr_{c'}(\rvy\pl); 
 \end{align*}
\item The pre-activations $\rvy\pl$ at layer $l$ are Gaussian in each channel $c$ and PN's additional parameters $\tilde{\vbeta}\pl$, $\tilde{\vgamma}\pl$ are zero.
\end{itemize}
}
\emph{Then both the pre-activations $\rvy\pl$ and the proxy-normalized post-activations $\tz\pl$ at layer $l$ are normalized in each channel $c$:
\begin{align}
\pwrone\uc(\rvy\pl) & = 0, &  \quad
\pwr\uc(\rvy\pl) & = 1, \label{eq:pn_iterative1} \\
\pwrone\uc(\tz\pl) & = 0, &  \quad
\pwr\uc(\tz\pl) & = 1. \label{eq:pn_iterative2}
 \end{align}
}

\begin{proof}\textbf{ of \Eqref{eq:pn_iterative1}.} 
If the denominator of \Eqref{eq:propagation2} is nonzero for all layers, inputs and channels, then it follows from \Propref{prop:powers_one} that $\pwr(\rvy\pl)=\E\uc \big[ \pwr\uc(\rvy\pl) \big]=1$. 

Combined with $\pwrone\uc(\rvy\pl)=0$, $\forall c$ and $\pwr\uc(\rvy\pl)=\pwr_{c'}(\rvy\pl)$, $\forall c,c'$, we deduce that in each channel $c$:
\begin{align*}
\pwrone\uc(\rvy\pl) & = 0, &  \quad
\pwr\uc(\rvy\pl) & = 1. \tag*{\qed}
 \end{align*}
\end{proof}

\begin{proof}\textbf{ of \Eqref{eq:pn_iterative2}.} 
Given \Eqref{eq:pn_iterative1} and given that $\rvy\pl$ is Gaussian in each channel $c$, it holds $\forall c$:
 \begin{align*}
\rvy\pl_{\alpha,c} \underset{\rvx,\alpha}{\sim} \mathcal{N}\big(0,1\big).
 \end{align*}

We deduce that the distribution of $\rvz\pl_{\alpha,c} = \phi\big(\vgamma\pl\uc \rvy\pl_{\alpha,c} + \vbeta\pl\uc \big)$ with respect to $(\rvx,\alpha)$ and the distribution~of $\phi\big(\vgamma\pl\uc Y\pl\uc + \vbeta\pl\uc\big)$ with respect to $Y\pl\uc \sim \mathcal{N}\big(\tilde{\vbeta}\pl\uc, (1+\tilde{\vgamma}\pl\uc)^2\big)=\mathcal{N}\big(0,1\big)$ are equal.

We then get from \Eqref{eq:propagation4} at layer $l$ that
\begin{align*}
\tz\pl_{\alpha,c} & 
= \frac{\rvz\pl_{\alpha,c}  - \E_{Y\pl\uc}\big[\phi\big(\vgamma\pl\uc Y\pl\uc + \vbeta\pl\uc)\big]}{\sqrt{\Var_{Y\pl\uc}\big[\phi\big(\vgamma\pl\uc Y\pl\uc + \vbeta\pl\uc)\big]+\epsilon}} 
= \frac{\rvz\pl_{\alpha,c} - \E_{\rvx,\alpha}\big[ \rvz\pl_{\alpha,c} \big] }{\sqrt{\Var_{\rvx,\alpha}\big[ \rvz\pl_{\alpha,c} \big]}} 
 = \frac{\rvz\pl_{\alpha,c} - \mu\uc(\rvz\pl) }{\sigma\uc(\rvz\pl)},
\end{align*}
where we used $\epsilon=0$.

We deduce that $\mu\uc(\tz\pl)=0$ and $\pwr\uc(\tz\pl)=1$ in each channel $c$, implying that in each channel $c$:
\begin{align*}
\pwrone\uc(\tz\pl) & = \mu\uc(\tz\pl)^2 = 0, &  \quad
\pwr\uc(\tz\pl) & = 1. \tag*{\qed}
 \end{align*}
\end{proof}

\end{document}